\documentclass[nohyperref]{article}

\usepackage{microtype}
\usepackage{graphicx}
\usepackage{subcaption}
\usepackage{booktabs} % for professional tables

\usepackage{hyperref}

\usepackage[accepted]{icml2022}

\usepackage{amsmath}
\usepackage{amssymb}
\usepackage{mathtools}
\usepackage{amsthm}

\usepackage{ulem}

\usepackage[capitalize,noabbrev]{cleveref}

\theoremstyle{plain}
\newtheorem{theorem}{Theorem}[section]
\newtheorem{proposition}[theorem]{Proposition}
\newtheorem{lemma}[theorem]{Lemma}
\newtheorem{corollary}[theorem]{Corollary}
\theoremstyle{definition}
\newtheorem{definition}[theorem]{Definition}
\newtheorem{assumption}[theorem]{Assumption}
\theoremstyle{remark}

\newtheorem{example}[theorem]{Example}

\usepackage[textsize=tiny]{todonotes}

\newcommand{\name}{mTAF}
\newcommand{\names}{mTAFs}
\usepackage{amsmath,amsfonts,bm}

\def\1{\bm{1}}

\def\rvx{{\mathbf{x}}}

\def\rvz{{\mathbf{z}}}

\def\vu{{u}}

\def\vx{{x}}
\def\vz{{z}}

\DeclareMathAlphabet{\mathsfit}{\encodingdefault}{\sfdefault}{m}{sl}
\SetMathAlphabet{\mathsfit}{bold}{\encodingdefault}{\sfdefault}{bx}{n}

\def\sR{{\mathbb{R}}}
\def\sS{{\mathbb{S}}}

\newcommand{\E}{\mathbb{E}}

\renewcommand{\E}{\mathbb{E}}

\title{Marginal Tail-Adaptive Normalizing Flows}

\icmltitlerunning{Marginal Tail-Adaptive Normalizing Flows}

\begin{document}

\twocolumn[
\icmltitle{Marginal Tail-Adaptive Normalizing Flows}



\icmlsetsymbol{equal}{*}

\begin{icmlauthorlist}
\icmlauthor{Mike Laszkiewicz}{math,cs}
\icmlauthor{Johannes Lederer}{math}
\icmlauthor{Asja Fischer}{cs}
\end{icmlauthorlist}

\icmlaffiliation{math}{Faculty of Mathematics, Ruhr University, Bochum, Germany}
\icmlaffiliation{cs}{Center of Computer Science, Bochum, Germany}

\icmlcorrespondingauthor{Mike Laszkiewicz}{Mike.Laszkiewicz@rub.de}

\icmlkeywords{Normalizing Flows, Heavy Tails, Extreme-Value Theory}

\vskip 0.3in
]



\printAffiliationsAndNotice{\icmlEqualContribution} % otherwise use the standard text.
\normalem
\begin{abstract}
Learning the tail behavior of a distribution is a notoriously difficult problem. By definition, the number of samples from the tail is small, and deep generative models, such as normalizing flows, tend to concentrate on learning the body of the distribution. In this paper, we focus on improving the ability of normalizing flows to correctly capture the tail behavior and, thus, form more accurate models.
We prove that the marginal tailedness of an autoregressive flow can be controlled via the tailedness of the marginals of its base distribution.
This theoretical insight leads us to a novel type of flows based on flexible base distributions and data-driven linear layers. 
An empirical analy\-sis shows that the proposed method improves on the accuracy---especially on the tails of the distribution---and is able to generate heavy-tailed data.
We demonstrate its application on a weather and climate example, in which capturing the tail behavior is essential.
\end{abstract}

\section{Introduction}
Heavy-tailed distributions are known to occur in various applications in biology, finance, climate, and many other fields. %Examples %for such observations  } 
Quantities with a heavy-tailed distribution are, for example, the length of protein sequences in genomes \citep{tails_bio}, returns of stocks \citep{tails_stock}, or the occurence and impacts of weather and climate events \citep{climate}. 
In these applications, heavy-tailed events are often the most substantial samples and hence, ignoring them---thinking of underestimating a maximum flood level or the loss of a financial crisis---would yield to crucial model failures. 
From a theoretical point of view, heavy-tailed distributions emerge from several circumstances, including the limiting
distribution in the generalized central limit theorem, of a multiplicative process, or as the limit of
an extremal process \citep{heavytails_book}. Further, many distributions that arise from a functional relationship, such as the ratio of two standard normal distributed random variables, are heavy-tailed, highlighting their importance for models that incorporate known physical relationships among quantities.
Given the frequency of occurrence and their potential impact, developing generative models that allow to learn heavy-tailed distributions are an urgent task to solve. We approach this task by providing important theoretical groundings regarding the expressiveness of Normalizing Flows (NFs) for heavy-tailed data. 

%\ml{Heavy-tailed distributions occur in measurements of the lengths of protein sequences in genomes \citep{tails_bio}, returns of stocks \citep{tails_stock}, sizes of cities \citep{tails_city}, and many more applications in biology, finance, social sciences, and beyond.}
%\sout{Heavy-tailed distributions are known to occur in various applications in biology, finance, social sciences, and more. Examples for such observations include the length of protein sequences in genomes \citep{tails_bio}, returns of stocks \citep{tails_stock}, or the size of cities \citep{tails_city}.} Applications that are tightly connected to typical deep--learning applications include the frequency of class examples in image classification \citep{vanhorn2017devil} and the frequency of words \citep{zipf} in natural language processing. From a theoretical point of view, this variety of occurrences is not surprising since heavy-tailed distributions emerge from several circumstances, including the limiting distribution in the generalized central limit theorem, the limiting distribution of a multiplicative process, or as the limit of an extremal process \citep{heavytails_book}. Given the frequency of occurrence,  developing generative models that allow to learn heavy-tailed distributions is essential. 

Normalizing Flows \citep{rippel2013high, tabakturner, dinh2014nice, rezende2015variational} are a popular class of deep generative models.
Despite their success in learning tractable distributions where both sampling and density evaluation can be efficient and exact, their ability to model heavy-tailed distributions is known to be limited.
\citet{pmlr-v119-jaini20a} identified the problem that autoregressive affine NFs are unable to map a light-tailed distribution to a heavy-tailed distribution.
They propose to solve this issue by replacing the Gaussian base distribution by a multivariate $t$-distribution with one learnable degree of freedom, leading to a model referred to as Tail-Adaptive Flows (TAF).

\begin{figure*}[t]
\begin{center}
%\framebox[4.0in]{$\;$}
\includegraphics[width=\textwidth]{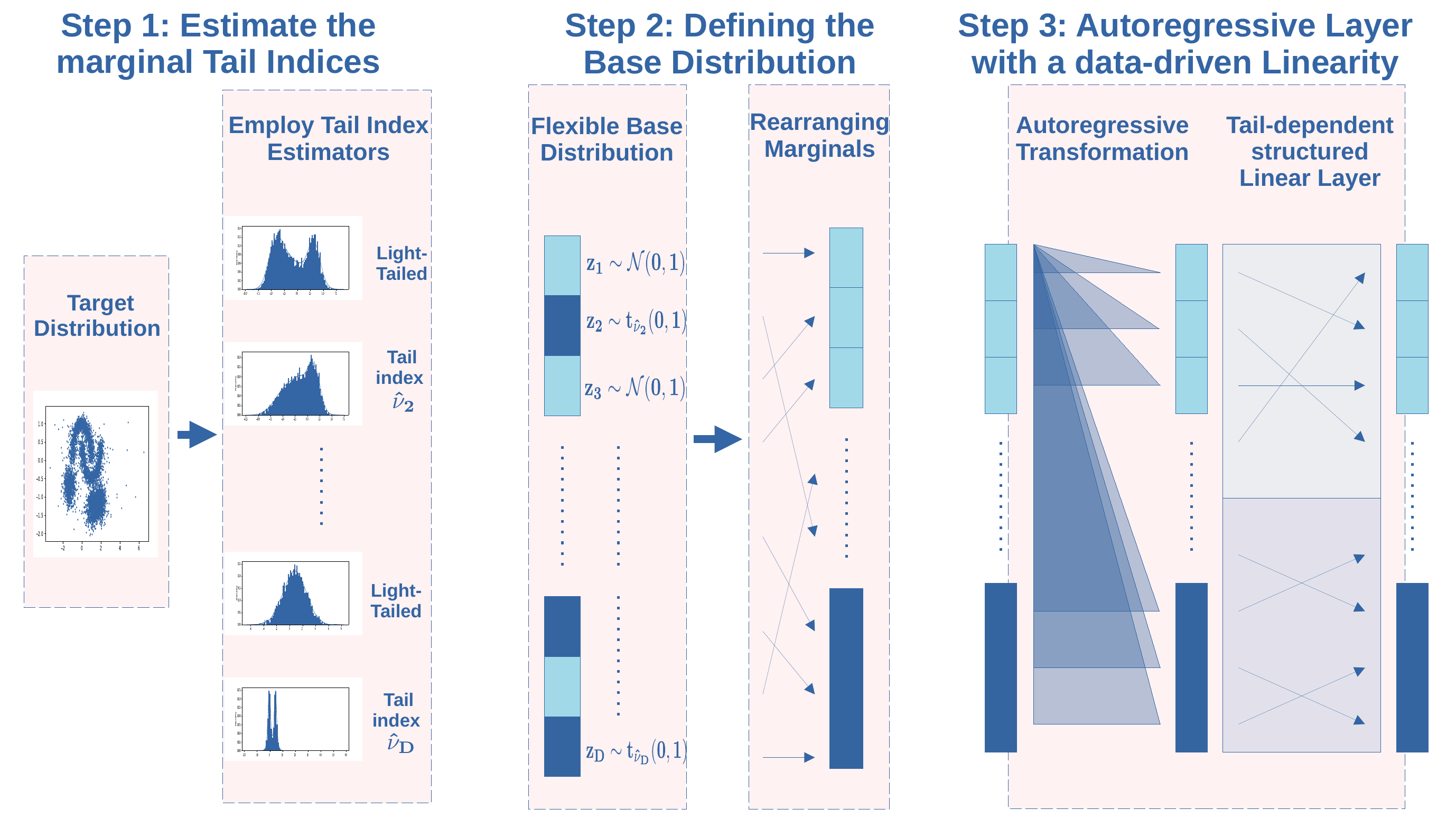}
\end{center}
\caption{An Overview of \name{}. In a first step, we apply estimators from extreme value theory to classify the marginals as heavy- or light-tailed. This classification defines a flexible base distribution consisting of marginal Gaussians and marginal $t$-distributions with flexible degree of freedom, as illustrated by Step 2 of this figure. Further, we rearrange the marginals such that the first $d_l$ marginals are light-tailed, whereas the remaining marginals are heavy-tailed. 
\name{} is then constructed using several flow-layers as visualized in Step 3: we employ a triangular mapping, followed by a 2-group permutation scheme, which can be generalized to general 2-group linearities (Section~\ref{sec:mtaf_nsf}).
At the end, we restore the original ordering using the inverse of the permutation employed in Step 2. Using Theorem~\ref{maintheo}, we prove that \names{} are marginally tail-adaptive (Corollary~\ref{corro}).\label{fig:overview}}
\end{figure*}

\paragraph{Contributions}
In this paper, we extend the work of \citet{pmlr-v119-jaini20a} in multiple ways: First, while TAF allows to model distributions with a heavy-tailed Euclidean norm, we show that modeling multivariate distributions, where some of the marginals are heavy- and some are light-tailed, still poses a problem. More precisely, we identify the problem that an autoregressive affine NF using a base distribution with solely heavy-tailed marginals (such as TAF) is only able to provide a target distribution with just heavy-tailed marginals as well. Consequently, such a NF is not capable of learning distributions with mixed marginal tail behavior. Second, to solve the problem, we derive a theoretical result that states conditions under which the marginal tailedness of the base distribution is preserved. Third, we turn these theoretical insights into a novel modification of autoregressive NFs, which allows to model the marginal tail behavior. Since the proposed model preserves marginal tailedness, we call it \textit{marginal tail-adaptive flow} (\name). The proposed method combines estimators from extreme value theory, a flexible base distribution, and a novel data-driven type of linearities as illustrated in Figure~\ref{fig:overview}. Furthermore, we notice that the autoregressive layers in Neural Spline Flows (NSFs \citep{durkan2019neural}), which are a SOTA architecture, are linear in their tails and, therefore, that we can apply all theory derived by \citet{pmlr-v119-jaini20a} and presented in this paper for neural spline layers. We introduce a simple modification on the LU-layers that ensures that NSF preserve the full marginal tailedness structure of the base distribution. Lastly, we present a generalized and more flexible version of mTAF, which we call \textit{generalized Tail-Adaptive Flow} (gTAF). Our theory is backed up by an experimental analysis demonstrating the superior performance of the proposed methods in learning the tails, especially when it comes to generating synthetic tail samples. Finally, we apply \name{} and gTAF on a climate example to generate heavy-tailed weather data.

\paragraph{Notational Conventions}
In the following, we will denote random variables by bold letters, such as $\rvx$, and its realisations by non-bold letters, $x$. We use this notation for multivariate and for univariate random variables. Further, we denote the $j$th component of $\rvx$ by $\rvx_j$, and $\rvx_{\leq j}$ or $\rvx_{<j}$ are the first $j$ or $j-1$ components of $\rvx$, respectively. We denote the random variable representing the base distribution by $\rvz$ and the random variable representing the target distribution by $\rvx$. Further, for notational convenience, we denote the probability density functions (PDFs) of $\rvx$ and $\rvz$ by $p$ and $q$. %, whereas marginal PDFs are denoted by $p_j$ and $q_j$, respectively.
Finally, we assume that both random variables $\rvx$ and $\rvz$ have continuous and positive density on $\sR^D$, i.e $p(x), \, q(z)>0$ for all $x, z \in \sR^D$, where $D$ is the dimensionality of $\rvx$ and $\rvz$.

\section{Background}
In this section, we give a brief introduction to heavy-tailed distributions and
present needed background knowledge about normalizing flows. %\ml{At the end of this section, we review tail-adaptive flows, which are flows that have recently be proposed to enhance the learning of heavy-tailed distributions.} 

\subsection{Heavy-tailed Distributions} \label{sec:heavy-tails}
Heavy-tailed distributions are distributions that have heavier tails (i.e.~decay slower) than the exponential distribution. % that is, a heavy-tailed distribution decays slower than the exponential distribution.
Loosely speaking, slowly decaying tails allow to model distributions that generate samples, which differ by a large magnitude from the rest of the samples. %Even though it is typically assumed that these samples are some kind of outliers, they might stem from the generative process itself and, hence, should be incorporated to be part of the estimation process. For instance, heavy-tailed distributions are known to occur in various applications in biology, finance, or social sciences. Examples include the length of protein sequences in genomes \citep{tails_bio}, returns of stocks \citep{tails_stock}, or the size of cities \citep{tails_city}. Applications that are tightly connected to typical deep learning applications include the frequency of class examples in image classification \citep{vanhorn2017devil} and the frequency of words \citep{zipf} in natural language processing.
For a univariate random variable $\rvx$ we define heavy-tailedness via its moment-generating function\footnote{One can readily show that this definition is equivalent to the definition, which compares the tails of $\rvx$ to the tails of an exponential distribution. See Section~1 in \citet{heavytails_book}.}:
\begin{definition}[Heavy-Tailed Random Variables]\label{defi:heavytails}
    Consider a random variable $\rvx \in \sR$ with PDF $p$. We say that $\rvx$ is heavy-tailed if and only if
    \begin{equation*}
    \forall \lambda >0 \;: \; 
        \E_{\rvx} \bigl[ e^{\lambda \rvx} \bigr] = \infty \enspace .
    \end{equation*}
    The function $m_{p}(\lambda):= \E_{\rvx}[\exp(\lambda \rvx)]$ is known as the moment-generating function of $\rvx$. Random variables that are not heavy-tailed are said to be light-tailed.
\end{definition}
Note that this definition is, strictly speaking, merely a definition for heavy right tails. 
%\af{[Check: Könnte es Sinn machen "heavy right tailed" direkt mit in die Definition auf zu nehmen?]}\ml{Ich glaube das macht das restliche schreiben immer sehr mühseelig. Z.b. der letzte Satz der obigen DEfi müsste dann ungefähr so umformuliert werden "...that are not neither heavy-right-tailed nor heavy-left-tailed are said to be light-tailed".}
We say a random variable $\rvx\in\sR$ has heavy left tails if $-\rvx$ has heavy tails according to Definition~\ref{defi:heavytails}. For simplicity of derivations and w.l.o.g., we proceed with this definition but the derived results can analogously be applied to left tails.

We can assess the degree of tailedness of a distribution. While there are many equivalent notions of the so called tail index, the most straight-forward definition is via the existence of moments:
\begin{definition}[Tail Index] \label{defi:tail_index}
    A random variable $\rvx\in\sR$ with PDF $p$ is said to have tail index\footnote{Notice that the notion of a tail index is only valid for regularly varying random variables, which are a subclass of heavy-tailed random variables. For the purpose of this work, it is sufficient to consider regularly varying random variables. More details can be found in \citet{heavytails_book}.}
    $\alpha$ if it holds that 
    \begin{equation*}
        \E_{\rvx}[\vert \rvx \vert^\beta] \begin{cases}
            < \infty\; , \enspace \text{if }\beta < \alpha \; , \\
            = \infty\; , \enspace \text{if }\beta>\alpha \; .
        \end{cases}
    \end{equation*}
\end{definition}
Since the tail index is tightly related to the decay rate of the PDF, it enables us to assess the degree of heavy-tailedness of a random variable. Therefore, estimation of the tail index became an important objective in extreme value theory and statistical risk assessment \citep[see e.g.][]{embrechts2013modelling}. 
Since the existence of the moment does not depend on the ``body'' of $\rvx$ but only on the tails of $\rvx$ (see Proposition~\ref{prop:tailindex_dependsontails} in Section~\ref{sec:theoryforproofs} in the Appendix), estimating the tail index by fitting a full parametric model to all data e.g.~via likelihood maximization leads to a biased estimator. Instead, semi-parametric estimators have been developed, which aim to fit a distribution only on the tails. Popular methods for tail estimation include the Hill estimator \citep{hill1975simple}, the moment estimator \citep{moments_estimator}, and kernel-based estimators \citep{kernel_estimator}. 
In Section~\ref{sec:tailest} of the Appendix, we discuss these tail estimators and review some practical issues with these.

An example of a heavy--tailed distribution is the standardized $t$-distribution, which has parameter $\nu>0$ referred to as the degree of freedom and a density function given by
\begin{equation*}
    p(x):= \frac{ \Gamma\bigl( \frac{\nu + 1}{2} \bigr) }{\sqrt{\nu \pi} \Gamma \bigl( \frac{\nu}{2}\bigr)} \bigl( 1 + \frac{x^2}{\nu} \bigr)^{-\frac{\nu + 1}{2}} \; , \enspace  x \in \sR \; ,
\end{equation*}
%\af{[Check: Hier benutzen wir x als scalar, weiter oben in vector-notation. Wenn wir das in der Notation unterscheiden wollen, sollten wir kontrollieren, dass wir das konsistent im gesamten Paper machen.]}\ml{Wir benutzen für univariate, als auch multivariate die selbe notation (siehe Notation-Section) Also Zufallsvariablen immer dick und realisationen immer normal, unabhängig von der dimension.}
where $\Gamma$ is the Gamma function.
It is known that the $t$-distribution %is a heavy-tailed distribution and 
has tail index $\nu$ (see e.g. \citet{Kirkby2019MomentsOS} for a detailed reference).
%\af{[Note: Kann man hierfür eine Referenz angeben?]} \ml{das ist eine sehr detaillierte referenz... Kennt ihr ein Standardwerk, das sowas drin haben sollte?}
%Further, one can also show that the $t$-distribution converges to a standardized Gaussian distribution as $\nu$ tends to $\infty$. 

In the multivariate setting, there exist various definitions of heavy-tailedness.
For instance \citet{resnick} make use of a definition based on multivariate regular variation. 
\citet{pmlr-v119-jaini20a} define a multivariate random variable $\rvx$ to be heavy-tailed if the $\ell_2$-norm is heavy-tailed, a property which we refer to as $\ell_2$-heavy-tailed, and which is formally defined as follows:
\begin{definition}[$\ell_2$-Heavy-Tailed]\label{defi:ell2}
    Let $\rvx \in \sR^D$ be a multivariate random variable. Then, we call $\rvx$ $\ell_2$-heavy-tailed if $\Vert \rvx \Vert$ is univariately heavy-tailed according to Definition~\ref{defi:heavytails}, where $\Vert \cdot \Vert$ denotes the $\ell_2$-norm. Otherwise, we call $\rvx$ $\ell_2$-light-tailed.
\end{definition}

\subsection{Normalizing Flows} \label{sec:NFs}
The fundamental idea behind NFs is based on the change-of-variables formula for probability density functions (PDFs) given in the following theorem.
\begin{theorem}[Change-of-Variables]
    Consider random variables $\rvx, \rvz \in\mathbb{R}^D$ and a diffeomorphic map $T:\sR^D \rightarrow \sR^D$ such that $\rvx=T(\rvz)$. Then, it holds that the PDF of $\rvx$ satisfies 
    \begin{equation}
        \label{eq:changeofvar}
        p(\vx)= q\bigl(T^{-1}(\vx)\bigr) \bigl\vert \det J_{T^{-1}}(\vx) \bigr\vert \quad \forall \vx\in\sR^D \enspace ,
    \end{equation}
    where $J_{T^{-1}}(\vx)$ is the Jacobian of $T^{-1}$ evaluated at $\vx\in\sR^D$.  
\end{theorem}
This formula allows us to evaluate the possibly intractable PDF of $\rvx$ if we can evaluate both, the PDF of $\rvz$ and $T^{-1}(x)$, and efficiently calculate the Jacobian-determinant $\det J_{T^{-1}}(x)$. As $T$ maps $\rvz$ to $\rvx$, we denote the distribution of $\rvz$ and $\rvx$ as the \emph{base} and the \emph{target distribution}, respectively.
%Therefore,

To model the PDF of $\rvx$ using NFs, it is common to set the base distribution to a standard normal distribution (i.e.,
$\rvz\sim\mathcal{N}(0, I)$) and to employ  likelihood maximization to learn a parameterized transformation 
\begin{equation*}
    T_\theta := T_\theta^{(L)} \circ \cdots \circ T_\theta^{(1)},
\end{equation*}
which, yet, remains tractable and diffeomorphic. 
Masked autoregressive flows (MAFs \citep{papamakarios2017masked}) are one popular architecture\footnote{Stricly speaking, \eqref{eq:maf} shows the transformations of an IAF \citep{kingma2016improved}. However, MAF and IAF are theoretically equivalent, the differences lie only in their architectures.}, which employ transformations $T=(T_1, \dots , T_D)^\top$ of the form
\begin{equation} \label{eq:maf}
T_j(\vz):= \mu_j(\vz_{<j}) + \exp(\sigma_j(\vz_{<j})) \vz_j \quad \text{for } j\in\{1,\dots,D\} \enspace,
\end{equation}
where $\mu_j$ and $\sigma_j$ are neural networks, which obtain the first $j - 1$ components of $\vz$ as input and output a scalar. Composing several transformations of the form~\eqref{eq:maf}, we obtain the MAF. The autoregressive form in~\eqref{eq:maf} allows us to efficiently evaluate the Jacobian-Determinant due to the triangular form of $J_T(x)$, which is why autoregressive NFs are also referred to as triangular flows. 
By shuffling the ordering of the components, i.e. applying a permutation after each autoregressive transformation, we are able to form more diverse causal dependencies, leading to more expressive models. It has been shown that replacing the permutations by more general invertible layers can further improve the estimation performance \citep{oliva2018transformation}. 
%A crucial issue of such autoregressive models is that a component $\rvx_j$ only depends on the previous outputs $\rvx_{< j}$ and, therefore, they cannot model a causal relationship in which $\rvx_j$ causes $\rvx_i$ if $i < j$.  %To enhance the, so called, mixing between the variables, it is
%This issue can be solved by applying a permutation before each transformation $T_\theta^{(1)},\dots,T_\theta^{(L)} $. This is usually a random permutation or the one that reverses the ordering of the components. 
In summary MAFs consist of multiple consecutive layers $T_\theta^{(l)} \circ P^{(l)}$, where $P^{(l)}\in \sR^{D\times D}$ is some linear layer.
%
%MAFs belong to the class of triangular flows, which are defined as flows that consist of autoregressive transformations whose $j$th output only depends on $\rvz_{\leq j}$. 
Other examples for triangular flows include NAF \citep{huang2018neural}, and NSF \citep{durkan2019neural}, where the latter substitutes the affine transformation~\eqref{eq:maf} by a rational of two splines. %SOS \citep{jaini2019sum}. 
%If the triangular maps are affine linear (such as in~\eqref{eq:maf}), we call the resulting flow a triangular affine flow.
Further types of NFs include invertible ResNets \citep{jacobsen2018revnet, behrmann2019invertible, chen2019residual}, continuous flows \citep{chen2018neural, grathwohl2018ffjord}, and many more \citep{kobyzev2020normalizing}.
\paragraph{Tail-Adaptive Flows.}\label{sec:TAF}
%\af{[Note: Hier würde ich direkt mit Jainis Arbeit anfangen und hab mal ein paar Sätze auskommentiert]}
%The chain rule of probabilities allow us to write any joint distribution as a product of conditional probabilities. Therefore, autoregressive models are, presuming that the transformations are expressive enough, universal approximators \citep[see e.g][]{Gaussianization, papamakarios2019normalizing}. In spite of being a nice theoretical property, the practical impact of the universal approximation ability for NFs is limited.
\citet{pmlr-v119-jaini20a} investigated the ability of triangular flows to learn heavy-tailed distributions. % 
The authors have shown that if a triangular affine flow transforms a $\ell_2$-light-tailed distribution, such as the multivariate Gaussian distribution, to a $\ell_2$-heavy-tailed target distribution, then $T_\theta$ cannot be Lipschitz continuous. And more explicitly, it holds the following. 
\begin{theorem}{\citep{pmlr-v119-jaini20a}}\label{theo:jaini}
    Let $\rvz$ be a $\ell_2$-light-tailed random variable and $T$ be an affine triangular flow such that $T_j(z_{\leq j})=\mu_j(z_{<j}) + \sigma_j(z_{<j}) z_j$ for all $j$. If $\sigma_j$ is bounded above and $\mu_j$ is Lipschitz for all $j$, then the transformed variable $\rvx$ is also $\ell_2$-light-tailed.  
\end{theorem}
%Note that heavy-tailedness in the above Theorem is defined via univariate heavy-tailedness of $\Vert \rvx \Vert$ and $\Vert \rvz \Vert$.
Furthermore, the authors prove that any triangular mapping from an elliptical distribution to a heavier-tailed elliptical distribution must have an unbounded Jacobian-determinant.
These results illuminate that learning a heavy-tailed distribution using NFs leads to non-Lipschitz transformations and unbounded Jacobians, which inevitably affects training robustness \citep{pmlr-v130-behrmann21a}.
Motivated by these result, \citet{pmlr-v119-jaini20a} propose \emph{Tail-Adaptive Flows} (TAF), which replace the Gaussian base distribution by a multivariate $t$-distribution with one learnable degree of freedom.

\section{Learning the correct marginal Tail Behavior with \name } \label{sec:mTAFandTheory}
In this section, we present a simple extension to triangular affine flows that
allows to model distributions with a flexible tail behavior. We start by presenting our theoretical results in Section~\ref{sec:motivation}.
Motivated by these results, we propose \textit{marginally Tail-Adaptive Flow} (\name) in Section~\ref{sec:method}, which we apply to NSFs in Section~\ref{sec:mtaf_nsf}.
Lastly, we present a more flexible relaxation of \name{}, which we call \textit{generalized Tail-Adaptive Flow} (gTAF) in Section~\ref{sec:gtaf}.

\subsection{The Necessity of a flexible Base Distribution}\label{sec:motivation}
In this work, we investigate the tailedness of NFs more thoroughly through the lense of marginal tailedness, i.e. we consider the univariate tailedness of the marginal distributions of $\rvx_j$.Therefore, we introduce the following definitions:
\begin{definition}[$j$-heavy-tailed, mixed-tailed, fully heavy-tailed, equal Tail Behavior]
    We call a random variable $\rvx\in\sR^D$ $j$-heavy-tailed if its $j$th marginal $\rvx_j$ is heavy-tailed according to Definition~\ref{defi:heavytails}. Otherwise, we call $\rvx$ $j$-light-tailed. $\rvx$ is said to be mixed-tailed if there exists $j_1, j_2$ such that $\rvx$ is $j_1$-heavy-tailed and $j_2$-light-tailed. Further, we say that $\rvx$ is fully heavy-tailed if $\rvx$ is $j$-heavy-tailed for all $j\in \{1,\dots, D\}$. We define two random variables $\rvx$ and $\rvz$ to have equal tail behavior if it holds for all $j$ that 
    \begin{equation*}
        \rvx \text{ is } j\text{--heavy--tailed} \enspace \Leftrightarrow \enspace \rvz \text{ is } j\text{--heavy--tailed} \enspace .
    \end{equation*}
\end{definition}
We found the following relation to Definition \ref{defi:ell2}.
\begin{proposition}[$j$-heavy-tailedness induces $\ell_2$-heavy-tailedness]\label{prop:marg&normtail}
    Assume that $\rvx$ is $j$-heavy-tailed for any $j$. Then, $\rvx$ is also $\ell_2$-heavy-tailed. 
\end{proposition}
The proof can be found in Section~\ref{sec:theoryforproofs} in the Appendix.
The proposition shows that $j$-heavy-tailedness is a more specific notion of multivariate heavy-tailedness than $\ell_2$-heavy-tailedness, which allows a narrow inspection of the tail behavior. More precisely, the new notion allows us to differentiate between fully heavy-tailed random variables and mixed-tailed random variables, which are both $\ell_2$-heavy-tailed. 
%\sout{One can now wonder how the tails, described by the novel notation, behave for NFs, which we answer by the following two results.} 
The first result states that, under mild technical conditions, fully heavy-tailedness of the base distribution is preserved by triangular affine maps.
\begin{proposition}[Triangular affine Maps preserve fully heavy-tailedness] \label{prop:allheavybad}
Let $\rvz$ be a fully heavy-tailed random variable that satisfies Assumption~\ref{assu:copuladensity}\footnote{This Assumption can be found in Section~\ref{sec:proofs} in the Appendix.} and let $T$ be a a triangular affine map, that is, $T_j(z_j, z_{<j})=\mu_j(z_{<j}) + \sigma_j(z_{<j}) z_j$ with $\sigma_j>0$. 
%Then, it holds that all marginals of $T(\rvz)$ are heavy-tailed.
Then, it holds that $T(\rvz)$ is also fully heavy-tailed.
\end{proposition}
We provide a formal proof in Section~\ref{sec:proofs} of the Appendix. Assumption~\ref{assu:copuladensity} is a mild condition on the decay rate of the copula density of $\rvz$.
In Section~\ref{sec:notesonassumption} of the Appendix, we explain this condition in more detail and give various examples. 

It is clear that permuting the marginals does not change the heavy-tailedness. Hence, by iterative application of Proposition~\ref{prop:allheavybad}, we deduce that affine triangular flows that employ permutation layers and a fully heavy-tailed base distribution are unable to model mixed tailed distributions.
Implicitly, Proposition~\ref{prop:allheavybad} states  that a Lipschitz normalizing flow as proposed by \citet{pmlr-v119-jaini20a} is not able to model mixed-tailed distributions.
The following theorem provides sufficient conditions under which a flow is able to model mixed-tailed distributions, which guides us towards a marginally tail-adaptive flow architecture.
%However, the following theorem guides us towards a flow architecture that is able to model target distributions that are mixed-tailed in a marginally tail-adaptive way.
\begin{theorem}[Learning the correct Tail Behavior] \label{maintheo}
    Consider a random-variable $\rvz$ that is $j$-light-tailed for $j\in \{1,\dots, d_l\}$ for some $d_l<D$ and $j$-heavy-tailed for $j\in\{d_l +1, \dots ,D\}$. Then, under the same conditions as in Theorem~\ref{theo:jaini} and Proposition~\ref{prop:allheavybad}, it holds that $\rvz$ and $T(\rvz)$ have the same tail behavior.
\end{theorem}
\begin{proof}
    Since the result combines Theorem~\ref{theo:jaini} and Proposition~\ref{prop:allheavybad} in an evident fashion, we just quickly present a sketch of the proof. 
    First, let us consider $j\leq d_l$. Then it holds for the moment-generating function of $\rvx_j$ that 
    \begin{align*}
        m_{\rvx_j}(\lambda) &= \int_{\sR^D} e^{\lambda T_j(z_{\leq j})} q(z)dz \\
        &= \int_{\sR^j}  e^{\lambda T_j(z_{\leq j})} p_{\leq j}(z_{\leq j})dz_{\leq j} \enspace ,
    \end{align*}
    which has been shown to be bounded for some $\lambda >0$ (see the proof of Theorem~\ref{theo:jaini} in \citet{pmlr-v119-jaini20a}). Therefore, $\rvx$ is $j$-light-tailed for all $j\leq d_l$. 
    In the case $j>d_l$, we notice\footnote{For details, we refer to the proof of Proposition~\ref{prop:allheavybad} in the Appendix.} that the proof for heavy-tailedness of $T_j(\rvz_{\leq j})$ involves just the heavy-tailedness of $\rvz_j$ and not of any other component of $\rvz_{<j}$.  Hence, if $\rvz_j$ is heavy-tailed, then $\rvx_j=T_j(\rvz_{\leq j})$ is also heavy-tailed, regardless of $\rvz_{<j}$. Therefore, $\rvx$ is $j$-heavy-tailed for all $j>d_l$, which completes the proof. Note that in general we cannot deduce the latter conclusion for light-tailed marginals, i.e. if $\rvz_j$ is light-tailed, this does not mean that $\rvx_j$ is also light-tailed. This is only the case, if all $\rvz_{<j}$ are light-tailed as well. 
\end{proof}
\subsection{Marginally Tail-Adaptive Flow (\name)} \label{sec:method}
Our main result, Theorem~\ref{maintheo}, %thus 
prompts that if we maintain an ordering of the marginals such that the first marginals are light-tailed and the following are heavy-tailed in each flow step, %then 
we retain the marginal tail behavior of the base distribution in the estimated target distribution.  
This finding motivates the novel NF proposed in this paper. The proposed approach combines research findings from extreme value theory \citep{embrechts2013modelling, heavytails_book}, recent findings about normalizing flows \citep{pmlr-v119-jaini20a, alexanderson2020robust, laszkiewicz2021copulabased}, and the results presented herein. 
The proposed \names{} consists of three steps that are depicted in Figure~\ref{fig:overview} and described in the following: % \ml{Ist der Satz so richtig?}

%\begin{enumerate}
  %  \item[\textbf{Step 1:}] 
  %\textbf{Estimat\af{ing} the marginal tail indeces \af{and defining the marginal distributions}.} 
    \textbf{Step 1: Estimating the marginal tail indices and defining the marginal distributions.} 
    For each marginal, i.e.~for the marginal distribution $q_j$ of each $\rvx_j$, $j=1,\dots, D,$ %we use the moments double-bootstrap estimator \citep{moments_bootstrap} and the kernel-type double-bootstrap estimator \citep{kernel_bootstrap} to assess heavy-tailedness of the data distribution. If both estimators predict a light-tailed distribution, we set the corresponding marginal base distribution $q_j$ to be standard normal distributed, i.e.~%we assume\ml{define/fix?}  
    we assess heavy- or light-tailedness using tail estimators. If the marginal is predicted to be light-tailed, we set the corresponding marginal base distribution to be standard normal distributed $\rvz_j \sim \mathcal{N}(0, 1)$. Otherwise we set the marginal to the standardized $t$-distribution with the estimated degree of freedom, i.e.~$\rvz_j \sim t_{\hat{\nu}_j}$, where $\hat{\nu}_j$ is the Hill double-bootstrap estimator \citep{hill_bootstrap1, hill_bootstrap2}. We present all the details about the tail-assessment scheme in Section~\ref{sec:tailest}.%In Section~\ref{sec:tailest} of the Appendix, we present more details about the tail-assessment scheme.
    
   % \item[\textbf{Step 2:}]
  % \textbf{Defining the base distribution.}
\textbf{Step 2: Defining the base distribution.}
    We construct the base distribution as the mean-field approximation of the marginals, i.e. $\rvz$ has the density $q(z):= \prod_{j=1}^D q_{j}(z_j)$ with marginal densities $q_{j}$ defined in step 1. Further, to satisfy the assumptions of Theorem~\ref{maintheo}, we need to permute the marginals such that it holds $\rvz_{j} \sim \mathcal{N}(0,1)$ for $j\leq d_l$ and %the remaining components 
    $\rvz_{j} \sim t_{\hat{\nu}_j}$ for $j>d_l$. We apply the same permutation to restructure our data according to the base components. To account for tail index estimation errors and for more flexible learning, one can make the tail indices (i.e. the degrees of freedom of each $t$-distribution) learnable. That is, we initialize the degree of freedom of the $j$th marginal with $\hat{\nu}_j$ but adapt the parameter together with the network parameters throughout training.
    
    \textbf{Step 3: A data-driven permutation scheme.}  
    Recall, that vanilla autoregressive flows employ a permutation step after each transformation to enhance the mixing of variables. However, purely random permutations might lead to a violation on the ordering of marginals, which is necessary to ensure Theorem~\ref{maintheo}. Therefore, we permute only within the set of heavy-tailed marginals and within the set of light-tailed marginals, to ensure the validity of Theorem~\ref{maintheo}. Within these groups one can choose any permutation scheme. We generalize this result for LU-layers in Section~\ref{sec:mtaf_nsf}. 
%\end{enumerate}
%We provide an overview of the %architecture 
%\af{approach}
%in Figure~\ref{fig:overview}. %\af{Note,   that \name} can be employed to extend any triangular flow. \af{[Note: nicht wirklich jeden, oder? ]}
%

Without loss of generality, we assume that the first $d_l$ components of $\rvz$ are light-tailed and the remaining $D-d_l$ components are heavy-tailed\footnote{Otherwise we permute the marginals as described in Step 2.}. Then, the training objective is to optimize for flow parameters $\hat{\theta}$ and degrees of freedom $\hat{\nu}=[\hat{\nu}_{d_l +1}, \dots , \hat{\nu}_{D}]$ %such that 
to maximize the log-likelihood
\begin{align*}
    L(\hat{\theta}, \hat{\nu}; X)
    &=  \sum_{j=1}^N \Biggl\{ \sum_{i=1}^{d_l} \log \pi \Bigl( T_{\hat{\theta}}^{-1}\bigl(x^{(j)}\bigr)_i \Bigr)  \\ &\hspace{-1cm}+ \sum_{i=d_l + 1}^D  \log t_{\hat{\nu}_i}\Bigl( T^{-1}_{\hat{\theta}} \bigl(x^{(j)}\bigr)_i \Bigr) - \log \det J_{T_{\hat{\theta}}}\bigl(x^{(j)}\bigr) \Biggr\} %\enspace ,
\end{align*}
where $X:=(x^{(1)}, \dots x^{(N)})$ is the data, and $\pi$ and $t_{\hat{\nu}}$ are the PDF of the standard normal distribution and the standard $t$-distribution with $\hat{\nu}$ degrees of freedom, respectively.

When applying our theoretical results presented in the previous section to the proposed \name, we can show that it fulfills the desired tail-preserving property, as formalized by the following corollary: 
\begin{corollary}[Marginally tail-adaptive]\label{corro} 
    Under the same assumptions as in Theorem~\ref{theo:jaini} and in Proposition~\ref{prop:allheavybad}, \names{} are marginally tail-adaptive, that is, $\rvz$ and $\rvx=T(\rvz)$ have the same tail behavior. 
\end{corollary}

\subsection{Marginal Tail-Adaptive Neural Spline Flows}\label{sec:mtaf_nsf}
Recent findings on NFs lead to significant improvements of their performance, such as employing LU-Layers instead of permutations \citep{kingma2018glow, oliva2018transformation} and more expressive autoregressive layers.
%more general linear layers \citep{oliva2018transformation} instead of permutations and more expressive autoregressive layers. 
One of the current SOTA NFs, which combine both improvements, are NSFs \citep{durkan2019neural}. 
In this section we apply the theory from the previous sections to NSFs with a modified version of the LU-Layers, while retaining their computational benefits. 
%\sout{Further, using this modification, we realize that NSF naturally preserve the tail behavior of the base distribution, and, in fact, are very related to composite models, which are some known models in heavy-tailed statistics.} 

First, let us provide sufficient conditions under which linear layers preserve the marginal tail behavior:\footnote{The proof is presented in Section~\ref{sec:theory_LU}.}
\begin{theorem}\label{theo_LU}
Let $\rvz$ be a random variable that is $j$-light-tailed for $j\in\{1,\dots ,d_l \}$ and $j$-heavy-tailed for $j \in \{d_l +1, \dots, D\}$. Further, consider a block-diagonal invertible matrix 
\begin{equation} \label{eq_blockmatrix_main}
    W = \begin{pmatrix}
    A & 0 \\ 
    B & C   \end{pmatrix} \enspace 
\end{equation}
with $A\in \mathbb{R}^{d_l \times d_l}, \, B \in \mathbb{R}^{(D-d_l ) \times d_l }, \, C\in\mathbb{R}^{(D-d_l ) \times (D-d_l )}$ and $0$ is a zero matrix of size $d \times (D-d_l )$. 
Then, it follows that $W\rvz$ and $\rvz$ have equal tail behavior. 
\end{theorem}
As a special case of Lemma~\ref{lemma_inversion}, we can invert a block-diagonal matrix 
\begin{equation}\label{eq_tailLU_inv}
        \begin{pmatrix}
            A & 0 \\ 
            B & C 
        \end{pmatrix}^{-1} 
        = 
        \begin{pmatrix}
            A^{-1} & 0 \\ 
            -C^{-1}BA^{-1} & C^{-1} 
        \end{pmatrix} \enspace  .
    \end{equation}
Further, the log-determinant of $W$ is given by 
\begin{equation}
    \label{eq_det_tailLU}
    \log \det(W) = \log \det(A) + \log \det(C) \enspace. 
\end{equation}
Therefore, as long as log-determinant computation and inversion of $A$ and $C$ are efficient, we can efficiently employ a block-diagional matrix as a flow layer. Luckily, both is given if we parameterize $A$ and $C$ using the LU-decomposition, whereas $B$ can be an arbitrary matrix. 

Moreover, NSFs make use of monotonic rational-quadratic splines to define the autoregressive layers. These splines are defined within some interval $[-b, b]$ and are linearily extended outside this interval. 
Hence, the autoregressive NSF layers are in fact affine linear in their tails, which in turn means that we can apply all the theory from the previous sections on NSFs as well (compare with Lemma~\ref{prop:technicaltails}). 

In summary, NSF with linear layers according to the block-form~\eqref{eq_blockmatrix_main} preserve the marginal tailedness of the base distribution. We want to highlight that even though each autoregressive NSF layer is linear outside $[-b,b]$, this does not mean that the whole flow is linear outside $[-b, b]$. This is because the (modified) LU-layers in between can map a component in and out of $[-b, b]$, leading to a non-trivial mapping outside that interval.

%A further difference of NSF to MAF is the more expressive autoregressive layer. Instead of considering affine transformations such as in~\eqref{eq:maf}, the transformation in NSF is given by a monotonic rational-quadratic spline within some interval $[-B, B]$ and outside the interval, it is continuously extended by the identity function. Making use of Lemma~\ref{prop:technicaltails}, it is clear that these layers do not effect the marginal tail behavior. In summary, we can inherit the expressiveness of NSF within the body of the distribution, while making use of the tail estimators in the tails. 

\subsection{Generalized Tail-Adaptive Flows}\label{sec:gtaf}
Even though being theoretically founded, we introduce a more flexible relaxation of \name{}, which we call generalized TAF (gTAF): We drop the structural restrictions on the linearities and set the base distribution to a mean-field approximation of $t$-distributions with different trainable degrees of freedom. Therefore, gTAF is a compromise between the theoretically stronger \name{} and TAF. Note that since $t_\nu \xrightarrow{\nu \rightarrow \infty} \mathcal{N}(0,1)$ we are able to approximately model heavy- as well as light-tailed marginals. Further, since LU-layers are trainable as well, gTAF is able to approximate \name{} arbitrary well by learning a structure as in~\eqref{eq_blockmatrix_main}.
\begin{theorem}\label{theo:gTAFapproxmTAF}
    Let $T:\mathbb{R}^D \rightarrow \mathbb{R}^D$ be almost surely continuous. Further, let $z=(z_1,\dots , z_D)$ be the mTAF base distribution with $d_l$ Gaussian marginals and let $z_{\nu }$ be the gTAF base distribution with marginals
    \begin{equation*}
        z_{\nu, j}\sim \begin{cases}
            t_\nu(0, 1) \quad &\text{ for  } j\leq d_l  \\ 
            z_{j} \quad &\text{ for } j>d_l
        \end{cases} \enspace . 
    \end{equation*} 
    Then, it holds that $T(z_{\nu}) \xrightarrow[]{\nu \rightarrow \infty}_D T(z)$,
    where $\rightarrow_D$ denotes convergence in distribution. 
\end{theorem}
Hence, when fixing the flow $T$, gTAF converges to the mTAF solution as the light-tailed degrees of freedom tend to $\infty$. We provide a proof in Section~\ref{sec:proofgTAFapproxmTAF}. 

\begin{table*}[t]
\caption{Average test loss, Area under log-log plot, and $\operatorname{tVaR}$ (lower is better for each metric) in the setting $\nu=2$ and $d_h\in \{1,4\}$. The copula model serves as an oracle baseline. 
\label{table:synth_metrics}}
\label{sample-table}
\begin{center}
\begin{tabular}{@{}lccccccccccc@{}}
%& \multicolumn{4}{c}{number of heavy-tailed components $h$} \\
%\cmidrule{2-5}
\toprule
$d_h$ & \multicolumn{5}{c}{1} & &\multicolumn{5}{c}{4}  \\
\cmidrule{2-6}
\cmidrule{8-12}
& $L$ & $\operatorname{Area}_l$ & $\operatorname{Area}_h$ & $\operatorname{tVaR}_l$ & $\operatorname{tVaR}_h$ && $L$ & $\operatorname{Area}_l$ & $\operatorname{Area}_h$ & $\operatorname{tVaR}_l$ & $\operatorname{tVaR}_h$ \\ 
\midrule
%vanilla & 10.11 & 0.24 & 4.41 & 0.62 & 31.98 & & 8.80 & 0.23 & 4.37 & 0.51 & 30.53 \\
%TAF & 9.97 & 0.34 & 3.52 & 0.77 & 4.08 & & 8.48 & 0.40 & 4.08 & 0.91 & 4.58 \\
%gTAF & 9.96 & 0.53 & 3.59 & 1.12& 1.84 & & 8.38 & 0.47 & 3.40 & 0.98 & 5.07\\
%mTAF & 9.95 & 0.33 & 2.21 & 1.07 & 2.46 & & 8.36 & 0.34 & 2.49 & 1.21 & 6.00\\
vanilla & 10.25 & 0.25 & 4.19 & 0.60 & 29.56 & & 8.98 & 0.25 & 4.30 & 0.54 & 28.55 \\
TAF & 10.15 & 0.37 & 3.55 & 0.78 & 4.21 & & 8.69 & 0.42 & 4.05 & 0.89 & 4.36 \\
gTAF & 10.12 & 0.55 & 3.24 & 1.16& 2.67 & & 8.57 & 0.50 & 3.38 & 0.98 & 5.55\\
mTAF & 10.11 & 0.26 & 2.22 & 0.59 & 2.94 & & 8.55 & 0.25 & 2.60 & 0.57 & 6.74\\
\midrule 
copula & 9.75 & 0.20 & 1.23 & 0.45 & 2.22 & & 9.75 & 0.19 & 1.43& 0.46 & 3.49 \\
 \bottomrule
\end{tabular}
\end{center}
\end{table*}
\section{Experimental Analysis}
To investigate the benefits of the derived methods, we perform an empirical analysis on synthetic data (Section~\ref{sec:synthexp}) to understand the behavior and benefits of mTAF and gTAF in comparison to other flows models. In Section~\ref{sec:weather}, we demonstrate how we can exhibit the heavy-tailed behavior of the model to sample new extremes in weather and climate example. We provide a PyTorch implementation and the code for all experiments, which can be accessed through our public git repository\footnote{https://github.com/MikeLasz/marginalTailAdaptiveFlow}.
\subsection{Synthetic Experiments}\label{sec:synthexp}
In this series of experiments, we compare the performance of 4 different NSFs: vanilla flow (i.e. $\rvz \sim \mathcal{N}(0,I)$), TAF (i.e. $\rvz \sim t_{\hat{\nu}}(0, I)$), %a generalization of TAF (gTAF) in which all marginals have their own independent degree of freedom (i.e. $\rvz_j\sim t_{\hat{\nu}_j}(0,1)$), 
gTAF, and mTAF with fixed degrees of freedom. The data is generated by a 8-dimensional Gaussian copula with $d_h \in\{1,4\}$ heavy-tailed marginals with tail index $\nu=2$. Details about the data generation can be found in Section~\ref{sec:synth_data}. 
As an oracle baseline, we fit a Gaussian copula to the data. 
To measure the overall fit of the model, we track the negative log-likelihood loss $L$.
Since it is well-known that a good likelihood fit is not equivalent to high sampling quality \citep{theis2015note}, we take a closer on the samples in the tails of the distribution by considering the following three metrics.\footnote{All metrics are formally defined in Section~\ref{sec_suppsynthexp}.} 
%However, it is well-known that a good likelihood fit is not equivalent to adequate sample generation quality \citep{theis2015note}. Hence, we take a closer look on the sample quality in the tails of the distribution by considering the following three metrics. 
\begin{enumerate}
    \item \textbf{Tail Value at Risk} ($\operatorname{tVaR}_{\alpha}$), also known as \textit{expected shortfall}, %:
    %\begin{equation*}
    %    \operatorname{tVaR}_\alpha := \frac{1}{1-\alpha} \int_{\alpha}^1 \hat{F}^{-1}(u) \, du \enspace , 
    %\end{equation*}
    which is the expectation of the quantile function given that we consider a quantile level larger than $\alpha$. By calculating the absolute difference between $\operatorname{tVaR}_\alpha$ based on the data distribution and based on synthetic samples generated by the flow, we obtain the $\operatorname{tVaR}$-difference. 
    \item \textbf{Area under log-log plot} $\operatorname{Area}$,
    %defined by
    %\begin{equation*}
        %\operatorname{Area} := \sum_{i=1}^n \biggl\vert \log \bar{F}^{-1}_{\operatorname{data}} \bigl(\frac{i}{n} \bigr) - \log \bar{F}^{-1}_{\operatorname{flow}} \bigl( \frac{i}{n} \bigr) \biggr\vert \log \frac{i + 1}{i} \enspace ,
    %\end{equation*}
    %where $\bar{F}^{-1}_{\operatorname{data}}, \, \bar{F}^{-1}_{\operatorname{flow}}$ denote the inverse empirical complementary CDFs given by the test data and the flow samples, respectively. $\operatorname{Area}$
    which can be interpreted as a reweighted version of $\operatorname{tVaR}$ that puts more weight on the extremes.
    \item \textbf{Synthetic Tail Estimates} are the marginal tail estimators based on synthetic data generated by the flow. Ideally, we expect the NF to generate samples according to the true tail-index. 
\end{enumerate}
We measure $\operatorname{tVaR}$-differences componentwise and average over all heavy-tailed and light-tailed components to obtain $\operatorname{tVaR}_h$ and $\operatorname{tVaR}_l$, respectively. The same applies for $\operatorname{Area}$.

We fit each model $25$ times to $3$ different synthetic distributions and summarize the numeric results in Table~\ref{table:synth_metrics}\footnote{Standard deviations for one synthetic distribution are shown in Table~\ref{table:synth_metrics_stds}.}. It is no surprise that vanilla performs well for light-tailed components but suffers on capturing the tails of the heavy-tailed distributions. In the setting $d_h =1$, we observe the same behavior for TAF, which could be attributed to having only one joint tail parameter $\nu$ to model the seven light-tailed and one heavy-tailed marginal. \name{} does not always perform best but it manages to find a good balance between the fit on the light-tailed, as well as heavy-tailed components. This is demonstrated more clearly when considering the tail estimation indices of the marginals (Figure~\ref{fig:synth_tail}). 
In this figure, we summarize the estimated marginal tail behavior of the learned model in a confusion matrix. While most of the generated marginals are classified as light-tailed for vanilla and TAF, mTAF is able to recover the marginal tail behavior almost perfectly. We make similar observations in the other settings and when using MAFs instead of NSFs (Section~\ref{sec_suppsynthexp}). 
%TAF generates multiple heavy-tailed marginals, which are supposed to be light-tailed (Figure~\ref{fig:synth_tail}, left), while, most importantly, neither vanilla, nor TAF produce heavy-tailed marginals in the heavy-tailed components of the distribution (Figure~\ref{fig:synth_tail}, right). Although gTAF is able to generate heavy-tailed samples, it also generates heavy-tailed samples in the light-tailed components. In addition, we present QQ-plots in Section~\ref{sec_suppsynthexp} in the Appendix, which visualize that only \name{} and gTAF manage to generate heavy-tailed samples close to the distributions tail.
We extend our analysis by providing a $50$-dimensional example in Section~\ref{sec_suppsynthexp}.

%To investigate the benefits of \name{} we perform an empirical analysis in which we conduct experiments on synthetic heavy-tailed data. 
%We construct synthetic $8$-dimensional data sets and investigate eight different settings:
%The marginals are chosen to contain $h \in \{ 1, 4 \}$ mixtures of two $t$-distributions, two Gaussian distributions, two mixtures of three Gaussian distributions, and $4-h$ mixtures of two Gaussian distributions. The degree of freedom of all $t$-distributions is either $\nu =2$ or $\nu =3$. In all mixtures all mixture components are weighted equally, means are uniformly sampled from $[-−4 ,4]$, and standard-deviations are sampled from $[1,2]$.
%Using a Gaussian Copula, we construct a complex joint distribution with mixed-tailed marginals.
%A detailed description of the data set generation can be found in Section~\ref{sec:synth_data} in the Appendix. 
%We apply the \name on NSF and on MAF, where the results of the latter are shown in Section~... in the Appendix. 
%Details about hyperparameters can be found in Section~\ref{sec:hyperparameters} in the Appendix. 
%We provide a PyTorch implementation and the code for all experiments along the submission.

\begin{figure}[ht]
\vskip 0.2in
\begin{center}
\centerline{\includegraphics[width=0.24\columnwidth]{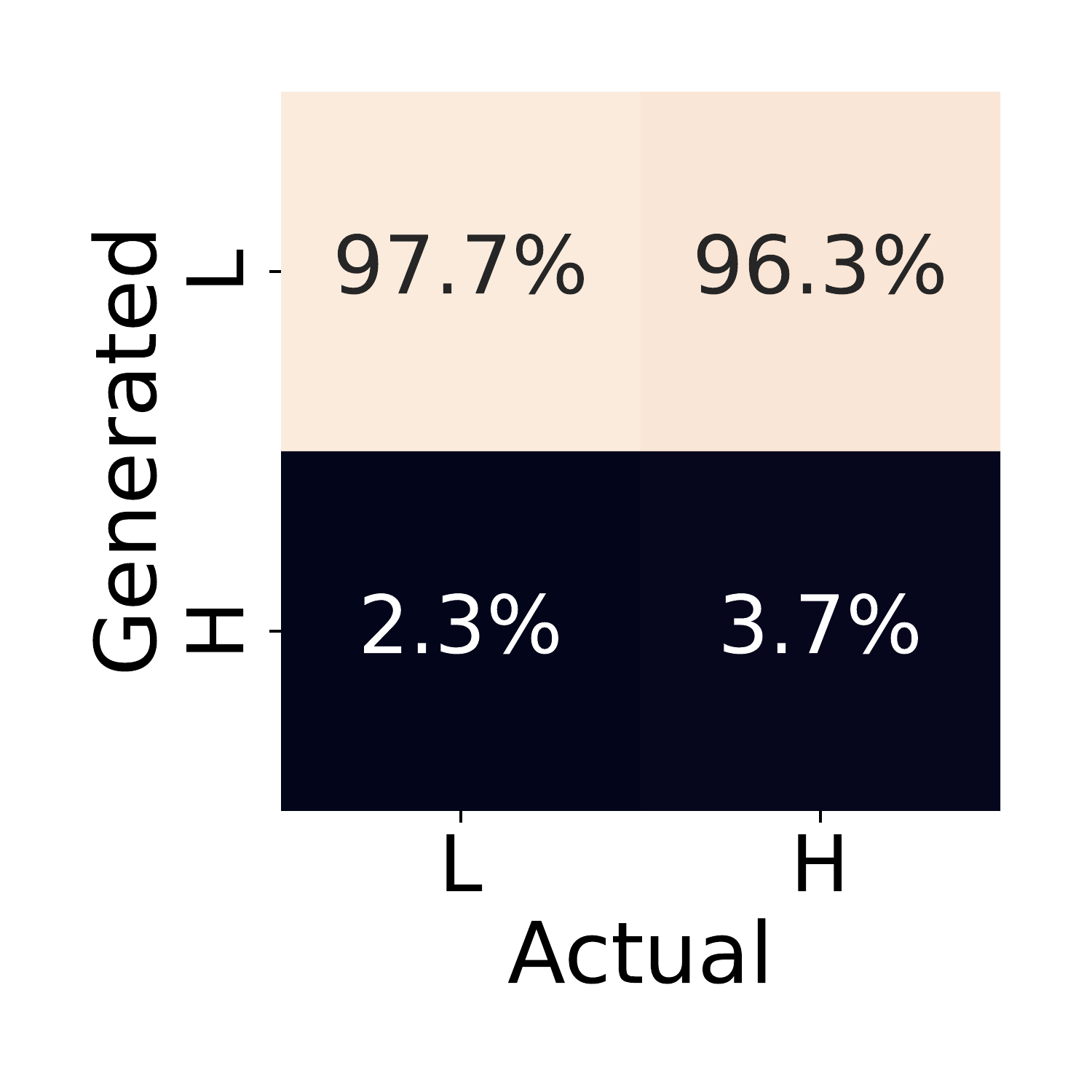} 
\includegraphics[width=0.24\columnwidth]{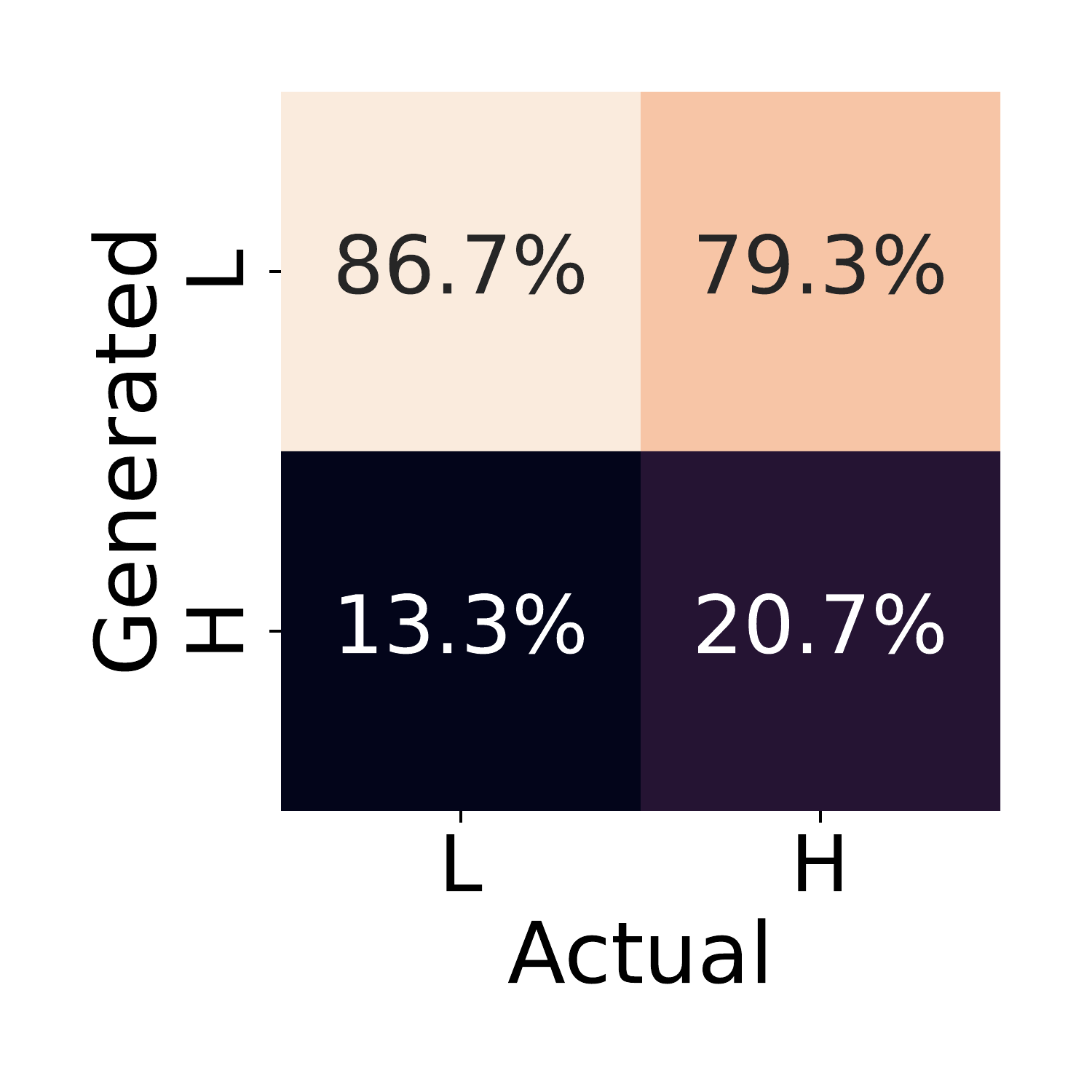}
\includegraphics[width=0.24\columnwidth]{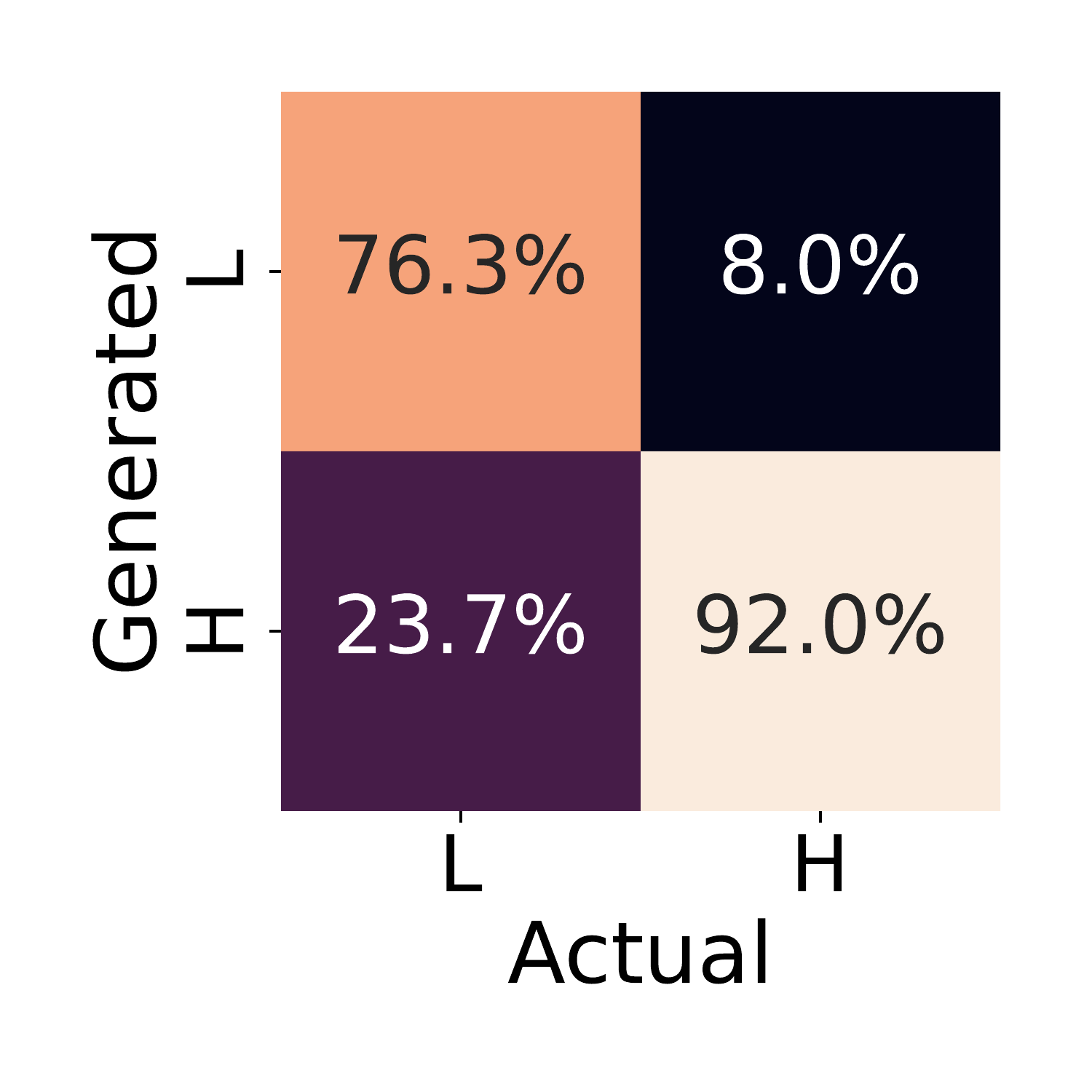}
\includegraphics[width=0.24\columnwidth]{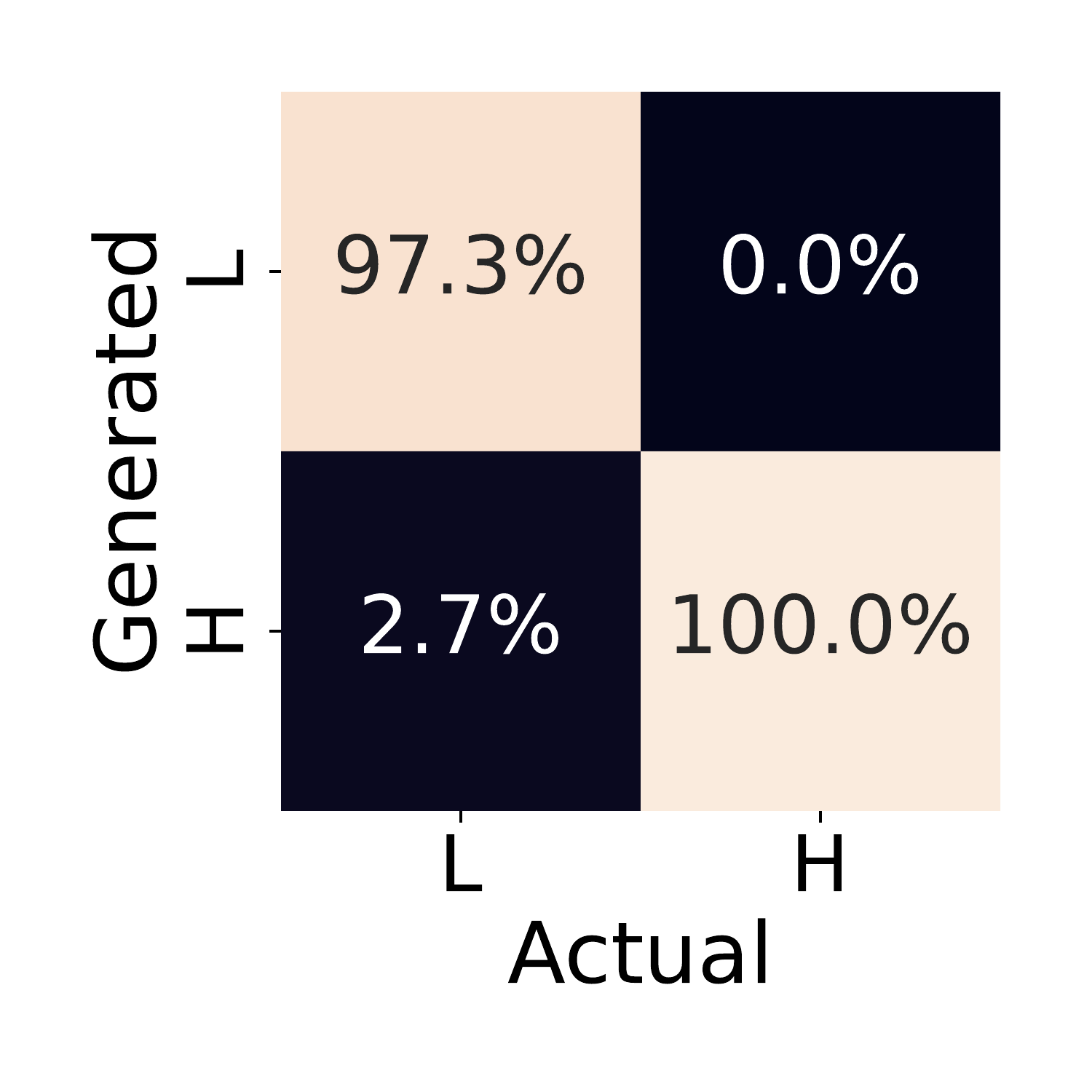}
}
\caption{Marginal tail estimation based on synthetic flow samples for $d_h =4,\, \nu=2$ of vanilla, TAF, gTAF, and mTAF (from left to right). We classify marginals whose tail estimator is less than $10$ as heavy-tailed, otherwise it is classified as light-tailed.}
\label{fig:synth_tail}
\end{center}
\vskip -0.2in
\end{figure}

\subsection{Modeling Climate Data} \label{sec:weather}
\begin{figure}[h!]
%\vskip 0.2in
\begin{center}
\centerline{\includegraphics[width=\columnwidth]{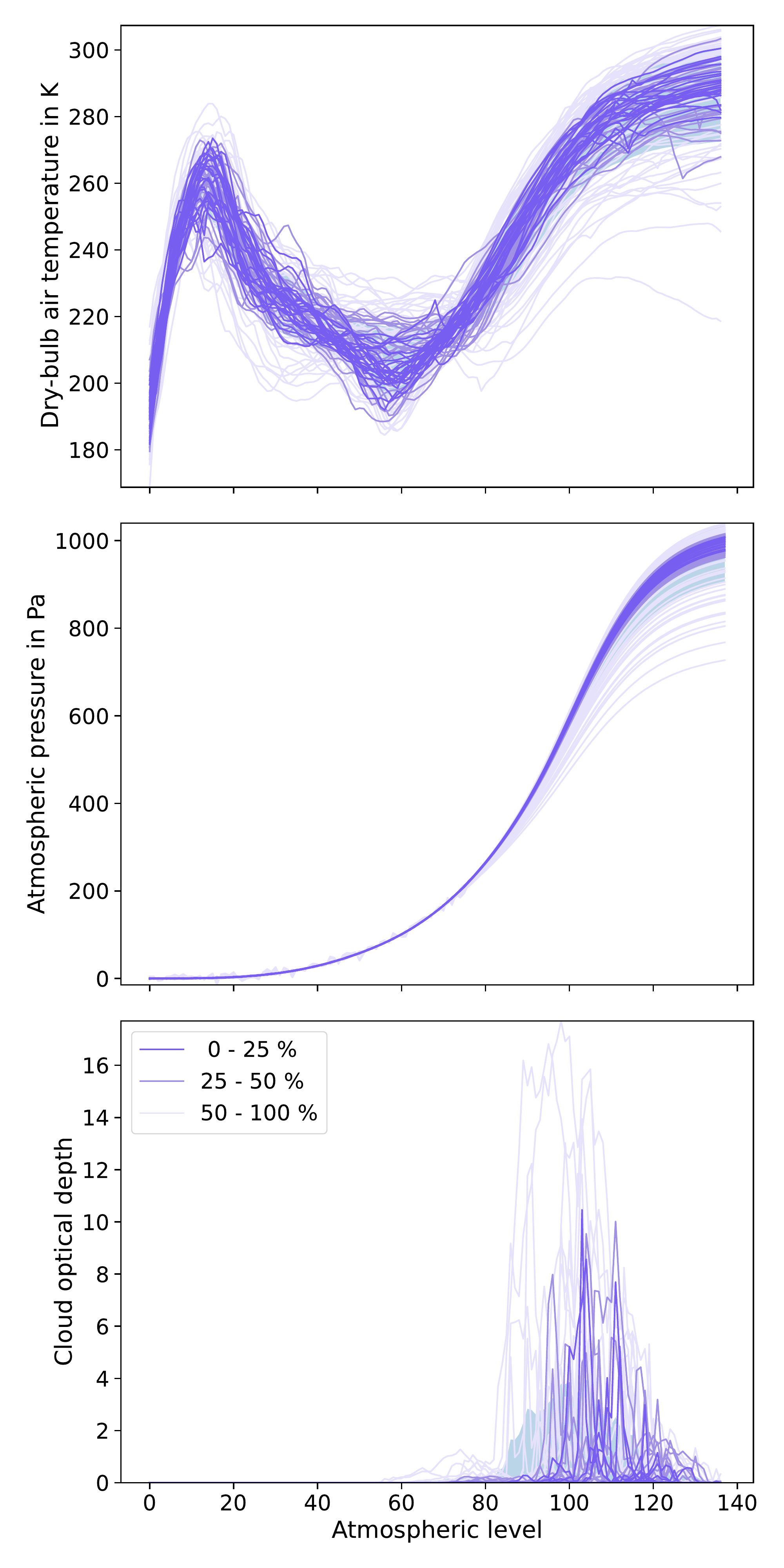}}
\caption{Synthetic flow samples using mTAF, where we clipped the lower-values of the cloud-optical depth at 0. The corresponding negative log-likelihood is $-2121.46$. The profiles are ordered using band depth statistics \citet{depth_statistics} and the shaded areas represent standard deviations.} 
\label{fig:mTAF_samps}
\end{center}
\vskip -0.2in
\end{figure}
We demonstrate the benefits of the proposed methods on an example, in which tail modeling is crucial: 
%it is crucial to model the tails:
we apply mTAF to generate new data following the distribution of the EUMETSAT Numerical Weather Prediction Satellite Application Facility (NWP-SAF) dataset \citep{weather_data}. The data set consists of $25\,000$ measurements of $3$ meteorological quantities, measured on different atmospheric levels. Considering each measurement in each atmospheric level, we obtain a $412$-dimensional dataset, which we visualize in Figure~\ref{fig:weather_true} in the Appendix. We train a vanilla, TAF, gTAF, and mTAF model using NSFs to fit the weather data.
All architectural details are listed in Section~\ref{sec:supp_climate}. 
%Due to large dependencies among the components and the natural autoregressive ordering, we decided to delibertely choose a larger set of heavy-tailed components, while making the degree of freedom learnable\footnote{Further architectural details are shown in Section~\ref{sec:supp_climate}}.
\begin{table*}
\caption{Average test loss on the NWP-SAF dataset. We average over 25 trials per model and show the standard deviations in brackets.
\label{table:loss_weather}}
\begin{center}
\begin{tabular}{@{}lcccc@{}}
\toprule
& vanilla & TAF & gTAF & mTAF \\
\midrule 
$L$ & $-2101.91 (\pm 9.44)$ & $-2110.56 (\pm 7.87)$ &$-2113.48(\pm 7.93)$ & $-2121.38 (\pm 10.91)$\\ 
 \bottomrule
\end{tabular}
\end{center}
\end{table*}

Qualitatively, all generated profile lines appear reasonable, see Figure~\ref{fig:mTAF_samps} and Section~\ref{sec:supp_climate}. From a quantitative perspective, we observe that mTAF achieves the smallest negative log-likelihood loss, as displayed in Table~\ref{table:loss_weather}. We extend the quantitative analysis by investigating $1$-dimensional random projections of the data and flow samples\footnote{Note that in contrast to the synthetic experiments, this data has much more dimensions. Quantities such as the Area under log-log plot depend on the rare tail-samples, and hence, its estimation requires many samples. This requirement is further emphasized in the high-dimensional setting, which is why we resort to these $1$-dimensional projections.}. To do so, we follow the same procedure as \citet{meyer2021copula}, i.e. we generate random weights $w_1, \dots, w_{100} \sim U([0,1]^{412})$ and generate $100$ random $1$-dimensional data sets $X_{\operatorname{flow}}^{(j)}:= \{\langle w_j, T(z_i) \rangle: \; z_i\sim \bm{z} \}$ and $X_{\operatorname{data}}^{(j)}:=\{\langle w_j, x_i \rangle: \; x_i \in X \}$. For each $j$, we can compare statistics such as the means, standard deviations, and quantiles of $X_{\operatorname{flow}}^{(j)}$ and $X_{\operatorname{data}}^{(j)}$, which we visualize in Figure~\ref{fig:weather_stats}. While all methods are able to fit the mean very well, only mTAF generates data that obeys the same standard deviation and extreme quantiles.  
\begin{figure*}[h!]
\vskip 0.2in
\begin{center}
\begin{minipage}{0.3\textwidth}
\centering
\centerline{\includegraphics[width=\columnwidth]{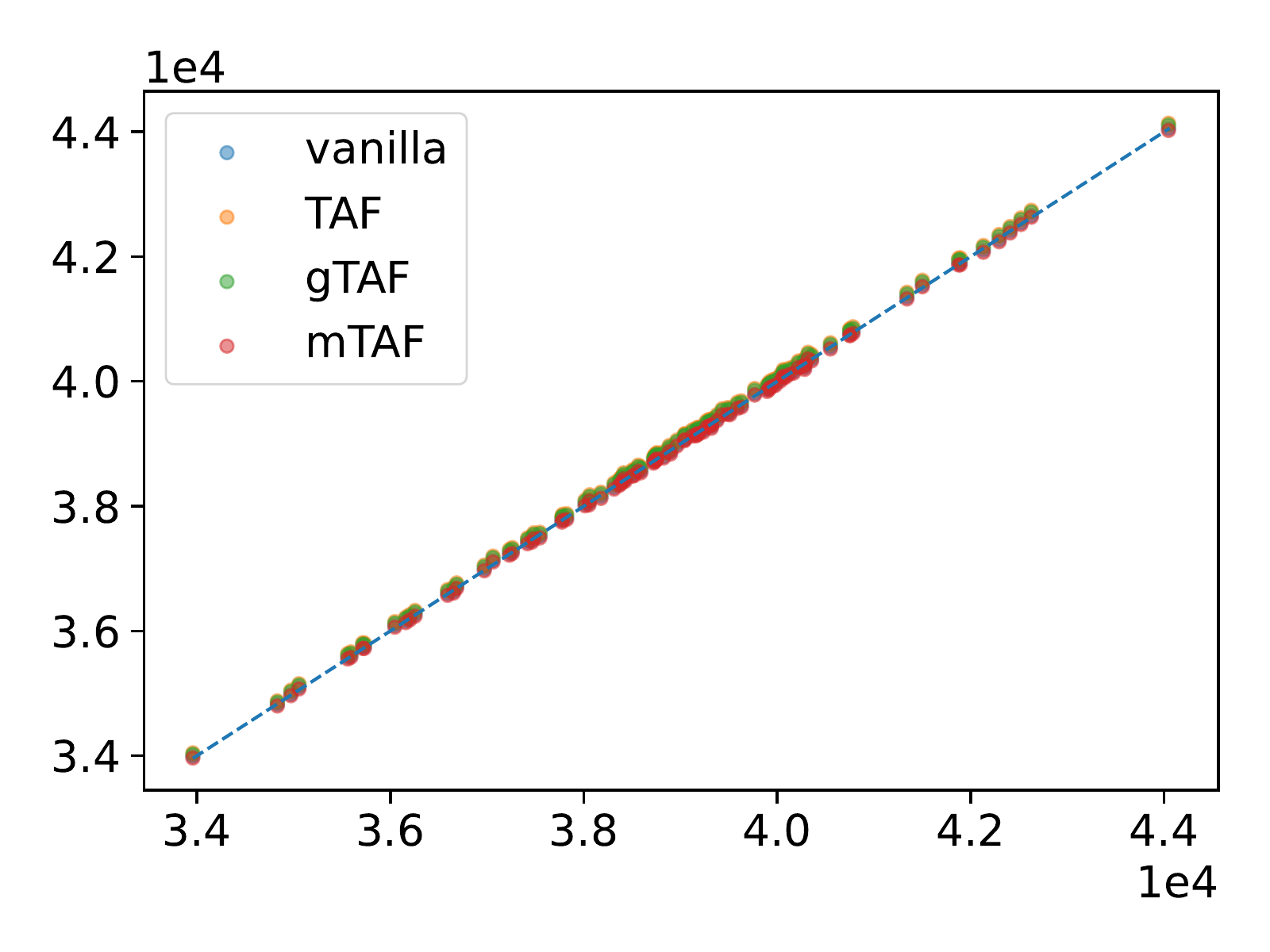}}
\subcaption{Mean.}
\label{fig:proj_mean}
\end{minipage}
\begin{minipage}{0.3\textwidth}
\centering
\centerline{\includegraphics[width=\columnwidth]{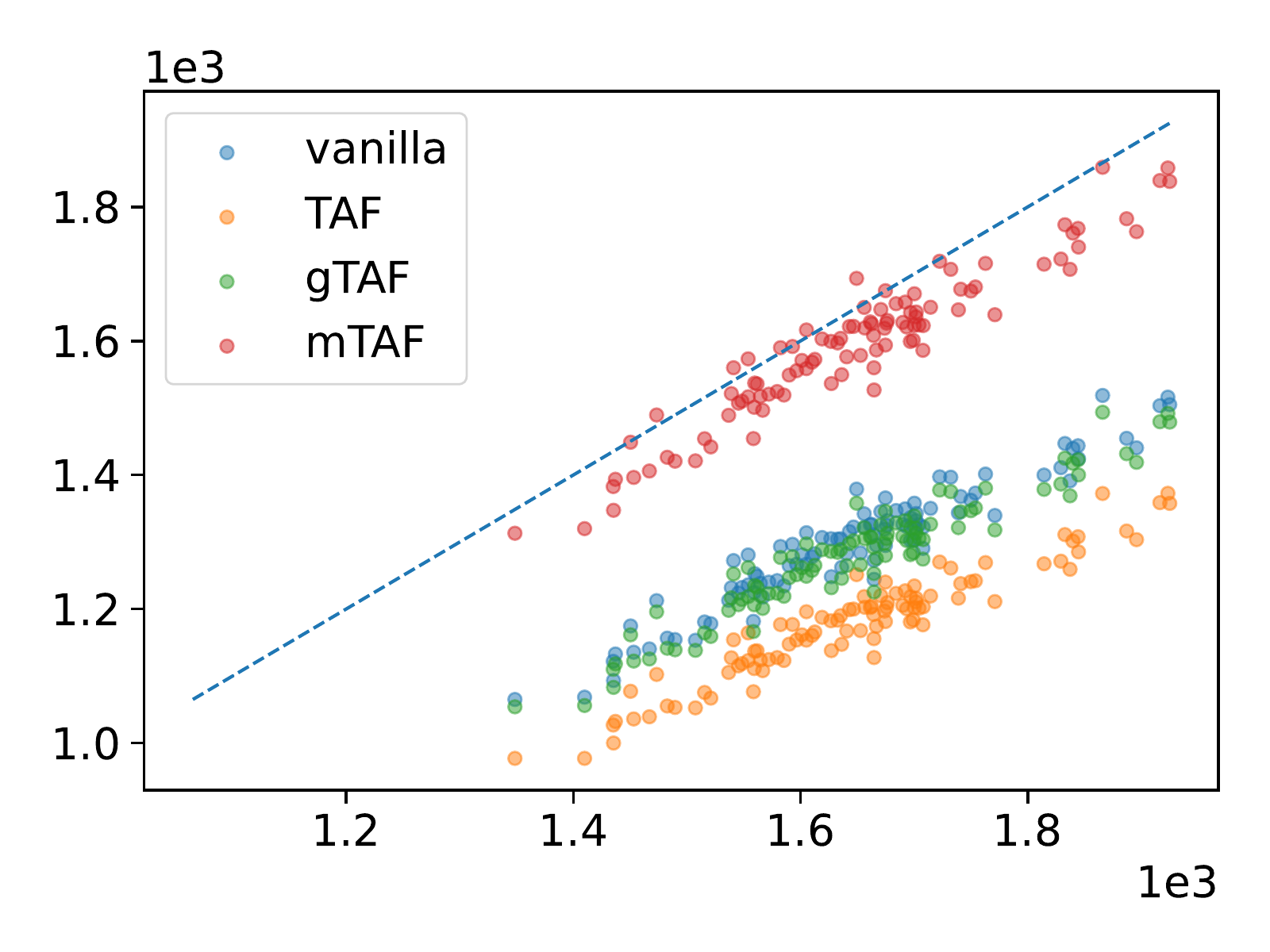}}
\subcaption{Standard Deviation}
\label{fig:synth_taildf3h4}
\end{minipage}
\begin{minipage}{0.3\textwidth}
\centering
\centerline{\includegraphics[width=\columnwidth]{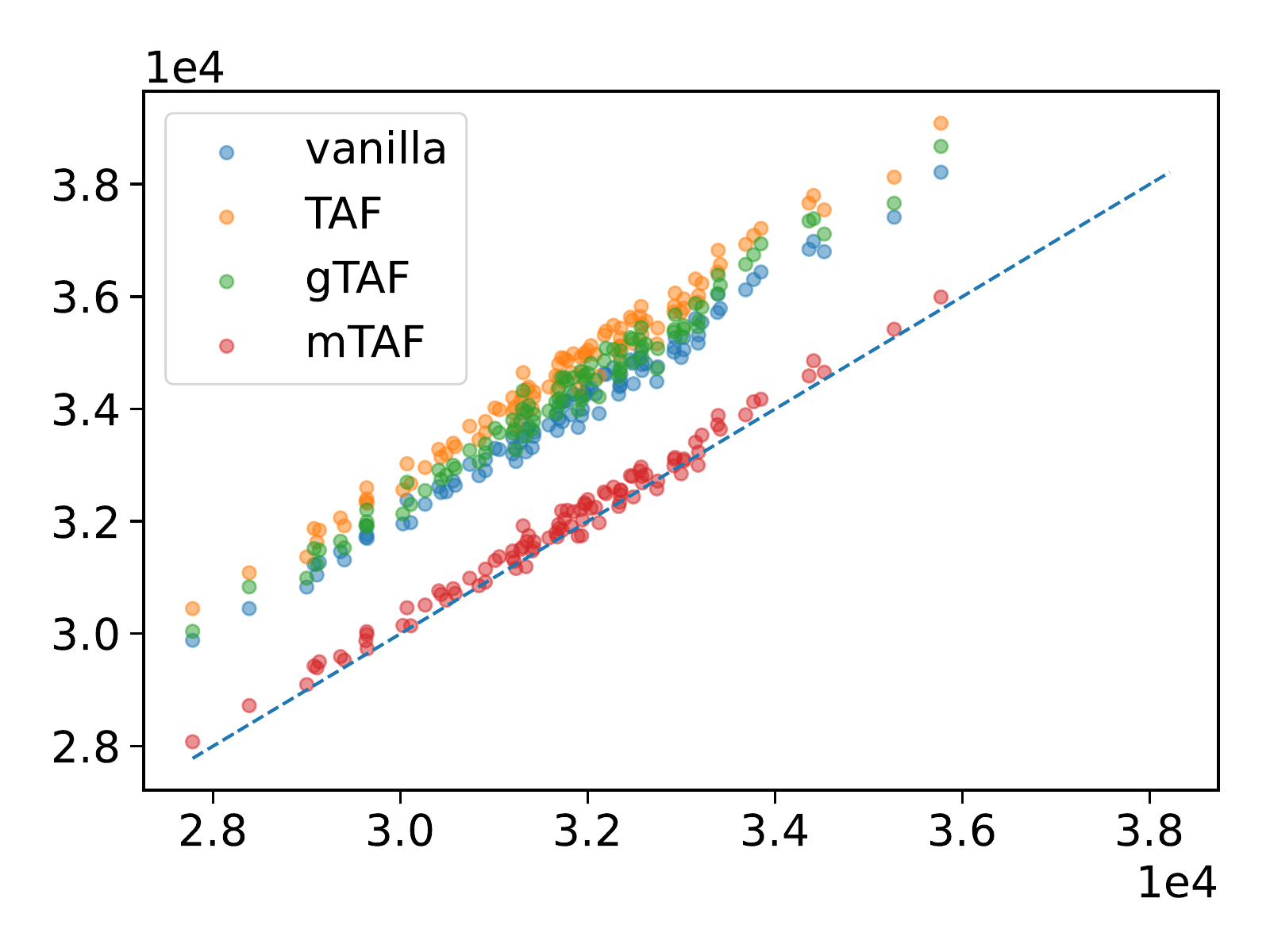}}
\label{fig:synth_taildf3h1}
\subcaption{1$\%$-quantile}
\end{minipage}
\caption{Statistics from 1-dimensional random projections of the data and flow samples. Each point represents the statistic $(S(X_{\operatorname{data}^{(j)}}), S(X_{\operatorname{flow}^{(j)}}))$, which are calculated using the random weight $w_j \sim U([0,1]^{412})$. In the ideal case, all points lie on the dotted diagional line.}
\label{fig:weather_stats}
\end{center}
\vskip -0.2in
\end{figure*}

%Interestingly, we observe that during training a majority of the tail indices converge to larger values, leaving only a few extremely heavy-tailed components\footnote{$4$ out of $412$ total components have a tail index less than $2.5$}. We conjecture that these latent components are of major importance for generating tail samples. To investigate this hypothesis, we generate extreme base samples from these components and append the remaining base components, which we sample as usual. By pushing these samples through the flow, we expect to generate tail samples of the target distribution, which we visualize in Figure~\ref{fig:climate_tailsamp}. Compared to gTAF, \name{} (Figure~\ref{fig:weather_synthsamps}) achieves the overall best test log-likelihood but the sampled profiles look less smooth and seem to be too condensed at the light-tailed components, whereas gTAF manages to find a sweet-spot between realistic body samples and extremal samples. % Similarly, \name{} is also able to generate tail samples but with less smooth profiles, which we attribute to the generality of $C$ in~\eqref{eq_blockmatrix_main}. 
%Note that in contrast to gTAF and \name{}, vanilla and TAF are not amenable to the described selection procedure in the latent space due to their spherical base distribution. Moreover, we find that only gTAF and \name{} are able to reliably generate extreme samples for the most heavy-tailed component of the model, which is the cloud optical depth. 
More details and further comparisons are provided in Section~\ref{sec:supp_climate}.

%To demonstrate the application of \name{} in practice, we evaluated its performance on $3$ standard datasets. Details about the datasets and the employed preprocessing can be found in \citet{papamakarios2017masked}. Similar to the empirical analysis in Section~\ref{sec:synthexp}, we compare the performances of a vanilla NSF, TAF, and of the proposed \name{} using NSFs. All architectural details are deferred to Section~\ref{sec:hyperparameters} in the Appendix. 

\section{Limitations and Extensions}
Our theoretical findings from Section~\ref{sec:mTAFandTheory} provide a solid foundation on learning heavy-tailed generative models. Nonetheless, there exist limitations that are worth being addressed in future works.
\paragraph{Asymmetric Tail Behavior} is a widespread property of real distributions in which only one of the tails (lower or upper tail) is heavy-tailed whereas the other is light-tailed. For instance the cloud optical depth in the climate example (Section~\ref{sec:weather}) has just heavy upper tails since cloud-optical depth cannot drop below $0$. One way to potentially solve this issue would be to allow for even more flexible base distributions using composite models (see for instance \citet{composite_models} and the references therein). Recently, COMET Flows \cite{mcdonald2022comet} employed those composite models to learn the lower tail, the upper tail, and the body of each marginal separately to improve on the performance of a variant of copula flows \cite{wiese2019copula}. 
\paragraph{Tail Dependencies} are a popular property investigated in areas such as financial risk analysis, and which quantify the dependency of two marginals given that a tail event occurred. More precisely, we define the tail dependency between the marginals $\rvx_i, \; \rvx_j$ as 
\begin{equation*}
    \lambda_{i,j} = \lim_{u\rightarrow 1^-}\mathbb{P}\bigl(x_j > F_{j}^{-1}(u) \vert x_i > F_{i}^{-1} (u) \bigr) \enspace ,
\end{equation*}
where $F_{j}^{-1}$ and $F_{i}^{-1}$ are the quantile functions of $\rvx_j$ and $\rvx_i$, respectively.
There is a large body of theoretical works revolving around copulae to model tail dependencies \cite{joe2014dependence}. Hence, a possible extension might be to replace the mean field assumption in the base distribution by a copula distribution, i.e. by dependent marginals with PDF 
\begin{equation*}
    q(z) = c\bigl( 
    F_{1}(z_1), \dots  , F_{D}(z_D) \bigr) \prod_{j=1}^D q_j (z_j) \enspace ,
\end{equation*} 
where $F_j$ are the CDFs and $q_j$ are the marginal PDFs of $\rvx_j$ for each $j$.

\section{Conclusion}
In this work, we deepen the mathematical understanding of the tail behavior of autoregressive normalizing flows. We note that the distribution we want to model may have heavy- as well as light-tailed marginals, and we prove that standard normalizing flows are not able to learn such distributions. Our developed theory shows how the marginal tail behavior of the target distribution of the flow relates to the tail behavior of its base distribution. Based on these theoretical findings we propose a new algorithm, which we refer to as \name. In particular, we initialize a base distribution based on statistical tail estimates of the target, and employ structured linearities that guarantee the correct tail behavior of the target distribution. Importantly, we extend our theory to Neural Spline Flows and provide a modification that casts them marginally tail-adaptive. Lastly, as a trade-off between the theoretically founded \name{} and the less restricted linearities in standard normalizing flows, we present a more flexible relaxation of \name{} called gTAF, which is able to converge to \name{} as a special case.
An in-depth empirical analysis with heavy- and light-tailed marginals shows that \name{} and gTAF are superior in terms of estimation performance, especially when it comes to learning a heavy tail of a distribution. In contrast to standard normalizing flows, only gTAF and \name{} are able to reliably generate samples that follow the desired tail behavior. 

In summary, we believe that the theoretical soundness and the ability to faithfully generate extreme samples is a major strength of gTAF and \name{}, which are both crucial properties in various applications such as finance, climate and related areas. 

\section*{Acknowledgements}
This work was supported by the Deutsche Forschungs-
gemeinschaft (DFG, German Research Foundation)
under Germany’s Excellence Strategy – EXC- 2092
CASA – 390781972. 
We also thank the anonymous reviewers for their careful reading and their useful suggestions. 

%\paragraph{Reproducibility Statement}
%This work concentrates on the mathematical understanding of normalizing flows, and in doing so, includes technical proofs, which might rely on abstract assumptions. Nonetheless, we put our best effort in making all the proofs and the derived theory in Section~\ref{sec:theoryforproofs} and~\ref{sec:proofs} precise, accessible, and correct. To avoid misconceptions, we dedicate Section~\ref{sec:notesonassumption} to the explanation and clarification of Assumption~\ref{assu:copuladensity}. We provide documented source code including code execution instructions, which allows the reproducibility of our empirical results. Additionally, all algorithms and hyperparameters are presented in Section~\ref{sec:algosanddetails}.
\clearpage
\bibliography{references}
\bibliographystyle{icml2022}

%%%%%%%%%%%%%%%%%%%%%%%%%%%%%%%%%%%%%%%%%%%%%%%%%%%%%%%%%%%%%%%%%%%%%%%%%%%%%%%
%%%%%%%%%%%%%%%%%%%%%%%%%%%%%%%%%%%%%%%%%%%%%%%%%%%%%%%%%%%%%%%%%%%%%%%%%%%%%%%
% APPENDIX
%%%%%%%%%%%%%%%%%%%%%%%%%%%%%%%%%%%%%%%%%%%%%%%%%%%%%%%%%%%%%%%%%%%%%%%%%%%%%%%
%%%%%%%%%%%%%%%%%%%%%%%%%%%%%%%%%%%%%%%%%%%%%%%%%%%%%%%%%%%%%%%%%%%%%%%%%%%%%%%
\newpage
\appendix
\onecolumn

\section{Theory and Proofs}
In this Section, we derive and prove our theoretical results. We start by presenting some preliminary results in Section~\ref{sec:theoryforproofs}, which help us providing the proof of our main result in Section~\ref{sec:proofs}. In Section~\ref{sec:notesonassumption}, we illuminate the technical Assumption~\ref{assu:copuladensity} and provide examples and intuitions. 
\subsection{Preliminary theoretical Results}\label{sec:theoryforproofs}
\begin{proposition}\label{prop:tailindex_dependsontails}
    Let $\rvx\in\sR$ be a random variable. Then, it holds that $\rvx$ has tail index less or equal than $\alpha$ iff. for any $\beta > \alpha$ and any $C>0$ it is 
    \begin{equation*}
        \int_{\vert x \vert > C} \vert x \vert^\beta p(x) dx = \infty \enspace .
    \end{equation*}
\end{proposition}
\begin{proof}
Let us first assume that $\rvx$ has tail index at most $\alpha$. Then, according to Definition~\ref{defi:tail_index}, we know that for any $\beta>\alpha$ and $C>0$
\begin{align*}
     \infty = \E_{\rvx}[\vert \rvx \vert^\beta] &=\int_{\sR} \vert x \vert^\beta p(x) dx 
     = \int_{\vert x \vert \leq C} \vert x \vert^\beta p(x) dx + \int_{\vert x \vert > C} \vert x \vert^\beta p(x) dx \\
     &\leq C^\beta \int_{\vert x \vert \leq C} p(x) dx  + \int_{\vert x \vert > C} \vert x \vert^\beta p(x) dx \enspace . 
\end{align*}
Since we assume $p$ to be continuous, we can bound $p$ on the compact interval $[-C, C]$, and hence, the first above integral must be bounded. Therefore, it is \begin{equation*}
        \int_{\vert x \vert > C} \vert x \vert^\beta p(x) dx = \infty \enspace .
    \end{equation*}
To prove the back-direction, let us consider a $\beta>\alpha$ and $C>0$. Very similarly to the forward-proof, we can now see that 
\begin{align*}
    \infty &= \int_{\vert x \vert > C} \vert x \vert^\beta p(x) dx \\
        &= C^\beta \int_{\vert x \vert \leq C} p(x) dx + \int_{\vert x \vert > C} \vert x \vert^\beta p(x) dx \\
        &= \E_{\rvx}\bigl[ \vert \rvx \vert^\beta \bigr] \enspace ,
\end{align*}
where the second equality follows from the finiteness of the integral. Since this follows for any $\beta> \alpha$, $\rvx$ must have tail index $\alpha$ or less.
    
\end{proof}
This simple result demonstrates that the tail index, as indicated by the name, depends on the tail of the distribution, i.e. on the PDFs behavior for large values $\vert \vx \vert> C$. This fact motivates why maximum likelihood estimations of the tail index, which depend on the whole distribution are biased. For more elaborate details, we refer to Section~9 in \citet{heavytails_book}.

In a similar fashion to the previous result, the next technical lemma states that unboundedness of the moment-generating function is due to the unboundedness of the integrand for tail events, i.e. for $\vz > \vz^\ast$. This little lemma turns out to be useful in the proof of Proposition~\ref{prop:allheavybad}. 
\begin{lemma}\label{prop:technicaltails}
Let $\rvz \in \sR$ be heavy-tailed. Then it holds for any $z^\ast \in \sR$ and $\lambda > 0$ that
\begin{equation*}
    \int_{z>z^\ast} e^{\lambda z} p_{\rvz}(z) dz = \infty \enspace .
\end{equation*}
\end{lemma}
\begin{proof}
    Since $\rvz \in \sR$ is heavy-tailed, we know that for all $\lambda>0$
    \begin{align*}
        \infty = m_{\rvz}(\lambda) &=
        \int_{\sR} e^{\lambda z} p_{\rvz}(z)dz \\ 
        &= \int_{z\leq z^\ast} e^{\lambda z} p_{\rvz}(z)dz + \int_{z>z^\ast} e^{\lambda z} p_{\rvz}(z)dz \\ 
        &\leq F_\rvz (z^\ast) e^{\lambda z^\ast} + \int_{z>z^\ast} e^{\lambda z} p_{\rvz}(z)dz \enspace ,
    \end{align*}
    where $F_\rvz$ is the CDF of $\rvz$. The last inequality follows from the fact that $\exp(\lambda z)$ is monotonic increasing in $z$. Since the first summand is bounded, it follows that the second summand must be unbounded. This completes the proof.
\end{proof}
Recall that in Proposition~\ref{prop:marg&normtail} we state that $j$-heavy-tailedness induces $\ell_2$-heavy-tailedness. In the following, we provide a formal proof of this result.
\begin{proof}[Proof of Proposition~\ref{prop:marg&normtail}]
    For this proof, we employ the equivalent definition of heavy-tailedness of $\rvx_j$ via the decay rate of its distribution function (see e.g. Lemma~1.1. in \citet{heavytails_book}), i.e.
    %Recall that heavy-tailedness of $\rvx_j$ according to Definiton~\ref{defi:heavytails} is equivalent to 
    \begin{equation}
        \limsup_{x_j \rightarrow \infty} \frac{1 - F_j(x_j)}{e^{-\lambda x_j}} = \infty \quad \text{for all }\lambda >0 \enspace , \label{eq:equivalentheavytail}
    \end{equation}
    where $F_j$ is the CDF of $\rvx_j$. Since $x_j \leq \Vert x\Vert$ for all $x\in\sR^D$, we can conclude that $F_j(a) \geq F_{\Vert x\Vert}(a)$ for $a\in\sR$. Therefore, 
    \begin{equation*}
        \frac{1-F_{\Vert \rvx\Vert}(a)}{e^{-\lambda a}} \geq \frac{1-F_{x_j}(a)}{e^{-\lambda a}} \rightarrow \infty \quad \text{for } a\rightarrow\infty \enspace .
    \end{equation*}
    According to the equivalent definition in~\eqref{eq:equivalentheavytail}, $\Vert \rvx \Vert$ is heavy-tailed, which proves that $\rvx$ is $\ell_2$-heavy-tailed.
\end{proof}
The following is a well-known implication of the change of variables formula and the integration rule by substitution, which we are going to apply in the subsequent proofs. 
\begin{lemma}[Substitution in the Moment-Generating Function]\label{lemma:cov}
    Let $T$ be a diffeomorphism such that $T(\rvz)=\rvx$ for some random variables $\rvx, \rvz \in \sR^D$. Then, we can rewrite
    \begin{equation*}
    \int_{\sR^D} e^{\lambda \vx} p(\vx) d\vx = \int_{\sR^D} e^{\lambda  T(\vz)} q(\vz) d\vz \enspace . 
    \end{equation*}
\end{lemma}
For completeness, we give a brief proof of this result.
\begin{proof}
    Using the change of variables formula, see \eqref{eq:changeofvar}, we can write 
    \begin{equation*}
        \int_{\sR^D} e^{\lambda \vx} p(\vx) d\vx = \int_{\sR^D} e^{\lambda \vx} q\bigl(T^{-1}(\vx) \bigr) \bigl\vert \det J_{T^{-1}(\vx)} \bigr\vert d\vx \enspace .
    \end{equation*}
    Now, we can rewrite $\exp(\lambda \vx)=\exp(\lambda T\bigl(T^{-1}(\vx)\bigr)$ and substitute $\vz = T^{-1}$. Integration by substitution completes the proof.
\end{proof}

Next, we present how we can use copulae to reformulate a multivariate PDF.
\begin{definition}[Copula] \label{defi:copula}
A copula is a multivariate distribution with cumulative distribution function (CDF) $C:[0,1]^D \rightarrow [0, 1]$ that has standard uniform marginals, i.e. the marginals $C_j$ of $C$ satisfy $C_j \sim U[0,1]$.
\end{definition}

\begin{theorem}[Sklar's Theorem] Taken from \citet{elementsofcopula}. \label{sklarstheorem}
\begin{enumerate}
    \item For any $D$-dimensional CDF $F$ with marginal CDFs $F_{1},\dots, F_{D}$, there exists a copula $C$ such that
    \begin{equation}
    \label{eq:sklar_df}
        F(\vz) = C\bigl( F_{1}(\vz_1), \dots, F_{D}(\vz_D) \bigr) \enspace 
    \end{equation}
    for all $\vz\in\mathbb{R}^D$. The copula is uniquely defined on $\mathcal{U}:=\prod_{j=1}^D \operatorname{Im}(F_{j})$, where $\operatorname{Im}(F_{j})$ is the image of $F_{j}$. For all $\vu\in\mathcal{U}$ it is given by 
    \begin{equation*}
        C(\vu) = F \bigl( F_{1}^{\leftarrow}(\vu_1), \dots , F_{D}^{\leftarrow}(\vu_D) \bigr)\enspace ,
    \end{equation*}
    where $F_{j}^{\leftarrow}$ are the right-inverses of $F_{j}$.    
    \item Conversely, given any $D$-dimensional copula $C$ and marginal CDFs $F_{1}, \dots F_{D}$, a function $F$ as defined in \eqref{eq:sklar_df} is a $D$-dimensional CDF with marginals $F_{1},\dots ,F_{D}$.
\end{enumerate}
\end{theorem}
Therefore, if $F$ is absolutely continuous, we can differentiate \eqref{eq:sklar_df} to obtain the PDF of $\rvz$
\begin{equation*}
    q(z)=c\bigl(F_1(z_1), \dots , F_D(z_D)\bigr) \prod_{j=1}^D q_j(z_j) \enspace ,
\end{equation*}
where $c$ denotes the PDF of the copula $C$. 

Lastly, we present the following asymptotic behavior of the inverse CDF of a standard Gaussian distribution, which we use in Section~\ref{sec:notesonassumption} to explain Assumption~\ref{assu:copuladensity}.
\begin{lemma}[Asymptotic Behavior\footnote{The idea of the proof is due to \citet{stackexchange-asymptotics}} of $\Phi^{-1}(1-y)$] \label{lemma:asymp}
    Denote by $\Phi$ the CDF of a standard Gaussian distribution. Then, it holds for the inverse of $\Phi$ that 
    \begin{equation*}
        \Phi^{-1}(1-y)\sim \sqrt{-2 \log(y)} \quad \text{ for } y\rightarrow 0 \enspace .
    \end{equation*}
\end{lemma}
\begin{proof}
    First, we note that 
    \begin{equation*}
        \Phi(x) = \frac{1}{2} + \frac{1}{2} \operatorname{erf}\biggl( \frac{x}{\sqrt{2}} \biggr) \sim 1 - \frac{1}{x\sqrt{2}} e^{-x^2/2}\enspace ,
    \end{equation*}
    which is a well-known asymptotic \citep{book_formulas}. Here, $\operatorname{erf}$ denotes the error-function. Rearranging terms gives
    \begin{equation*}
        \log \bigl(1-\Phi(x)\bigl) \sim -\log\Bigl(x\sqrt{2\pi}\Bigr) - \frac{x^2}{2} \sim - \frac{x^2}{2} \quad \text{ as } x\rightarrow\infty \enspace . 
    \end{equation*}
    Finally, we can invert the above asymptotic equation to obtain
    \begin{equation*}
        \Phi^{-1}(y) = \sqrt{-2\log(1-y)} \quad \text{ for }y \rightarrow 1
    \end{equation*}
    or equivalently
    \begin{equation*}
        \Phi^{-1}(1-y) = \sqrt{-2\log(y)} \quad \text{ for }y\rightarrow 0 \enspace .
    \end{equation*}
\end{proof}
\subsection{Proof of the Main Result}\label{sec:proofs}
The proof of Proposition~\ref{prop:allheavybad} relies on lower-bounding the moment-generating function of each marginal $\rvx_j$. In order to derive such a bound of a multivariate integral, we rewrite the joint distributions $q_{\leq j}$ using their copula densities: 
\begin{equation*}
    q_{\leq j}(z_{\leq j}) = c_j\bigl(F_1(z_1), \dots , F_j(z_j)\bigr) \prod_{i< j} q_i(z_i) 
\end{equation*}
for any $j\in\{1,\dots, D\}$ and for corresponding copula density $c_j$. Our proof relies on the following technical condition on the decay rate of the copula densities. 
\begin{assumption}[Bounding the Marginal Decay Rate of the Copula Densities]\label{assu:copuladensity}
        For all $j\in\{1,\dots, D\}$ and $\lambda >0$ there exists a compact set $\sS\subset \sR^{j-1}$ with positive (Lebesgue-)mass, a constant $z_j^\ast>0$, a scaling constant $s>0$, and a function $f(z_{<j})<\lambda \sigma(z_{<j})$ for $z_{<j}\in\sS$ such that 
        \begin{equation}\label{eq:assu}
            c_j\bigl(F_1(z_1),\dots ,F_j(z_j)\bigr) \geq se^{-f(z_{<j})z_j} \quad \text{ for }z_j>z_j^\ast \text{ and } z_{<j}\in\sS \enspace ,
        \end{equation}
        where $c_j$ is the copula density of $q_{\leq j}$
\end{assumption}
    This assumption sets a bound on the decay rate of the copula density with respect to $z_j$. We clarify this assumption in Section~\ref{sec:notesonassumption} with additional examples.
    
    Now, we set all preliminaries to prove Proposition~\ref{prop:allheavybad}.
\begin{proof}[Proof of Proposition~\ref{prop:allheavybad}]
We start by considering the case $j=1$. In this case it is $\rvx_1=\mu + \sigma \rvz_1$ and therefore \begin{align*}
        m_{\rvx_1}(\lambda) &= \int_{\sR} e^{\lambda x_1} p_1(x_1) dx_1 \\
        &= \int_{\sR} e^{\lambda (\mu_1 + \sigma_1 z_1 )} q_1(z_1) dz_1 \quad \text{(Lemma \ref{lemma:cov})} \\ 
        &= e^{\lambda \mu_1} \int_{\sR} e^{\lambda \sigma_1 z_1} q_1(z_1) dz_1 \enspace .
    \end{align*}
    Defining $\lambda^\prime:= \lambda \sigma_1>0 $, we can see that the last integral is unbounded due to the heavy-tailedness of $\rvz_1$, see Definition~\ref{defi:heavytails}. Therefore, $m_{\rvx_1}(\lambda)=\infty$ for all $\lambda>0$, which proves the heavy-tailedness of $\rvx_1$. \\
    Next, we consider the case $j>1$. Again, we examine the moment-generating function of $\rvx_j$. Define the $j$th canonical basis vector $v_j:=(0, \dots ,0,1,0,\dots, 0)^\top$. Then,\footnote{Note that for the sake of clarity, we leave out the index $j$ in $\mu_j$ and $\sigma_j$.}
    \begin{align}
        m_{\rvx_j}(\lambda) = m_{ v_j^\top \rvx} &=\int_{\sR^D} e^{\lambda v_j^\top x} p(x) dx  \quad \text{(LOTUS)}\nonumber\\
        &= \int_{\sR^D} e^{\lambda T_j(z_j, z_{<j})} q(z) dz \quad \text{(Lemma \ref{lemma:cov})} \nonumber \\
        &= \int_{\sR^j} e^{\lambda \mu(z_{<j}) + \lambda \sigma(z_{<j})z_j} q_{\leq j}(z_{\leq j}) dz_{\leq j} \nonumber \\
        &= \int_{\sR^{j-1}} e^{\lambda \mu(z_{<j})} \int_{\sR} e^{\lambda \sigma(z_{<j})z_j} q_{\leq j}(z_{\leq j}) dz_j dz_{< j} \enspace . \label{eq:proof-momentgen}
    \end{align}
    Using Sklar's Theorem (Theorem~\ref{sklarstheorem}), we can write any joint PDF as the product of marginals and a copula density $c_j$ such that 
    \begin{equation}
        \label{eq:copulaform}
        q_{\leq j}(z_{\leq j}) = c_j\bigl(F_1(z_1), \dots , F_j(z_j)\bigr) \prod_{i< j} q_i(z_i) \enspace .
    \end{equation}
    
    We plug~\eqref{eq:copulaform} into~\eqref{eq:proof-momentgen} to obtain
    \begin{align}
        m_{\rvx_j}(\lambda) &= \int_{\sR^{j-1}} e^{\lambda \mu(z_{<j})} q_{<j}(z_{<j}) \int_{\sR} e^{\lambda \sigma(z_{<j})z_j} c_j\bigl(F_1(z_1), \dots , F_j(z_j)\bigr) q_j (z_j) dz_j dz_{< j} \nonumber \\
        &\geq \int_{\sS} e^{\lambda \mu(z_{<j})} q_{<j}(z_{<j}) \int_{z_j > z_j^\ast} e^{\lambda \sigma(z_{<j})z_j} c_j\bigl(F_1(z_1), \dots , F_j(z_j)\bigr) q_j (z_j) dz_j dz_{< j} \enspace , \label{eq:proof_finished}
    \end{align}
    since all quantities within the integral are positive. Using Assumption~\ref{assu:copuladensity}, we can bound the inner integral of the above equation, which we denote by $A(z_{<j})$, and get 
    \begin{align*}
        A(z_{<j}) &\geq  s\int_{z_j > z_j^\ast} e^{(\lambda \sigma(z_{<j}) - f(z_{<j}))z_j} q_j (z_j) dz_j \\
        &= s\int_{z_j > z_j^\ast} e^{\lambda^\prime z_j} q_j (z_j) dz_j \quad (\text{define }\lambda^\prime := \lambda - \sigma(z_{<j}) - f(z_{<j}) \;) \\  &= \infty \quad \text{ for all }z_{<j}\in \sS \enspace ,
    \end{align*}
    due to the heavy-tailedness of $\rvz_j$ and Lemma~\ref{prop:technicaltails}. Since $\sS$ is compact, $\mu$ and $q$ are both continuous, and $q$ is positive, we deduce that $\exp(\lambda \mu(z_{<j})) q_{<j}(z_{<j})$ is lower-bounded (by a constant larger than $0$) in $\sS$. %The second derivation uses the continuity of $\mu$ and $q$. 
    Therefore, employing~\eqref{eq:proof_finished} and using that $\sS$ has positive mass, we can lower-bound the moment-generating function by $\infty$, which proves the heavy-tailedness of $\rvx_j$. In summary, $\rvx$ is $j$-heavy-tailed for all $j\in\{1,\dots ,D\}$.
\end{proof}

\subsection{Notes on Assumption~\ref{assu:copuladensity}}\label{sec:notesonassumption}
Assumption~1 might look troublesome at first sight,
but we will illustrate in this section that the condition is indeed very reasonable.
We will show how to verify it in simple examples,
and we will introduce a simpler, more intuitive sufficient condition for it.

First of all, let us present a restricted but more intuitive version of Assumption~\ref{assu:copuladensity}.
\begin{assumption}[Simplification of Assumption~\ref{assu:copuladensity}]\label{assu:simple}
    For all $j\in\{1,\dots, D\}$ and $\lambda>0$ it holds for $\sS:=[a,b]^{j-1}$ that there exist constants $z_j^\ast$ and $s>0$ such that 
        \begin{equation}\label{eq:assu_simp}
            c\bigl(F_1(z_1),\dots ,F_j(z_j)\bigr) \geq se^{-(\lambda_\sigma - \varepsilon) z_j} \quad \text{ for }z_j>z_j^\ast \text{ and } z_{<j}\in\sS \enspace ,
        \end{equation}
    where $c_j$ is the copula density of $q_{\leq j}$, $\lambda_\sigma$ is a lower bound of $\lambda \sigma(z_{<j})$, $\varepsilon>0$ is small such that $\lambda_\sigma - \varepsilon >0 $.
\end{assumption}
Let us summarize the simplifications that we make in Assumption~\ref{assu:simple}. First of all, we restricted $\sS$ to be a closed cube $[a,b]^{j-1}$, which is obviously a specific instant of a compact set with positive mass. Further, we assumed $\sigma$ to be continuous, and thus, $\lambda \sigma(z_{<j})$ must be lower-bounded in $\sS$. This allows us to replace the function $f$ by the constant $\lambda_\sigma - \varepsilon$ for arbitrary small $\varepsilon>0$. 

%After visualizing the meaning of Assumption~\ref{assu:copuladensity} and giving some intuition, we verify that it can indeed be proven in various settings. \ml{Entfernen}

After giving this simplified sufficient condition, we provide some intuition by presenting some examples where Assumption~\ref{assu:simple} holds true.
\begin{example}[Independent Variables]\label{exa:inde}
    Consider a random variable $\rvz$ with independent components, i.e. $q(\vz)=\prod_{j=1}^D q_j(\vz_j)$. Then, the associated copula is the independence copula (Figure~\ref{fig:copulas}), which is a uniform random distribution on $[0,1]^D$. Therefore it is $c\bigl(F_1(\vz_1), \dots , F_D(\vz_D)\bigr)=1$ for all $\vz\in \sR^D$ and Assumption~\ref{assu:simple} follows immediately since $s\exp(-(\lambda_\sigma - \varepsilon)z_j)\rightarrow 0$ for $z_j\rightarrow \infty$. 
\end{example}
\begin{example}[Bounded Copula Density]\label{exa:bounded}
    Consider a lower-bounded copula density, i.e. there exists a lower bound $a>0$ such that 
    \begin{equation*}
        c(\vu_1, \dots , \vu_D) \geq a \quad \text{ for all } u\in[0,1]^D \enspace .
    \end{equation*}
    Again, the validity of Assumption~\ref{assu:simple} in this setting is clear.
\end{example}
\begin{figure}[h]
\begin{center}
%\framebox[4.0in]{$\;$}
\begin{minipage}{0.4\textwidth}
\centering
\begin{center}
\includegraphics[width=\columnwidth]{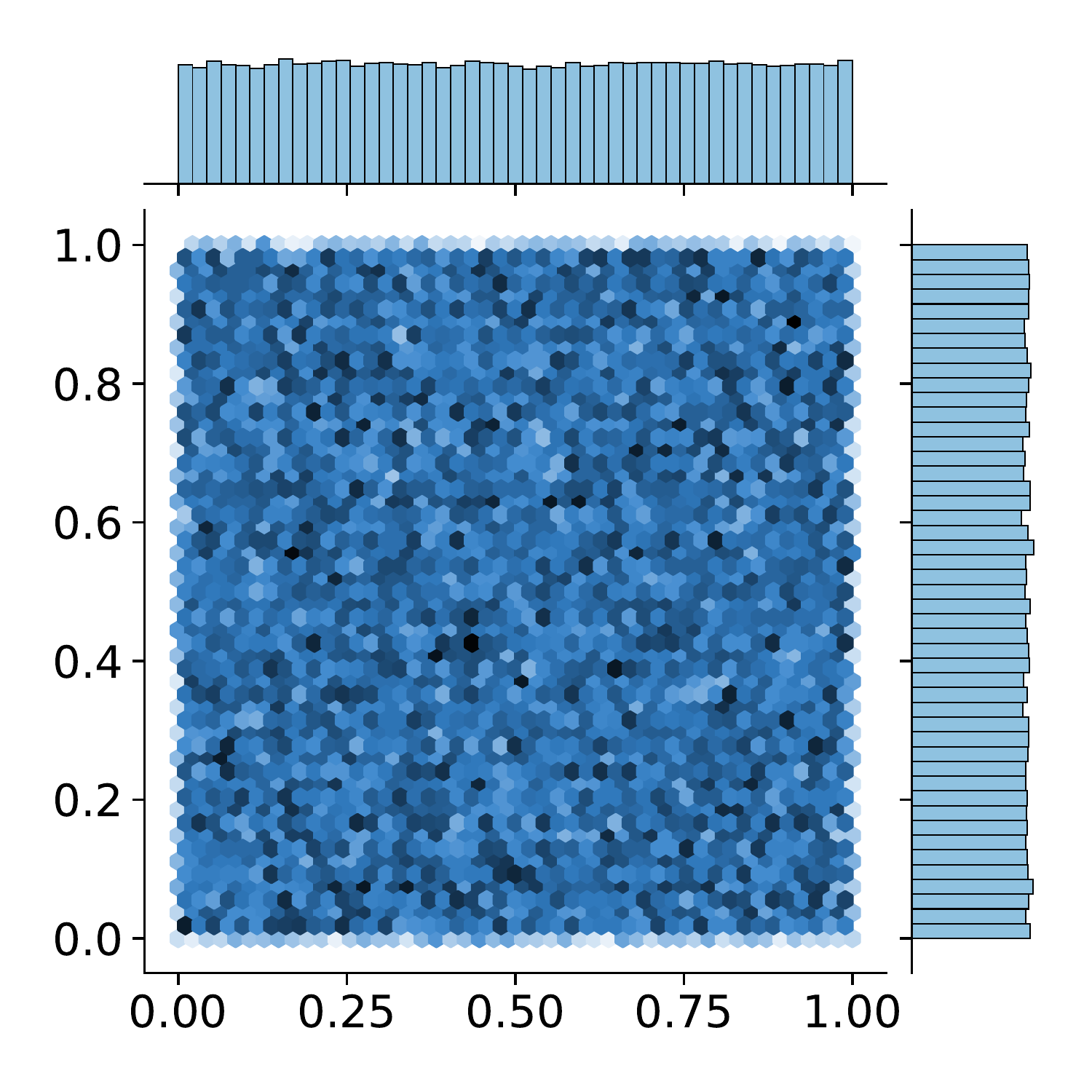}
\end{center}
\subcaption{Independence copula: $c(\vz_1, \vz_2)=1$. \label{fig:independence_copula}}
\end{minipage}
\begin{minipage}{0.4\textwidth}
\centering
\begin{center}
\includegraphics[width=\linewidth]{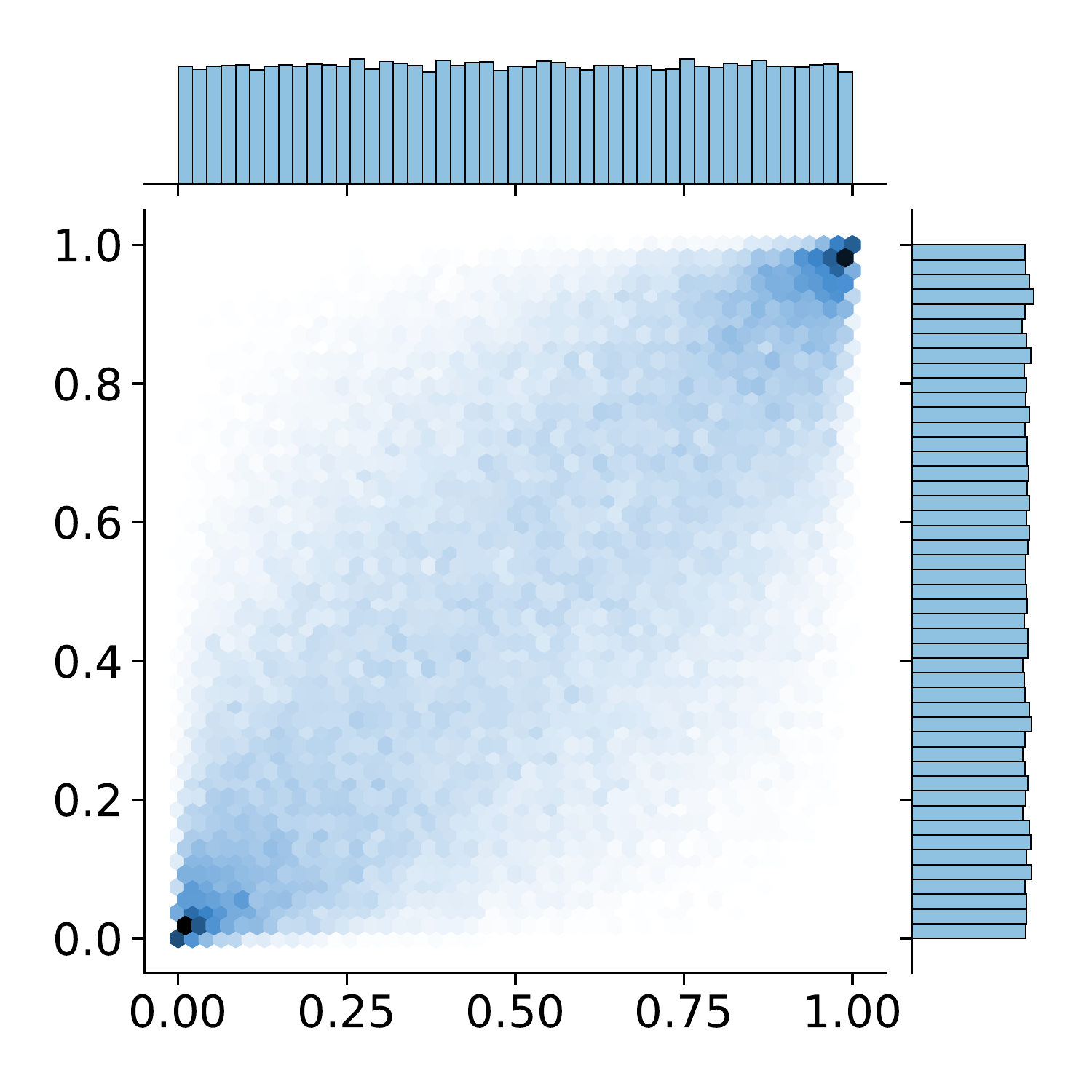}
\end{center}
\subcaption{Gaussian copula with correlation $\rho_{12}=0.7\enspace$.\label{fig:gaussiancopula}}
%\end{subfigure}
\end{minipage}
\caption{Two 2-dimensional Copula densities.\label{fig:copulas}}
\end{center}
\end{figure}
Furthermore, this assumption is obviously not limited to bounded copula densities but also holds for copula densities that converge to $0$ but whose decay rate in $z_j$ is lower-bounded by \eqref{eq:assu_simp}. To visualize the intuition, consider the 2-dimensional copula density of a Gaussian copula in Figure~\ref{fig:copulas}. 
Imagine fixing $\sS$ such that $F_1(\vz_1)\in[0.5, 0.75]$, which is compact for continuous $F_1$. Then \eqref{eq:assu_simp} bounds the decay rate within the ``tube'' $[0.5, 0.75]$ if we consider $\vz_2\rightarrow\infty$, i.e. if $F_2(\vz_2)\rightarrow 1$. Next, we show how we can formally prove the assumption for Gaussian copulae.
Recall, the Gaussian copula with correlation matrix $R\in\sR^{D \times D}$ has density function 
    \begin{equation}
        c(\vu) = \frac{1}{\sqrt{\det R}} \exp \Biggl( 
        -\frac{1}{2} \bigl(\Phi^{-1}(\vu_1), \dots \Phi^{-1}(\vu_D)\bigr) \bigl(R^{-1} - I\bigr) \bigl(\Phi^{-1}(\vu_1), \dots \Phi^{-1}(\vu_D)\bigr)^\top 
        \Biggr) \enspace , \label{eq:gaussiancopula}
    \end{equation}
    where $I\in\sR^{D \times D}$ is the identity matrix, $\Phi^{-1}$ is the inverse CDF of the univariate standard Gaussian distribution, and $u\in[0,1]^D$. In the following, we consider Assumption~\ref{assu:simple} for $j=D$. 
    
    Note that $\sS$ is assumed to be a compact set, therefore $F_j(\vz_j)$ are all upper and lower-bounded by some value for $\vz_{<j}\in \sS$. This makes all polynomials of them also bounded. Hence, we can find constants $a^\prime, b^\prime, c^\prime$ such that we can lower-bound the term within the exponential in~\eqref{eq:gaussiancopula} by
    \begin{equation*}
          -\frac{1}{2} \bigl( a^\prime \Phi^{-1}\bigl(F_D(\vz_D)\bigr)^2 + b^\prime \Phi^{-1} \bigl( F_D(\vz_D)\bigr) + c^\prime \enspace. 
    \end{equation*}
    Plugging the above into~\eqref{eq:gaussiancopula} gives 
    \begin{align}
        c\bigl(F_1(\vz_1), \dots , F_D(\vz_D)\bigr) &\geq 
        \frac{1}{\sqrt{\det R}} \exp \Biggl(
            -\frac{1}{2} \bigl( a^\prime \Phi^{-1}\bigl(F_D(\vz_D)\bigr)^2 + b^\prime \Phi^{-1} \bigl( F_D(\vz_D)\bigr) + c^\prime
        \Biggr) \nonumber \\
        &\propto \exp \biggl( 
            -a \Phi^{-1} \bigl(F_D(\vz_D)\bigr)^2 + b \Phi^{-1}\bigl( F_D(\vz_D) \bigr) 
        \biggr) \quad \text{(for some }a, b) \nonumber \\
        &\geq \exp \biggl( 
            -\vert a\vert \Phi^{-1} \bigl(F_D(\vz_D)\bigr)^2 + \vert b\vert \Phi^{-1}\bigl( F_D(\vz_D) \bigr) 
        \biggr) \nonumber \\
        &\geq \exp \biggl( 
            - \vert a \vert \Phi^{-1} \bigl(F_D(\vz_D)\bigr)^2
        \biggr) \enspace ,\label{eq:gauscopula_lower}
    \end{align}
    where the last line applies if $\Phi^{-1}\bigl(F_D(\vz_D)\bigr)\geq 0$, which is satisfied if $z^\ast$ is large enough\footnote{for instance  if $z^\ast$ is larger than the median of $\rvz_D$}. 
    
    Next, we use the asymptotic relation from Lemma~\ref{lemma:asymp}
    \begin{equation*}
        \Phi^{-1}\bigl(F_D(\vz_D) \bigr) \sim \sqrt{-2 \log\bigl( 1-F_D(\vz_D) \bigr) } \enspace .  
    \end{equation*} 
    Hence, for each $\varepsilon>0$ there exists a $z^\ast$ large enough such that 
    \begin{equation*}
        \biggl\vert \frac{\Phi^{-1}\bigl(F_D(z_D)\bigr)}{\sqrt{-2\log\bigl( 1- F_D(z_D)\bigr) }} - 1 \biggr\vert < \varepsilon \enspace ,
    \end{equation*}
    which can be rearranged to 
    \begin{equation*}
        \Phi^{-1}\bigl(F_D(z_D)\bigr) < \sqrt{-2\log\bigl( 1- F_D(z_D)\bigr) } (1+\varepsilon) \enspace .
    \end{equation*}
    Plugging the above into~\eqref{eq:gauscopula_lower}, we obtain 
    \begin{align}
        c\bigl(F_1(\vz_1), \dots , F_D(\vz_D)\bigr) &\geq  \exp \biggl( 
            2\vert a \vert (1+\varepsilon)^2 \log\bigl( 1-F_D(\vz_D)\bigr) 
        \biggr) \nonumber \\
        &= \bigl( 1- F_D(\vz_D)\bigr)^{2\tilde{a}} \enspace , \label{eq:1-fd}
    \end{align}
    where we define $\tilde{a}:= 2\vert a \vert (1+\varepsilon)^2$.
    Hence, we are left to lower-bound~\eqref{eq:1-fd}, which we can do for a range of heavy-tailed marginal distributions such as:
    \begin{enumerate}
        \item \textbf{Pareto distribution:} The Pareto distribution with shape parameter $\alpha$ has CDF 
        \begin{equation*}
            F(\vz)= 1 - \frac{1}{\vz}^{\alpha} \enspace . 
        \end{equation*}
        Therefore, 
        \begin{align*}
            \bigl( 1- F_D(\vz_D)\bigr)^{2\tilde{a}} &= \frac{1}{\vz_D}^{2\vert a \vert \alpha} \\
            &= \exp \bigl(-2\tilde{a} \alpha \log( \vz_D) \bigr) \\
            &\geq \exp \bigl( -2 \tilde{a}\alpha \vz_D \bigr)
        \end{align*}
        for $\vz_D\geq e$. 
        \item \textbf{Scale invariant distributions:} Following the same argument as above, each distribution with CDF 
        \begin{equation}\label{eq:scale-invariant}
            F(\vz) = 1 - b \vz^{-\alpha} \quad \text{for } \vz>\vz^\ast 
        \end{equation}
        for constants $b, \alpha, \vz^\ast >0$ satisfies the bound \begin{equation*}
            \bigl( 1- F_D(\vz_D)\bigr)^{2\tilde{a}} \geq b \exp\bigl( -2\tilde{a} \alpha \vz_D \bigr) \enspace .
        \end{equation*}
        Each distribution with CDF as in~\eqref{eq:scale-invariant} is a scale-invariant distribution (see e.g.~Theorem~2.1 in \citet{heavytails_book}).
        %\item \textbf{Weibull distribution:}
        \item \textbf{Exponentially decaying distributions:} Every distribution that satisfies the bound 
        \begin{equation*}
            F(\vz) \leq 1 - \exp(\vz)^{-\alpha} 
        \end{equation*}
        for $\vz>\vz^\ast$ and $\alpha > 0$. In this case, we can again show that Assumption~\ref{assu:copuladensity} is valid: 
        \begin{equation*}
            \bigl( 1- F_D(\vz_D)\bigr)^{2\tilde{a}} \geq  \exp(-2\tilde{a} \alpha \vz_D) \enspace .
        \end{equation*}
    \end{enumerate}

Lastly, we want to emphasize that Corollary~\ref{corro} is derived by an iterative application of Theorem~\ref{maintheo}. Therefore, Assumption~\ref{assu:copuladensity} must hold for all ``flow steps'', i.e. if $T=T^{(L)} \circ \cdots \circ T^{(1)}$, we need to ensure validity of Assumption~\ref{assu:copuladensity} for $\rvz^{(0)}:=\rvz,\; \rvz^{(1)}:=T^{(1)}(\rvz), \; \rvz^{(2)}:=T^{(2)}\circ T^{(1)}(\rvz), \; \dots , \; \rvz^{(L-1)}:=T^{(L-1)}\circ \cdots \circ T^{(1)}(\rvz)$. In Example~\ref{exa:inde}, we show that this assumption holds true for $\rvz^{(0)}$ since we define our base distribution under the mean-field assumption. Furthermore, we conjecture that if we apply a Lipschitz-continuous diffeomorphism on a random variable with bounded copula density, then the transformed random variable must also have a bounded copula density. Hence, Assumption~\ref{assu:copuladensity} would be valid for all ``flow steps'' (see Example~\ref{exa:bounded}) when starting with a mean-field base distribution $q(\vz)=\prod_{j=1}^D q_j(\vz_j)$. However, this conjecture needs to be studied in further research.

\subsection{Proof of Theorem~\ref{theo:gTAFapproxmTAF}}\label{sec:proofgTAFapproxmTAF}
\begin{proof}
    We proceed in two steps: First, we show that $z_{\nu} \xrightarrow[]{\nu \rightarrow \infty}_D z$ by applying the \textit{Cram\'er-Wold Theorem}. Second, we use the \textit{Continuous-Mapping Theorem} to finish the proof. All used theorems can be found, for instance, in \citet{van2000asymptotic}.
    \paragraph{Step 1.} By the Cram\'er-Wold Theorem, it is sufficient to show that for every $v \in \mathbb{R}^D$ it is 
    \begin{equation*}
        v^\top z_{\nu}= \sum_{j=1}^D v_j z_{\nu, j} \xrightarrow[]{\nu \rightarrow \infty}_D \sum_{j=1}^D v_j z_{ j} = v^\top z \enspace. 
    \end{equation*}
    We prove this convergence via the convergence of the characteristic function. Since $z_{\nu, j}$ are independent for all $j\in\{1,\dots ,D\}$, it is 
    \begin{align*}
        \varphi_{v^\top z_{\nu}}(t) &= \varphi_{z_{\nu, 1}}(v_1 t)  \cdots \varphi_{z_{\nu, d_l }}(v_{d_l} t) \varphi_{z_{d_l +1}}(v_{d_{l} +1}t) \cdots  \varphi_{z_{D}}(v_{D}t) \\ 
        &\rightarrow \varphi_{z_1}(v_1 t) \cdots \varphi_{z_{d_l} }(v_{d_l} t) \varphi_{z_{d_l +1}}(v_{d_l +1}t) \cdots  \varphi_{z_{D}}(v_{D}t) \\ 
        &= \varphi_{v^\top z}(t) \quad \forall t\in \mathbb{R} \enspace ,
    \end{align*}
    where the convergence follows from the convergence $t_\nu \xrightarrow[]{\nu\rightarrow \infty}_D \mathcal{N}(0,1)$ and \textit{Levy's Continuity Theorem}. The last equality follows since the marginals are independent. 
    Using Levy's Continuity Theorem once more, we conclude that 
    \begin{equation*}
        v^\top z_{\nu} \xrightarrow[]{\nu \rightarrow \infty}_D v^\top z \enspace . 
    \end{equation*}
    Hence, we can apply Cram\'er-Wold to obtain the convergence $z_{\nu} \xrightarrow[]{\nu \rightarrow \infty}_D z$. 
    \paragraph{Step 2.} We apply the Continuous Mapping Theorem to obtain the convergence result.  
\end{proof}

\section{Data-Driven LU-Layers}
One of the major reason for recent improvements of NFs is due to the generalization of the permutation layers to more general invertible linear layers \citep{oliva2018transformation}. One type of these general invertible linear layers are based on the LU-decomposition, which were essential building blocks for many SOTA NFs, such as in \citet{kingma2018glow, durkan2019neural}. 
Therefore, it is important to address whether we can generalize our theory from Section~\ref{sec:method} of the main paper to the more expressive LU-Layers. In this section we provide supplementary materials to Section~\ref{sec:mtaf_nsf} to show that we are indeed able to generalize to LU-Layers, while retaining their computational benefits. First, in Section~\ref{sec:theory_LU} we derive sufficient conditions under which linear layers preserve the marginal tail behaviour. In Section~\ref{sec_implementationlu} we present an efficient way to implement a tail-preserving LU-type invertible layer. 

\subsection{Marginal Tail Behavior under Linear Transformations} \label{sec:theory_LU}
In this section, we assume without loss of generality that the components of $\rvz$ are ordered such that $\rvz$ is $j$-light-tailed for $j\in\{1,\dots ,d_l \}$ and $j$-heavy-tailed for $j \in \{d_l +1, \dots, D\}$. Our goal is to find conditions for a matrix $W$ under which $\rvz$ and $W\rvz$ have equal tail-behavior. 

\begin{theorem}\label{theo_LU}
Let $\rvz$ be a random variable that is $j$-light-tailed for $j\in\{1,\dots ,d_l \}$ and $j$-heavy-tailed for $j \in \{d_l +1, \dots, D\}$. Further, consider a diagonal invertible block-matrix 
\begin{equation} \label{eq_blockmatrix}
    W = \begin{pmatrix}
    A & 0 \\ 
    B & C   \end{pmatrix} \enspace ,
\end{equation}
with $A\in \mathbb{R}^{d_l \times d_l}, \, B \in \mathbb{R}^{(D-d_l ) \times d_l}, \, C\in\mathbb{R}^{(D-d_l ) \times (D-d_l ) }$ and $0$ is a zero matrix of size $d_l \times (D-d_l ) $. 
Then, it follows that $W\rvz$ and $\rvz$ have equal tail behavior. 
\end{theorem}
Before heading into the proof, we introduce the following useful lemma: 
\begin{lemma}\label{helper_lemma}
    Let $0<\lambda^{\ast}$ be a scalar such that $\mathbb{E}_{\rvz}[\exp(\lambda^\ast \rvz)] < \infty$ for some univariate random variable $\rvz$. Then, it holds that $\mathbb{E}_{\rvz}[\exp(\lambda \rvz)] < \infty$ for all $0<\lambda\leq\lambda^\ast$. 
\end{lemma}
\begin{proof}
    Consider $0<\lambda \leq \lambda^\ast$. Then, we can split the expectation 
    \begin{align*}
        \E_\rvz \bigl[ e^{\lambda \rvz } \bigr] &= \int_{(-\infty, 0]} q(z) e^{\lambda z} dz + \int_{(0, \infty)} q(z) e^{\lambda z} dz \\ 
        &\leq \int_{(-\infty, 0]} q(z) e^{0} dz + \int_{(0, \infty)} q(z) e^{\lambda^\ast z} dz \enspace .
    \end{align*}
    The first integral is upper-bounded by $1$, and the second integral can be upper-bounded by the integral over the same integrand on $\mathbb{R}$, which is bounded by definition of per assumption. Hence, $\mathbb{E}_{\rvz}[\exp(\lambda \rvz)] < \infty$ for all $0< \lambda \leq \lambda^\ast$.
\end{proof}
Now we are ready to prove Theorem~\ref{theo_LU}. 
\begin{proof}
    The idea of the proof is to show that a linear combination of $k$ (non-degenerate) random variables $\rvz_1, \dots , \rvz_k$ is light-tailed if and only if all $\rvz_1, \dots ,\rvz_k$ are light-tailed. We do this via algebraic induction, i.e. we show that 
    \begin{enumerate}
        \item $a \rvz_j$ is light-tailed iff. $\rvz_j$ is light-tailed for some scalar $a \in \mathbb{R}$; 
        \item $\rvz_1 + \rvz_2$ is light-tailed iff. $\rvz_1$ and $\rvz_2$ are both light-tailed.  
    \end{enumerate}
    For an arbitrary linear combination random variables, we can iterate through 1. and 2. to prove that the linear combination is light-tailed if and only if each component is light-tailed.  
    
    1. Let us assume that $\rvz_j$ is light-tailed. Therefore it exists some $\lambda^\ast \in \mathbb{R}$ such that $\E_{\rvz_j}[ \exp(\lambda^\ast \rvz_j)] < \infty$. Further, for the moment-generating function of $a\rvz_j$ it is 
    \begin{equation*}
        \E_{a \rvz_j}\bigl[e^{\lambda \rvx}] = \int_{\mathbb{R}} p_{\rvz_j}(z_j) e^{\lambda a z_j} dx < \infty \enspace 
    \end{equation*}
    for $\lambda:= \lambda^\ast / a$, where the first equality follows from the LOTUS. The other direction of the equivalence follows analogously. \\ 
    2. Consider $\rvz_1$ and $\rvz_2$ with joint PDF $q(z_1, z_2)$. Since both random variables are light-tailed, there exist $\lambda_1$ and $\lambda_2$ such that their moment-generating function is bounded. 
    Then, it is for $\lambda:=0.5 \min \{ \lambda_1, \lambda_2\}$
    \begin{align*}
        \E_{\rvz_1 + \rvz_2} \bigl[ e^{\lambda z} \bigr] &= \int_{\mathbb{R}} \int_{\mathbb{R}} e^{\lambda (z_1 + z_2)} q(z_1, z_2)dz_1 dz_2 \quad \text{(LOTUS)} \\
        &= \int_{\mathbb{R}} \int_{-\infty}^0  e^{\lambda(z_1 + z_2)} q(z_1, z_2) dz_1 dz_2  + \int_{\mathbb{R}} \int_{0}^{ \infty}  e^{\lambda(z_1 + z_2)} q(z_1, z_2) dz_1 dz_2 \\
        &\leq \int_{\sR} e^{\lambda z_2} q(z_2) dz_2 +  
        \int_{\mathbb{R}} \Biggl( \int_{0}^{\min \{0, z_2\} }  e^{\lambda(z_1 + z_2)} q(z_1, z_2) dz_1   \\ 
        & \hspace{4cm} + 
        \int_{\max \{0, z_2\}}^{\infty }  e^{\lambda(z_1 + z_2)} q(z_1, z_2) dz_1 \Biggr) dz_2 \enspace ,
    \end{align*}
    where the last line follows by replacing $z_2$ in the first integral by $0$, which is an upper bound for $z_2$. Using the definition of $\lambda$ and Lemma~\ref{helper_lemma}, we see that the first integral is bounded by some constant $c_1$. Using the monotonicity of the integrands once more, we get 
    \begin{align*}
        \E_{\rvz_1 + \rvz_2} \bigl[ e^{\lambda z} \bigr] &\leq c_1 + \int_{\sR} \Biggl( \int_{0}^{\min\{0, z_2\}} e^{2\lambda z_2} q(z_1, z_2) dz_1 +  \int_{\max\{ 0, z_2\}}^{\infty } e^{2\lambda z_1} q(z_1, z_2) dz_1 \Biggr) dz_2 \\ 
        &\leq c_1 + \int_{\sR} e^{2\lambda z_2} q(z_2) dz_2 + \int e^{2\lambda z_1} q(z_1) dz_1 \enspace, 
    \end{align*}
    which follows by upper-bounding the integral from $0$ to $\min \{0, z_2\}$ and the integral from $\max\{0, z_2\}$ by the integrals over $\sR$, respectively. Using Lemma~\ref{helper_lemma} and the definition of $\lambda$ once more, we can find constants $c_2$ and $c_3$ that bound the remaining integrals. Hence, we conclude that $\E_{\rvz_1 + \rvz_2} < \infty$, which proves the backward direction of 2. 
    To prove the forward direction, we show that if $\rvz_1$ and $\rvz_2$ are not both light-tailed (i.e at least of them is heavy-tailed), than $\rvz_1 + \rvz_2$ is not light-tailed. Without loss of generality, we assume $\rvz_2$ to be heavy-tailed. Then, we can write 
    \begin{align}
        \E_{\rvz_1 + \rvz_2}\bigl[ e^{\lambda z} \bigr] &= \int_{\sR} \int_{\sR} e^{\lambda (z_1 + z_2)} q(z_1, z_2) dz_1 dz_2 \nonumber \\ 
        &=\int_{\sR} \Biggl( 
            \int_{-\infty}^0 e^{\lambda(z_1 + z_2)} q(z_1, z_2) dz_1 + \int_0^\infty e^{\lambda (z_1 + z_2)} q(z_1, z_2) dz_1
        \Biggr) dz_2 \nonumber \\
        & \geq \int_{\sR} \Biggl( 
            \int_{-\infty}^0 e^{\lambda(z_1 + z_2)} q(z_1, z_2) dz_1 + \int_0^\infty e^{\lambda (z_2)} q(z_1, z_2) dz_1
        \Biggr) dz_2 \enspace .  \label{linez1+z2}
    \end{align}
    Note that we can lower-bound the last inner integral by 
    \begin{equation*}
        \int_0^\infty e^{\lambda z_2} q(z_1, z_2) dz_1 = e^{\lambda z_2} q(z_2) - e^{\lambda z_2}\int_{-\infty}^0 q(z_1, z_2) dz_1 \geq - e^{\lambda z_2}\int_{-\infty}^0 q(z_1, z_2) dz_1 \enspace . 
    \end{equation*}
    Plugging this bound into \eqref{linez1+z2} gives us 
    \begin{equation*}
        \E_{\rvz_1 + \rvz_2}\bigl[ e^{\lambda z} \bigr] \geq \int_{\sR} e^{\lambda z_2} q(z_2)dz_2 = \infty \enspace \forall \lambda >0 \enspace ,
    \end{equation*}
    which is due to the heavy-tailedness of $\rvz_2$. Therefore, $\rvz_1 + \rvz_2$ must also be heavy-tailed. This proves the equivalence in 2. \\
    Finally, we note that due to the block-triangular form of $A$ the first $d_l$ components of $\rvx := A\rvz$, i.e. $\rvx_1, \dots, \rvx_{d_l}$ are linear combinations of light-tailed components $\rvz_1, \dots, \rvz_{d_l}$, which implies the light-tailedness of $\rvx_1, \dots ,\rvx_{d_l}$. The remaining $D-d_l$ components of $\rvx$ are linear combinations containing at least one heavy-tailed component $\rvz_j$ with $j \in \{d_l , \dots, D\}$ and are therefore again heavy-tailed\footnote{Note that this argument assumes that $D$ has no empty rows, which is implicitly assume due to the invertibility of $A$. Compare with Equation ref(Algosection).}. This completes the proof. 
\end{proof}

\subsection{Implementation of Data-Driven LU-Layers} \label{sec_implementationlu}
In the previous section, we have seen that in order to retain the tail behavior, the block-matrix form given in~\eqref{eq_blockmatrix} is sufficient. In this section, we give more details on an efficient parameterization leading to fast inversion and log-determinant computations. 
It is well-known that under mild conditions the inversion of block-matrices is efficiently solvable. 
\begin{lemma}[Inversion of Block-Matrices]\label{lemma_inversion}
    Consider invertible square matrices $A\in\sR^{d_l \times d_l }, \, D\in\sR^{d_h \times d_h }$ and arbitrary matrices $B \in \sR^{d_l \times d_h }, \, C\in\sR^{d_h  \times d_l }$ for some $d_l, d_h \in \mathbb{N}$. Then it holds that 
    \begin{equation*}
        \begin{pmatrix}
            A & B \\
            C & D 
        \end{pmatrix}^{-1} = 
        \begin{pmatrix}
            (A - BD^{-1}C)^{-1} & 0 \\ 
            -D^{-1}C(A-BD^{-1}C)^{-1} & D^{-1}%  (D-CA^{-1}B)^{-1} 
        \end{pmatrix}
        \begin{pmatrix}
            I & -BD^{-1} \\ 
            %-CA^{-1} & I 
            0 & I 
        \end{pmatrix} \enspace . 
    \end{equation*}
    As a special case, it is 
\begin{equation}\label{eq_tailLU_inv}
        \begin{pmatrix}
            A & 0 \\ 
            B & C 
        \end{pmatrix}^{-1} 
        = 
        \begin{pmatrix}
            A^{-1} & 0 \\ 
            -C^{-1}BA^{-1} & C^{-1} 
        \end{pmatrix} \enspace 
    \end{equation}
    for invertible square matrices $A\in\sR^{d_l \times d_l }, \, C\in\sR^{d_h \times d_h }$ and arbitrary matrix $B\in\sR^{d_h  \times d_l }$.
\end{lemma}
Lemma~\ref{lemma_inversion} is a standard result that can be found in many linear algebra text books (see e.g. \citet{block_matrix_inversion}).

Furthermore, we can compute the determinant of a diagonal block-matrix $W$ as defined in~\eqref{eq_blockmatrix} as 
    \begin{equation}
        \det (W) = \det (A) \det(C) \enspace . \label{eq_blockdeterminant}
    \end{equation}

Now, let us summarize the expensive computations in~\eqref{eq_tailLU_inv} and~\eqref{eq_blockdeterminant} that both need to be efficient in NFs. The computations involve inversions of $A$ and $C$, only forward computations of $B$, and the computation of $\det(A)$ and $\det(C)$. Hence, we propose to parameterize $A$ and $C$ using a LU-decomposition for both matrices, which leads to efficient inverse\footnote{which can be guaranteed by restricting the diagonal elements of the upper diagonal matrix in the LU-decomposition.} and log-determinant computations. We do not make any restrictions on $B$ and parameterize it by a standard unconstrained linear layer. 
We provide a PyTorch implementation of this modified \textit{tail-preserving LU-layer}, which can be accessed through our public git repository.

\section{Algorithms and computational Details}\label{sec:algosanddetails}
\label{sec:compdetails}
\subsection{Tail Estimation}\label{sec:tailest}
Many heavy-tailed distributions can be characterized by their tail index, which include the set of regularly varying distributions,\footnote{see Section~2 in \cite{heavytails_book} for further details} such as the $t$-distribution, the Pareto distribution, and many more. However, as already shown in Section~\ref{sec:heavy-tails}, the tail index does not depend on the body of the distribution, and hence, non-tail samples must typically be discarded for tail index estimation. 
Although a variety of estimators for the tail index exist, such as the Hill estimator \citep{hill1975simple}, the moment estimator \citep{moments_estimator}, and kernel-based estimators \citep{kernel_estimator}, none of them is considered to be as superior in all settings. 
A major issue of all mentioned estimators is that they are based on a threshold defining the tail, i.e. the user needs to input statements of the form ``the $k$ largest samples are considered to be tail events''. Even though there exist some strategies to find $k$, there is none working robustly in all settings.
In fact, one can construct simple counter examples for all estimators that lead to failures of tail estimation. This is due to undesired properties of the estimators, such as the lack of translation invariance of the Hills estimator (while the tail index clearly is location invariant). We refer to Section~9 in \citet{heavytails_book} for a detailed text book treatment of tail index estimation.
In summary, robust tail estimation is still considered as an unsolved problem, which forces practitioners to consider multiple estimators to make a well-founded decision. %In Section ... of the Appendix, we give some further details on the tail estimators and their use-cases.
Furthermore, we note that the Hills estimator can only be applied for regularly varying distributions, which excludes the application of the Hills estimator to classify light-tailed distributions. In contrast, the moments and the kernel estimator can both be applied to identify heavy-tailed marginals and to assess a tail index. 

To implement the tail assessment scheme, see Step~1 of the proposed method in Section~\ref{sec:method}, we found that Algorithm~\ref{alg:tailestimation} works fine in classifying the correct tail behavior and giving a decent initialization for the tail indices. 
That is, we use the moments double-bootstrap estimator \citep{moments_bootstrap} and the kernel-type double-bootstrap estimator \citep{kernel_bootstrap} to assess heavy-tailedness of the data distribution. If both estimators predict a light-tailed distribution, we set the corresponding marginal base distribution to be standard normal distributed, otherwise we set it to a the standardized $t$-distribution, i.e.~$\rvz_j\sim t_{\hat{v}_j}$, where $\hat{v}_j$ is the Hill double-bootstrap estimator \citep{hill_bootstrap1, hill_bootstrap2}.
We reused the code by \citet{tail_estcode}, which implements all tail estimation procedures\footnote{including the hyperparameter selection \citep{hill_bootstrap1, hill_bootstrap2, moments_bootstrap, kernel_bootstrap}} from our Algorithm. Notice that we clip the tail index by $10$, i.e. the algorithm classifies marginals with a tail index larger than $10$ as light-tailed, which prevents a too restrictive set of allowed permutations, see Step~3 in Section~\ref{sec:method}. To illustrate this argument, consider the following simple example. Assume that we estimate all except of one marginal to be heavy-tailed. Then, the first component of the flow is never allowed to permute with other components, since they are classified as heavy-tailed. Hence, the mixing of the first component would be severely restricted. Further, since large tail indices indicate a less heavy-tailed distribution, it is reasonable to clip the tail index at some threshold.

\begin{algorithm}[tb]
   \caption{Marginal Tail Estimation} \label{alg:tailestimation}
\begin{algorithmic}
   %\Require Data\_val
\STATE {tail.est $\gets  [\,]$}
\FOR{j in $\{1,\dots , D\}$}
    \STATE {marginal $\gets$ Data\_val[:, j]}
    \STATE {moments $\gets$ moments\_est($\vert$marginal$\vert$)} \COMMENT{0 if $\vert$marginal$\vert$ is estimated to be light-tailed}
    \STATE {kernel $\gets$ kernel\_est($\vert$marginal$\vert$)} \COMMENT{0 if $\vert$marginal$\vert$ is estimated to be light-tailed}
    \IF{moments==kernel==0}
        \STATE {tail\_est.append(0)} \COMMENT{light-tailed if moments and kernel estimate a light-tailed marginal}
    \ELSE 
        \STATE{hill $\gets$ hills\_est($\vert$marginal$\vert$)}
        \IF{hill$>10$}
            \STATE{tail\_est.append(0)} \COMMENT{light-tailed if hills estimator predicts high tail index}
        \ELSE 
            \STATE{tail\_est.append(hill)} 
        \ENDIF
    \ENDIF
\ENDFOR
\STATE{ \textbf{return} tail\_est}
\end{algorithmic}
\end{algorithm}

\subsection{Synthetic Data Generation}\label{sec:synth_data}
The generation of the synthetic distribution consists of 3 steps: 1. Generating the marginal distributions, 2. Defining a copula distribution, 3. Combining the marginal and the copula to obtain a multivariate joint distribution. 

\paragraph{Generating the marginal distributions.} The first two marginals are defined to be Gaussians. The following marginals are a 2-mixtures of Gaussians and a mixtures of three Gaussians. The last $d_h \in \{1,4\}$ components are a mixture of two $t$-distributions and the remaining marginals are again mixtures of two Gaussians. All mixtures have equal weight for each mixture component and all means and standard-deviations are randomized. Means are constructed by uniformly sampling from $[-4, 4]$, whereas standard-deviations are sampled from $[1,2]$. 
\paragraph{Defining a copula distribution.}
Recall, a Gaussian copula~\eqref{eq:gaussiancopula} is parameterized by a correlation matrix $R$. To generate $R$, we randomly sample $16$ different pairs $(i,j)\in\{1,\dots , 8\}^2$ with $i\neq j$ and set the corresponding entry of the correlation matrix $
    R_{i,j}:=0.25. $ The diagonals of $R$ are set to $1$.
\paragraph{Obtaining a joint distribution.}
Lastly, we combine the marginals with the Gaussian copula using Sklar's Theorem~\ref{sklarstheorem}. This gives us a multivariate distribution with specified and complex marginals with a dependency structure given by the copula, see \citet{joe2014dependence} for more details on the induced dependencies.

To construct the training, test, and validation sets $15.000, \; 75\,000$, and $10\,000$ samples from this distribution are sampled, respectively.

In the setting $D=50$, we apply a similar procedure but with $d_h=10$ heavy-tailed components and with the first $40$ marginals being 2-mixtures of Gaussians, the remaining $10$ marginals being 2-mixtures t-distributions. We define the dependency-structure by randomly selecting $200$ pairs $(i,j)\in\{1, \dots, 50\}^2$ with $i\neq j$ and set $R_{i,j}=0.25$ again. Training, validation, and test sets consists of $50\,000, \; 10\,000$, and $75\,000$ samples, respectively.
\subsection{Synthetic Experiments}\label{sec_suppsynthexp}
In all synthetic experiments, we used a NSF with $5$ layers and corresponding LU-linearities. \name{} employs the modified LU-linearities from Section~\ref{sec:mtaf_nsf}. In the NSF layers, we used conditioner ResNets with $2$ hidden layers with $30$ or $200$ hidden neurons in the case $D=8$ and $D=50$, respectively and ReLU activations. Further, we used NSF layers with $3$ bins and set the tail-bound to $2$. We optimized the network using Adam with $5\,000$ or $20\,000$ train steps in the case $D=8$ and $D=50$, respectively, with a learning rate of $1e-5$ and a weight decay of $1e-6$. To fit the Gaussian copula baseline, we use the default settings of the \emph{copulas} \citep{sdv} library. 

To assess the sample quality on the tail of the distribution, we consider $3$ metrics.
\begin{enumerate}
    \item \textbf{Tail Value at Risk} for some level $\alpha$, which is defined by 
    \begin{equation*}
        \operatorname{tVaR}_{\alpha} := \operatorname{tVaR}_{\alpha}(F):= \frac{1}{1-\alpha} \int_{\alpha}^1 F^{-1}(u) du \enspace  
    \end{equation*}
    for some CDF $F$. $\operatorname{tVaR}_{\alpha}$ is a well-known metric and is widely used in finance \citep{mcneil2015quantitative}.
    We plug in the marginal empirical CDFs $\hat{F}_{\operatorname{data}}$ and $\hat{F}_{\operatorname{flow}}$, i.e. the empirical CDFs based on the data and on synthetic samples, respectively, and calculate the absolute difference between these quantities. The resulting metric is the marginal $\operatorname{tVaR}$-difference for the level $\alpha$. We set $\alpha=0.95$ in all our experiments. 
    \item \textbf{Area under log-log plot} is defined by 
    \begin{equation*}
        \operatorname{Area} := \sum_{i=1}^n \biggl\vert \log \bar{F}^{-1}_{\operatorname{data}} \Bigl(\frac{i}{n} \Bigr) - \log \bar{F}^{-1}_{\operatorname{flow}} \Bigl( \frac{i}{n} \Bigr) \biggr\vert \log \frac{i + 1}{i} \enspace ,
    \end{equation*}
    where $\bar{F}^{-1}_{\operatorname{data}}, \, \bar{F}^{-1}_{\operatorname{flow}}$ denote the inverse empirical complementary CDFs given by the test data and the flow samples, respectively.
    \item \textbf{Synthetic Tail Estimates}, where the tail-assessment is similar to the on described in Algorithm~\ref{alg:tailestimation}. We can then assess the similarity between the tail estimators based on the data and the tail estimators based on synthetically generated flow samples.
    \end{enumerate}

\paragraph{Setting D=8}
\begin{table*}[t]
\caption{Standard deviations corresponding to the experiments shown in Table~\ref{table:synth_metrics}, i.e. in the setting $\nu=2$ and $d_h\in \{1,4\}$, for one target distribution. 
\label{table:synth_metrics_stds}}
\label{sample-table}
\begin{center}
\begin{tabular}{@{}lccccccccccc@{}}
%& \multicolumn{4}{c}{number of heavy-tailed components $h$} \\
%\cmidrule{2-5}
\toprule
$d_h$ & \multicolumn{5}{c}{1} & &\multicolumn{5}{c}{4}  \\
\cmidrule{2-6}
\cmidrule{8-12}
& $L$ & $\operatorname{Area}_l$ & $\operatorname{Area}_h$ & $\operatorname{tVaR}_l$ & $\operatorname{tVaR}_h$ && $L$ & $\operatorname{Area}_l$ & $\operatorname{Area}_h$ & $\operatorname{tVaR}_l$ & $\operatorname{tVaR}_h$ \\ 
\midrule
vanilla & 0.02 & 0.02 & 0.12 & 0.16 & 3.74 & & 0.03 & 0.03 & 0.14 & 0.13 & 1.75 \\
TAF & 0.00 & 0.03 & 0.20 & 0.12 & 2.38 & & 0.01 & 0.06 & 0.12 & 0.22 & 1.15 \\
gTAF & 0.01 & 0.06 & 0.26 & 0.17 & 1.30 & & 0.01 & 0.06 & 0.14 & 0.20 & 0.64\\
mTAF & 0.01 & 0.03 & 0.21 & 0.10 & 1.41 & & 0.01 & 0.03 & 0.20 & 0.18 & 1.32\\
 \bottomrule
\end{tabular}
\end{center}
\end{table*}
In our experimental study, we generated 3 synthetic distributions per setting as explained in Section~\ref{sec:synth_data} and fit each model 25 times to each synthetic distribution. While it is reasonable to compare the averaged metrics over all 75 runs, investigating the standard deviation over all runs might be misleading since the metrics could be centered around different values for each synthetic distribution. For this reason, it is more insightful to compare standard deviations over runs where the target distribution is fixed, which we present in Table~\ref{table:synth_metrics_stds}. 
In Table~\ref{table:synth_metrics_D=3} we present the numeric results of the synthetic experiments for a larger tail index, i.e. a less extreme setting. 
In this setting, we observe that gTAF tends to perform slightly worse for light-tailed components, while achieving good results for heavy-tailed components. Note that in this case, the performance of mTAF degrades, which might be attributed to the less flexible structure of its linearities. 
When replacing the NSF-layers by MAF-layers, we observe that the MAF fails to converge for a vanilla base distribution. Again, mTAF strikes a balance between fitting heavy- and light-tailed marginals but the overall performance is better in the case of NSF-layers---especially for the heavy-tailed marginals. We conjecture that this is due to the linearity of the tails of each NSF-layer\footnote{This does not mean that the whole flow is linear in its tails since tail samples can be linearily mapped in and out of the bins, leading a non-linear mapping.}, which leads to a better generalization for those low-sample regions. Therefore, we continue the following experiments using the NSF architecture.
Furthermore, we investigate the tail indices of the generated samples by constructing confusion matrices similar to those in Figure~\ref{fig:synth_tail}, which we present in Figure~\ref{fig:synth_tail_supp}. We observe a similar behavior, that is, while vanilla and TAF produce mainly marginals with light-tailed marginals, gTAF is able to produce much better samples with heavy-tailed marginals. Again, \name{} produces marginals whose marginal tail behavior almost perfectly fits the true tail behavior. 
For further visual inspection of the generated marginals, we consider QQ-plots of the heavy-tailed components in Figure~\ref{fig:qq_df2h1} and Figure~\ref{fig:qq_plots_df2h4}. In both cases, vanilla and TAF---and sometimes in gTAF---we observe humps in the tails, which surrogate a bad sampling performance in their tails. \name{} does not have these, which is in accordance to our findings derived from Figure~\ref{fig:synth_tail} and \ref{fig:synth_tail_supp}. 
\paragraph{Setting D=50}
Table~\ref{table:synth_metrics_50D=2} compares the quantitative metrics for our synthetic experiments in the case $D=50$. In this case, we generate 3 synthetic distributions as explained in Section~\ref{sec:synth_data} and fit each model $5$ times. We observe that \name{} clearly outperforms vanilla and TAF in terms of $\operatorname{Area}$, while the flexibility in gTAF allows it to perform almost on par with the oracle copula baseline. 
Considering the metrics that account for the tail fit of the heavy-tailed components, we surprisingly observe that gTAF even outperforms the oracle copula model. 
However, considering the light-tailed components, gTAF performs a bit worse, which is not surprising since gTAF models each marginal distribution by a $t$-distribution.  
\begin{table*}[t]
\caption{Average test loss, Area under log-log plot, and $\operatorname{tVaR}$ (lower is better for each metric) in the setting $\nu=3$ and $d_h \in \{1,4\}$. The copula model serves as an oracle baseline. 
\label{table:synth_metrics_D=3}}
\begin{center}
\begin{tabular}{@{}lccccccccccc@{}}
\toprule
$d_h$ & \multicolumn{5}{c}{1} & &\multicolumn{5}{c}{4}  \\
\cmidrule{2-6}
\cmidrule{8-12}
& $L$ & $\operatorname{Area}_l$ & $\operatorname{Area}_h$ & $\operatorname{tVaR}_l$ & $\operatorname{tVaR}_h$ && $L$ & $\operatorname{Area}_l$ & $\operatorname{Area}_h$ & $\operatorname{tVaR}_l$ & $\operatorname{tVaR}_h$ \\ 
\midrule
vanilla & 10.82 & 0.25 & 2.77 & 0.55 & 12.06 & & 10.55 & 0.23 & 2.64 & 0.58 & 9.98 \\
TAF & 10.79 & 0.36 & 2.78 & 0.79 & 1.49 & & 10.46 & 0.38 & 2.95 & 0.93 & 2.36 \\
gTAF & 10.77 & 0.55 & 1.11 & 1.27 & 2.82 & & 10.38 & 0.50 & 1.13 & 1.00 &  2.56 \\
mTAF & 10.76 & 0.33 & 2.03 & 1.05 & 7.36 & & 10.38 & 0.34 & 1.54 & 1.09 & 5.02  \\
\midrule 
copula & 9.76 & 0.20 & 0.79 & 0.45 & 1.82 & & 9.76 & 0.19 & 0.91 & 0.46 & 1.62  \\
 \bottomrule
\end{tabular}
\end{center}
\end{table*}

\begin{table*}[t]
\caption{Average test loss, Area under log-log plot, and $\operatorname{tVaR}$ (lower is better for each metric) in the setting $\nu=2$ and $d_h \in \{1,4\}$ when using a MAF.
\label{table:synth_metrics_D=2_MAF}}
\begin{center}
\begin{tabular}{@{}lccccccccccc@{}}
\toprule
$d_h$ & \multicolumn{5}{c}{1} & &\multicolumn{5}{c}{4}  \\
\cmidrule{2-6}
\cmidrule{8-12}
& $L$ & $\operatorname{Area}_l$ & $\operatorname{Area}_h$ & $\operatorname{tVaR}_l$ & $\operatorname{tVaR}_h$ && $L$ & $\operatorname{Area}_l$ & $\operatorname{Area}_h$ & $\operatorname{tVaR}_l$ & $\operatorname{tVaR}_h$ \\ 
\midrule
vanilla & $>$1e6 & 0.82 & 4.88 & 2.76 & 8.67 & & $>$1e6 & 0.95 & 5.97 & 3.78 & 9.93 \\
TAF & 10.63 & 1.09 & 4.68 & 3.88 & 8.00 & & 9.97 & 1.12 & 5.90 & 4.28 & 9.42 \\
gTAF & 10.60 & 1.26 & 3.78 & 4.71 & 2.91 & & 9.84 & 1.32 & 4.89 & 5.29 &  4.22 \\
mTAF & 10.55 & 0.77 & 4.00 & 2.51 & 3.08 & & 9.81 & 0.48 & 5.01 & 1.32 & 4.49  \\
 \bottomrule
\end{tabular}
\end{center}
\end{table*}

\begin{figure}[ht]
\vskip 0.2in
\begin{minipage}{\textwidth}
\centerline{\includegraphics[width=0.24\columnwidth]{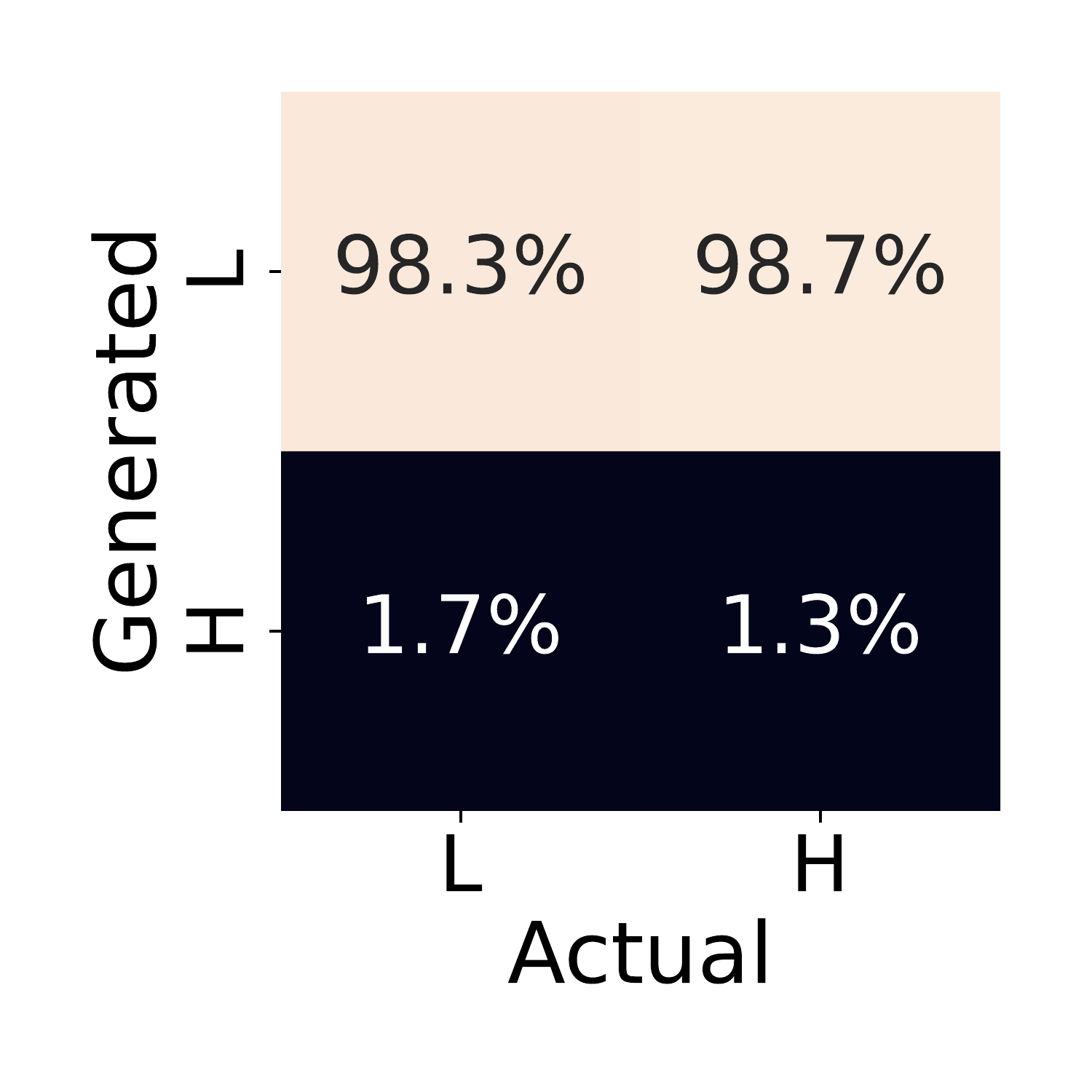} 
\includegraphics[width=0.24\columnwidth]{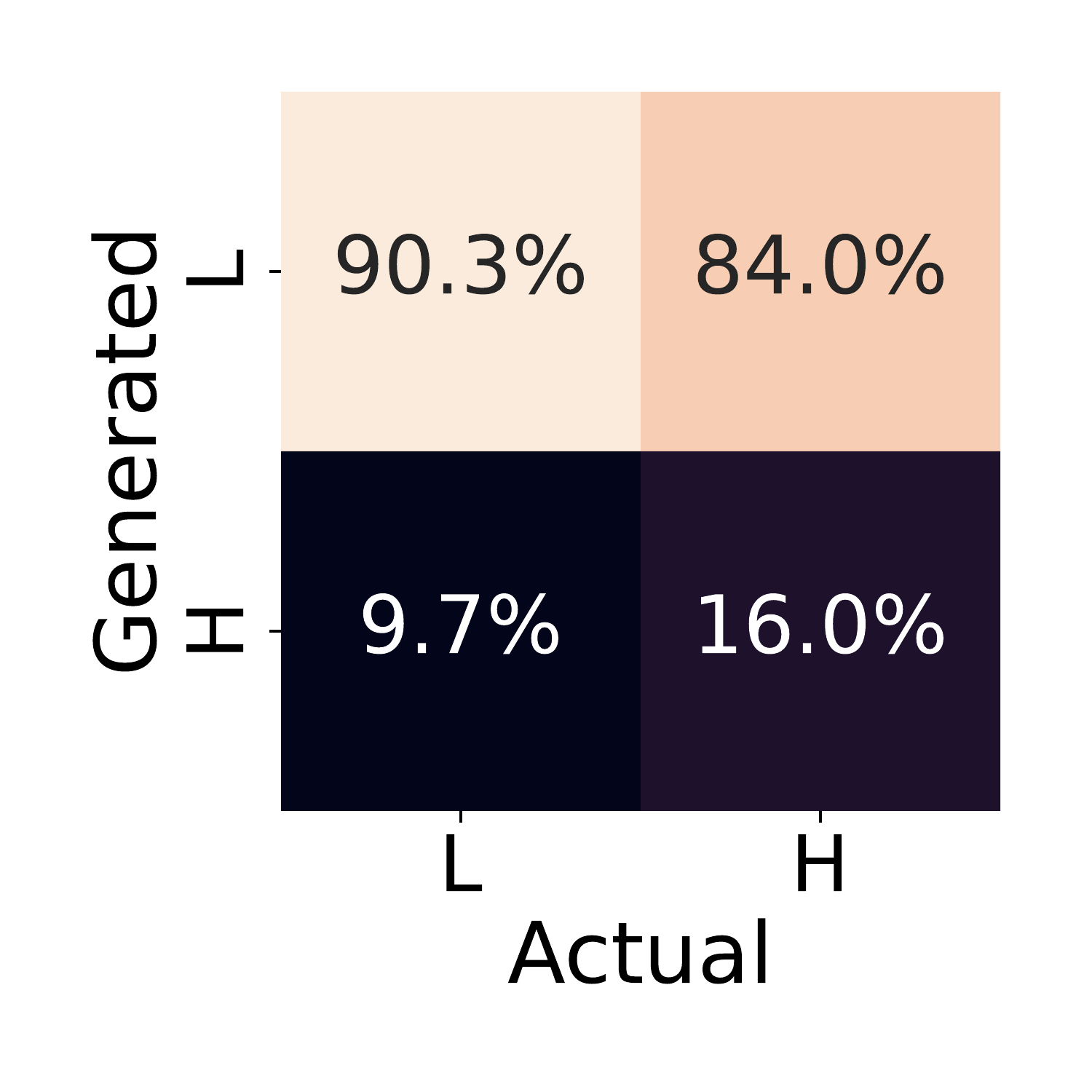}
\includegraphics[width=0.24\columnwidth]{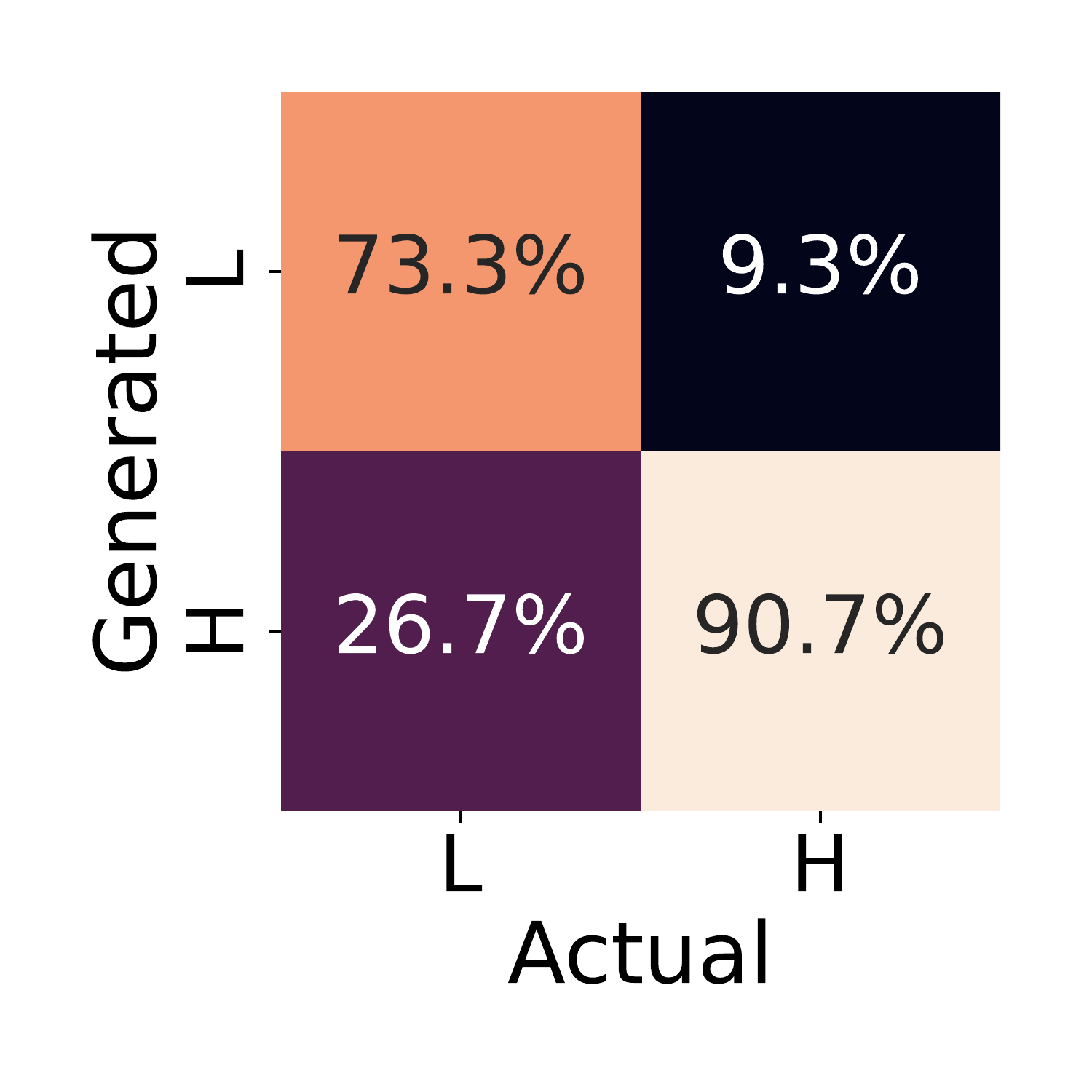}
\includegraphics[width=0.24\columnwidth]{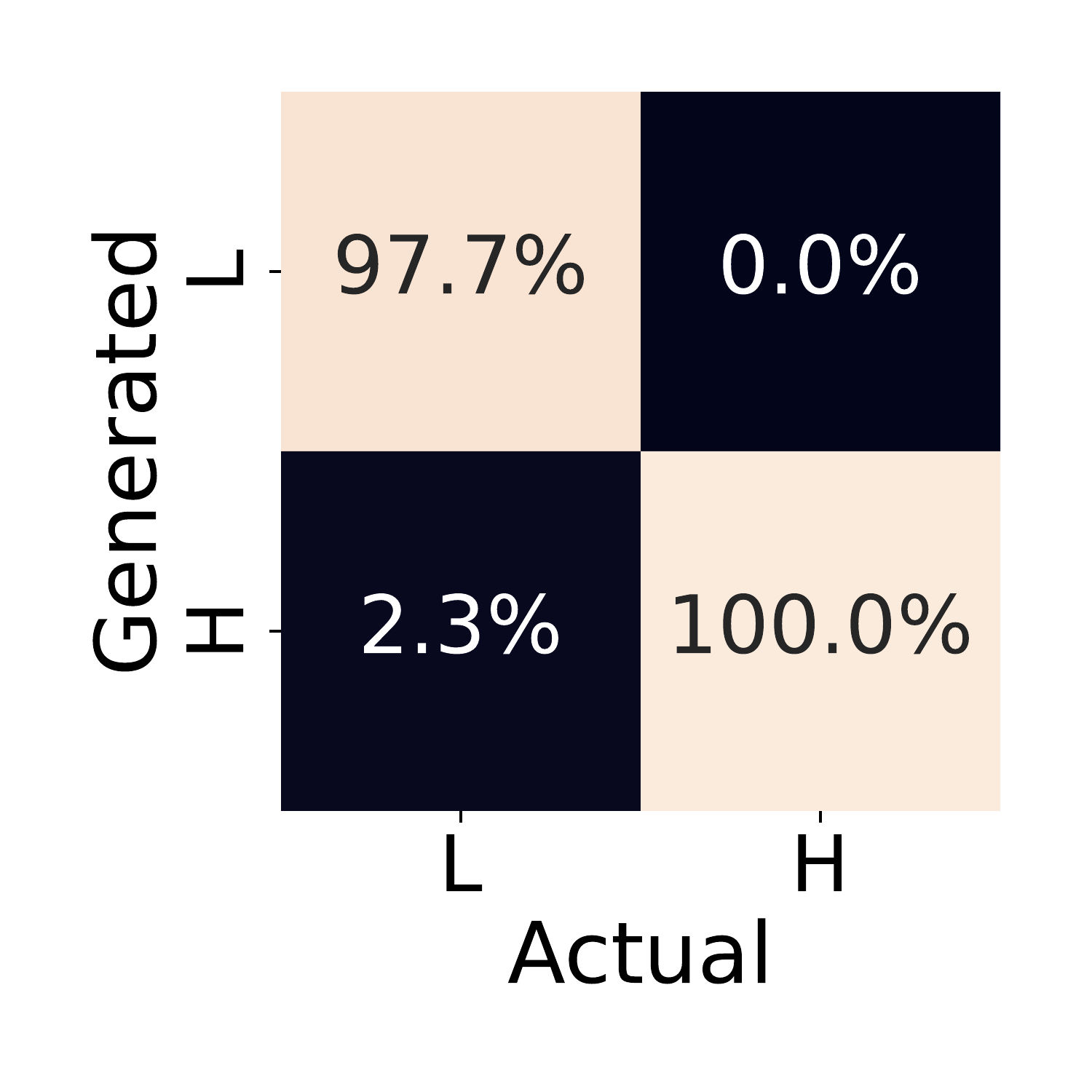}
}
\centering
\subcaption{$\nu=2, \; d_h =1$}

\end{minipage}

\begin{minipage}{\textwidth}
\centerline{\includegraphics[width=0.24\columnwidth]{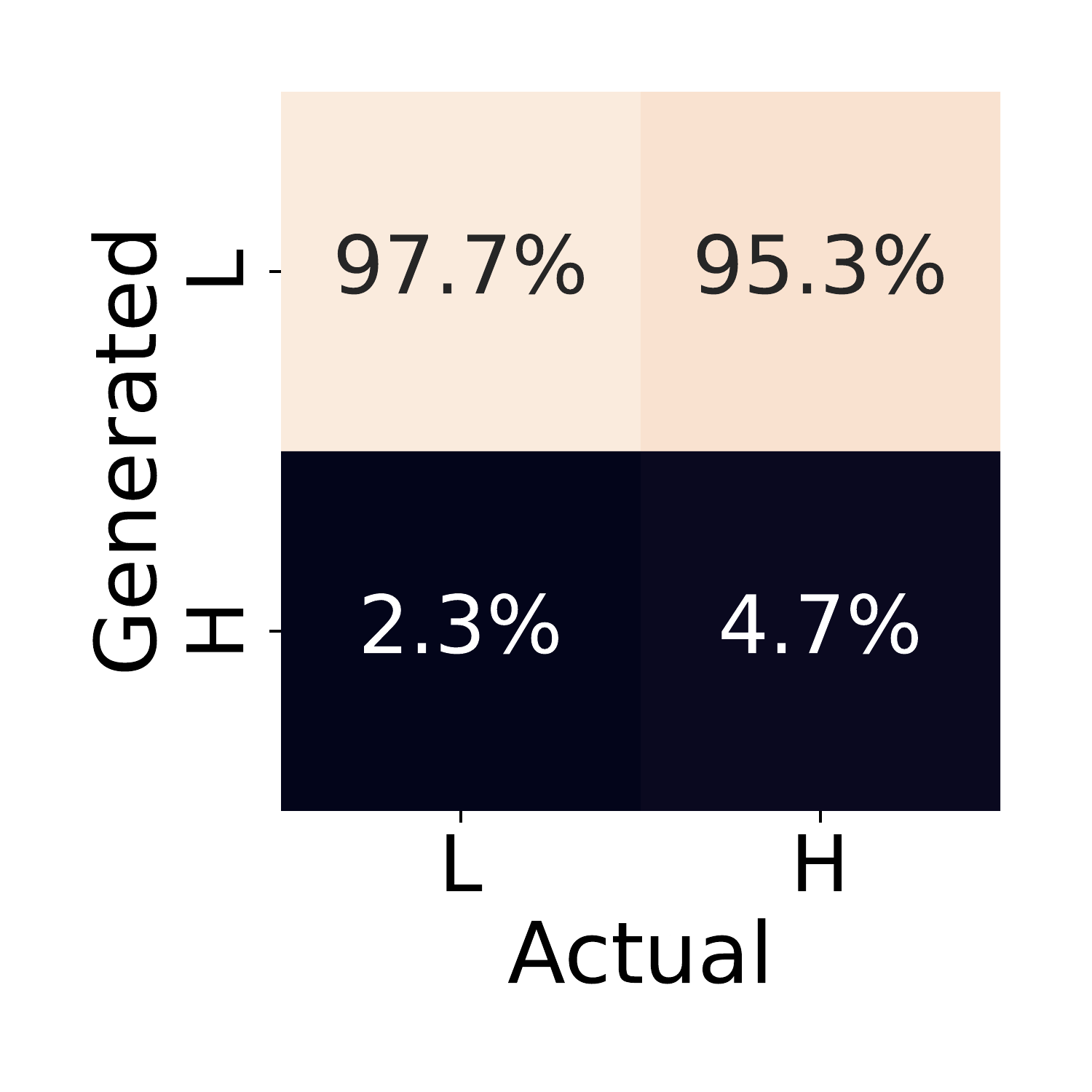} 
\includegraphics[width=0.24\columnwidth]{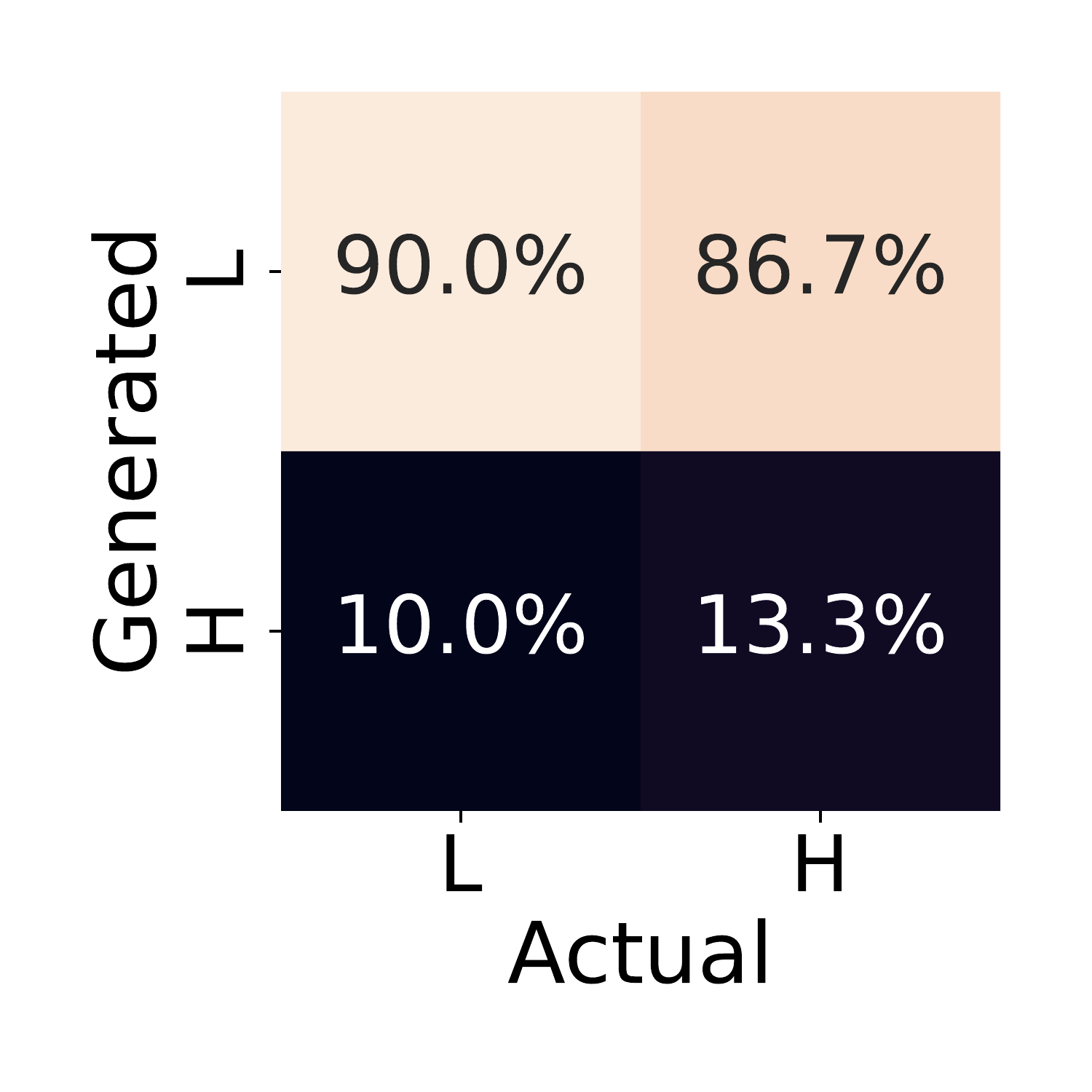}
\includegraphics[width=0.24\columnwidth]{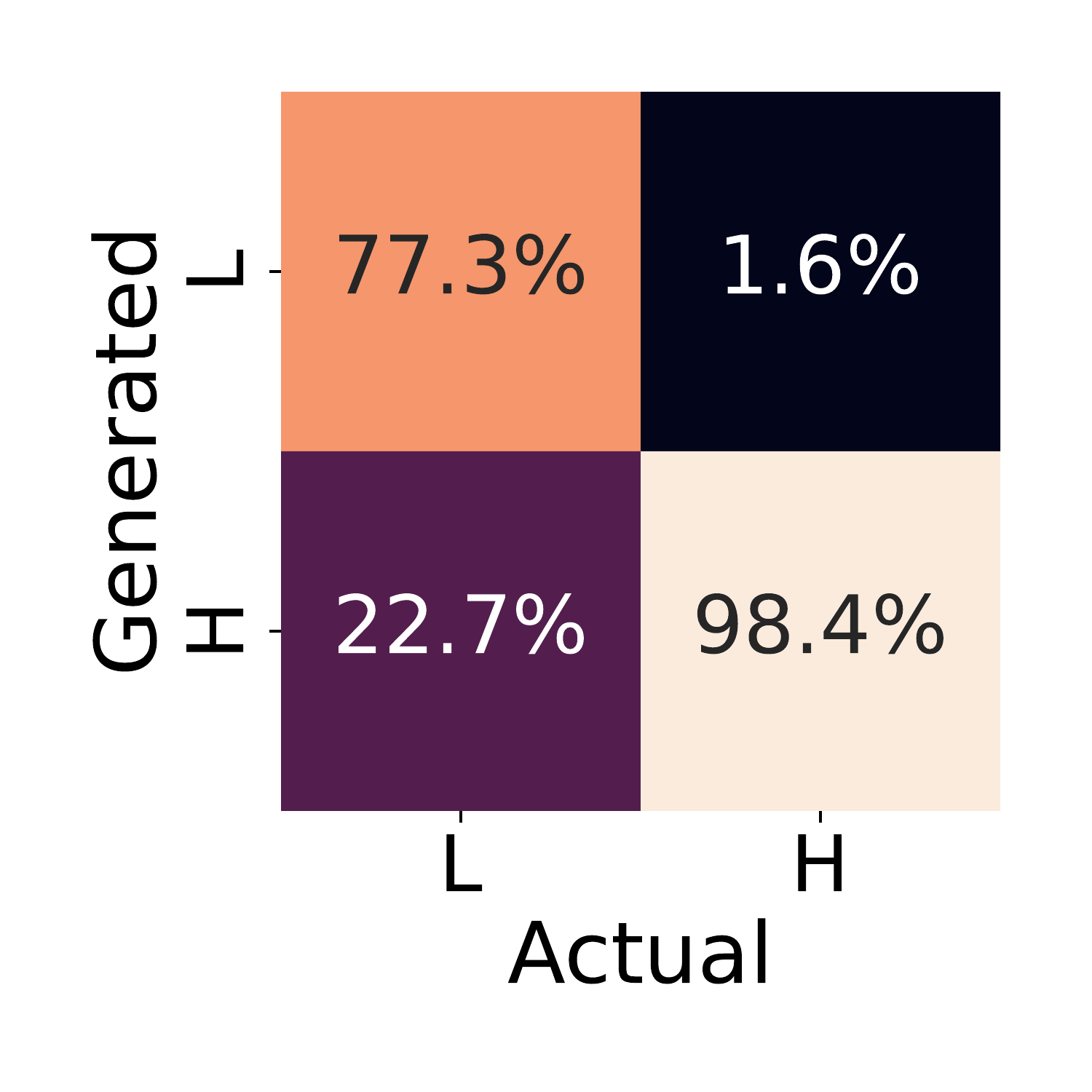}
\includegraphics[width=0.24\columnwidth]{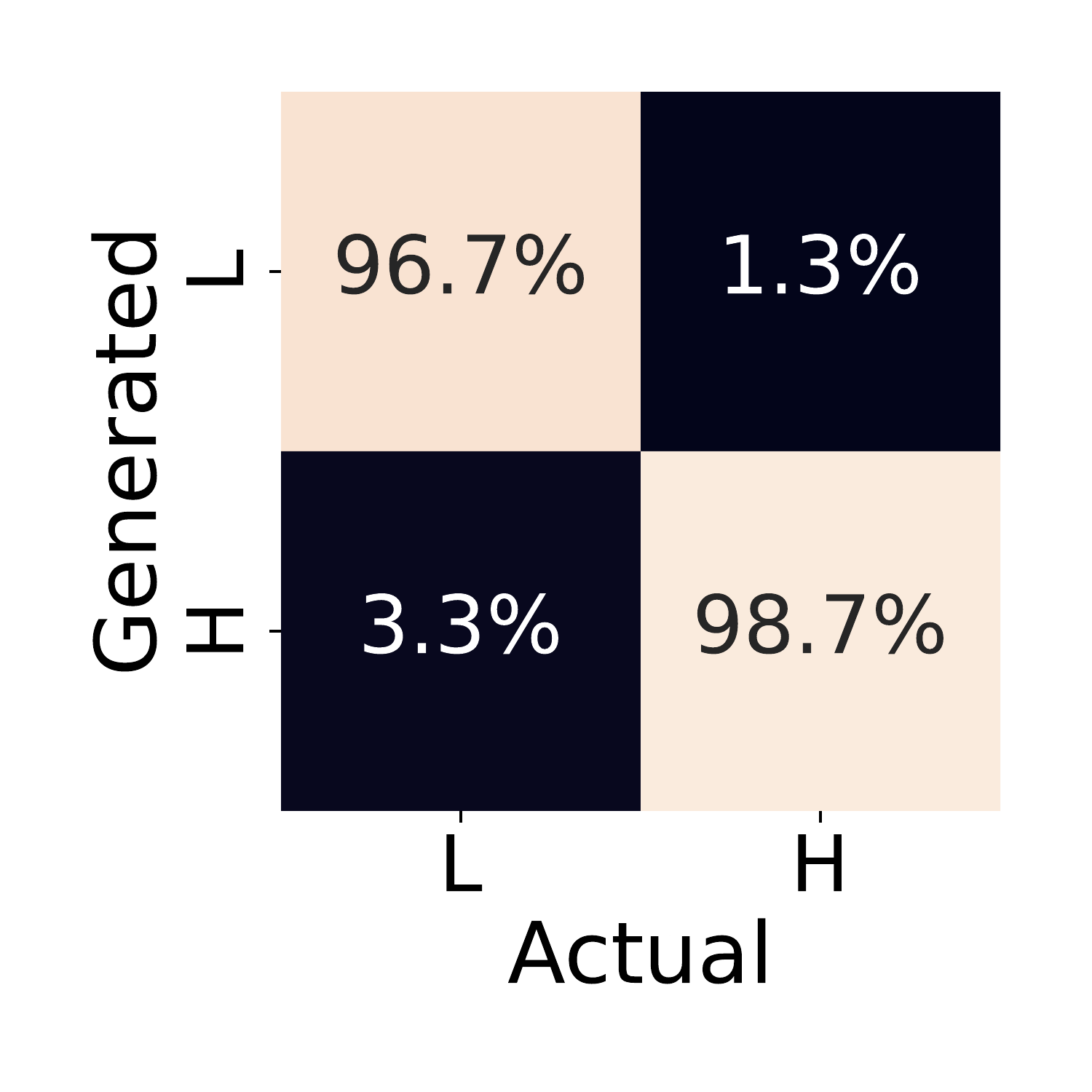}
}
\centering
\subcaption{$\nu=3, \; d_h =4$}

\end{minipage}

\begin{minipage}{\textwidth}
\centerline{\includegraphics[width=0.24\columnwidth]{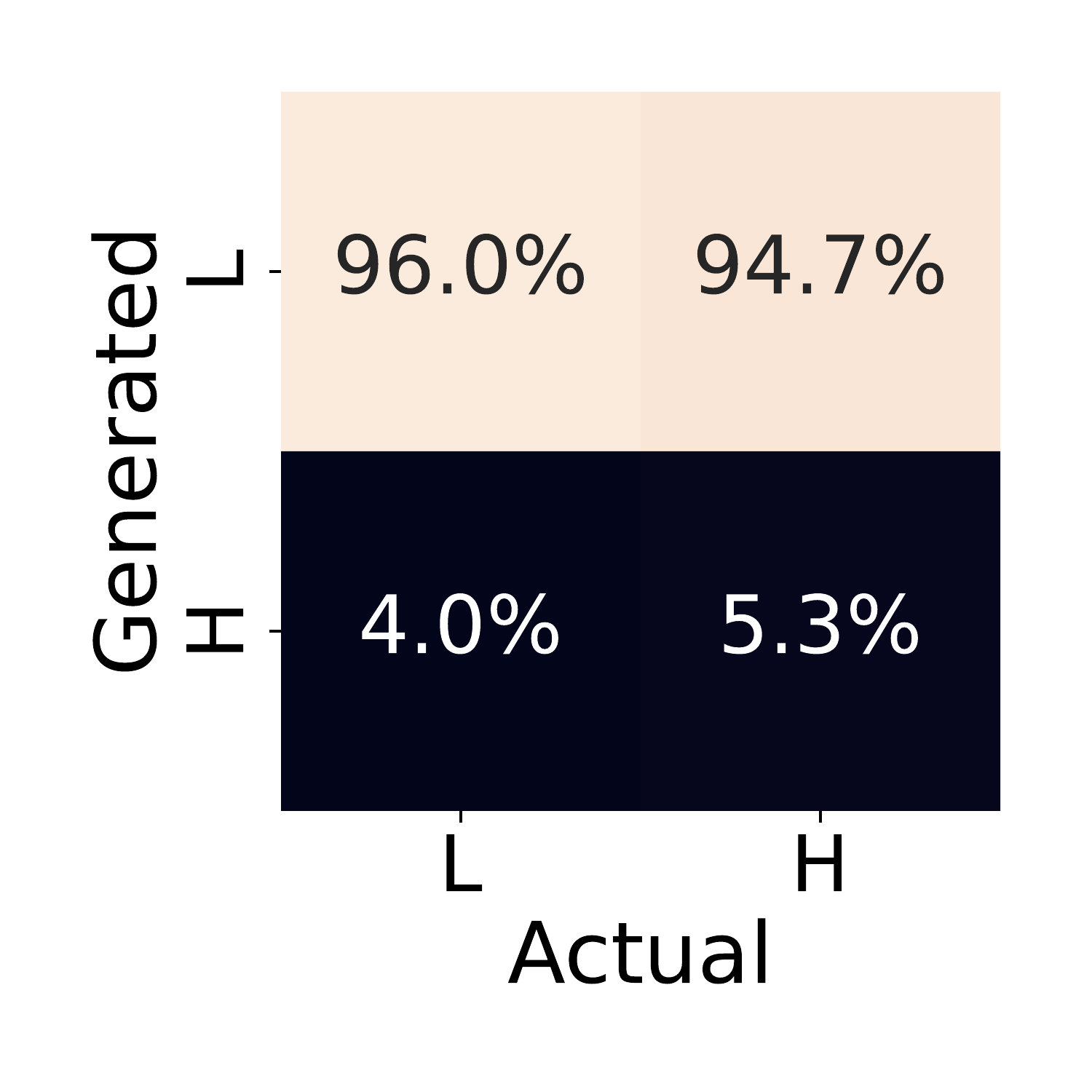} 
\includegraphics[width=0.24\columnwidth]{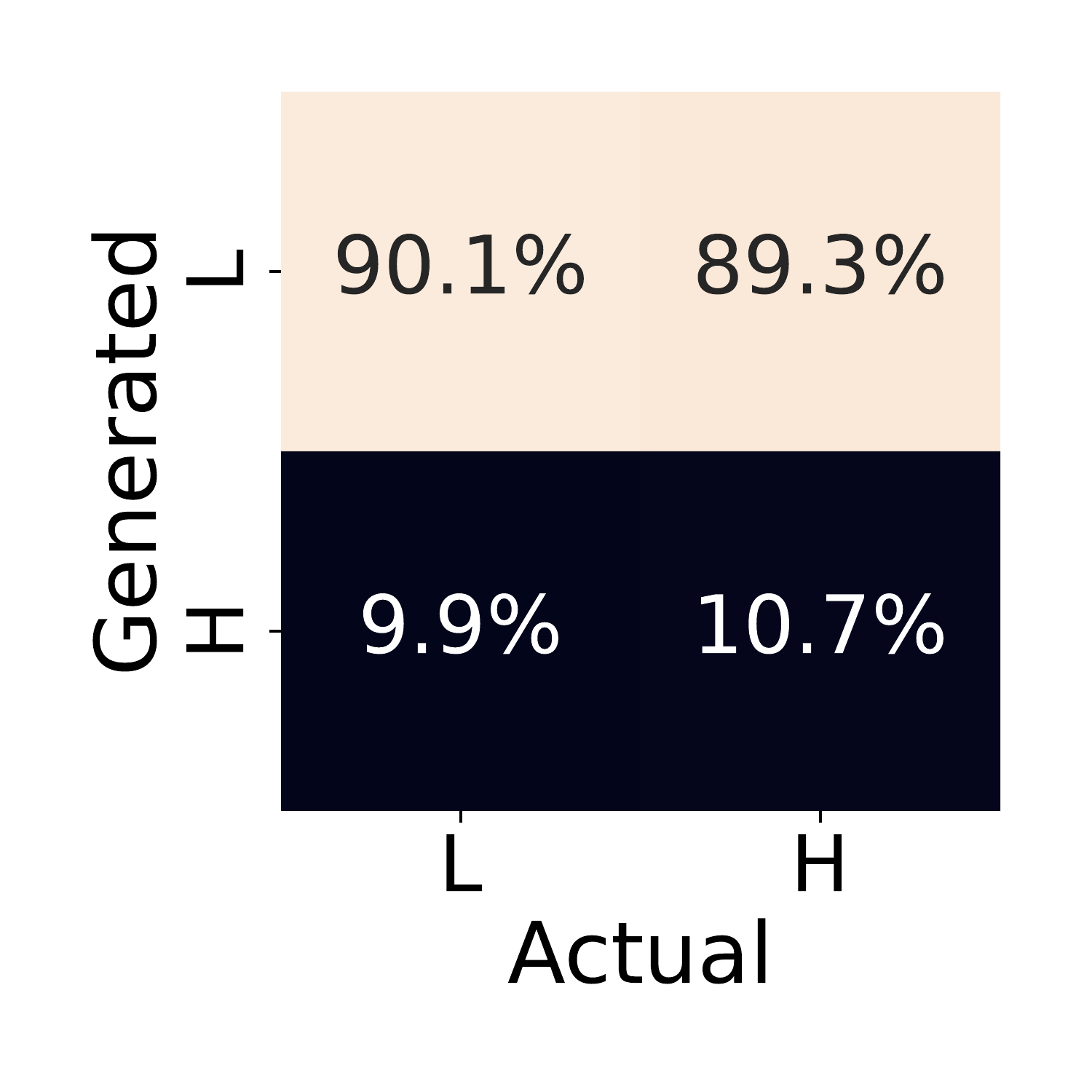}
\includegraphics[width=0.24\columnwidth]{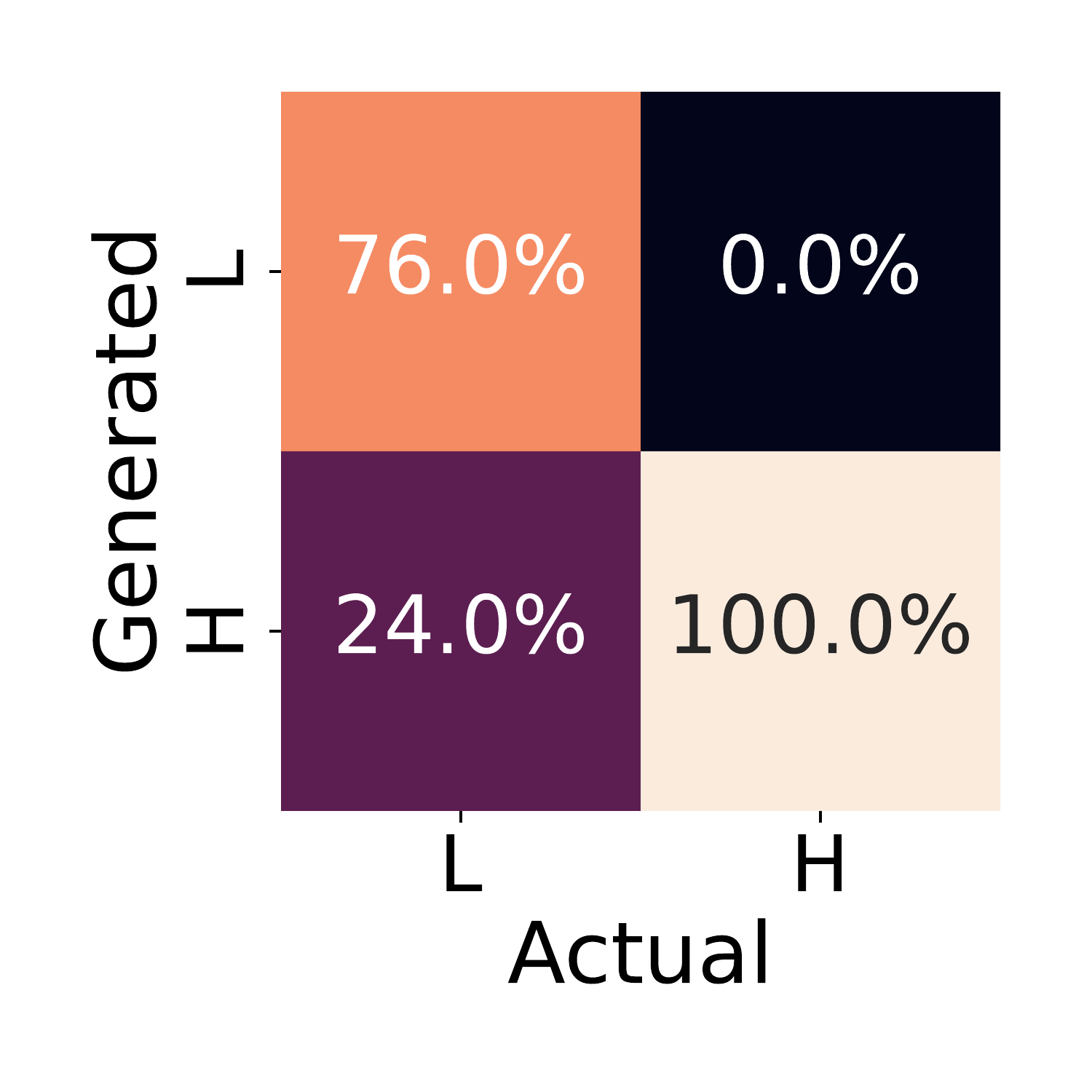}
\includegraphics[width=0.24\columnwidth]{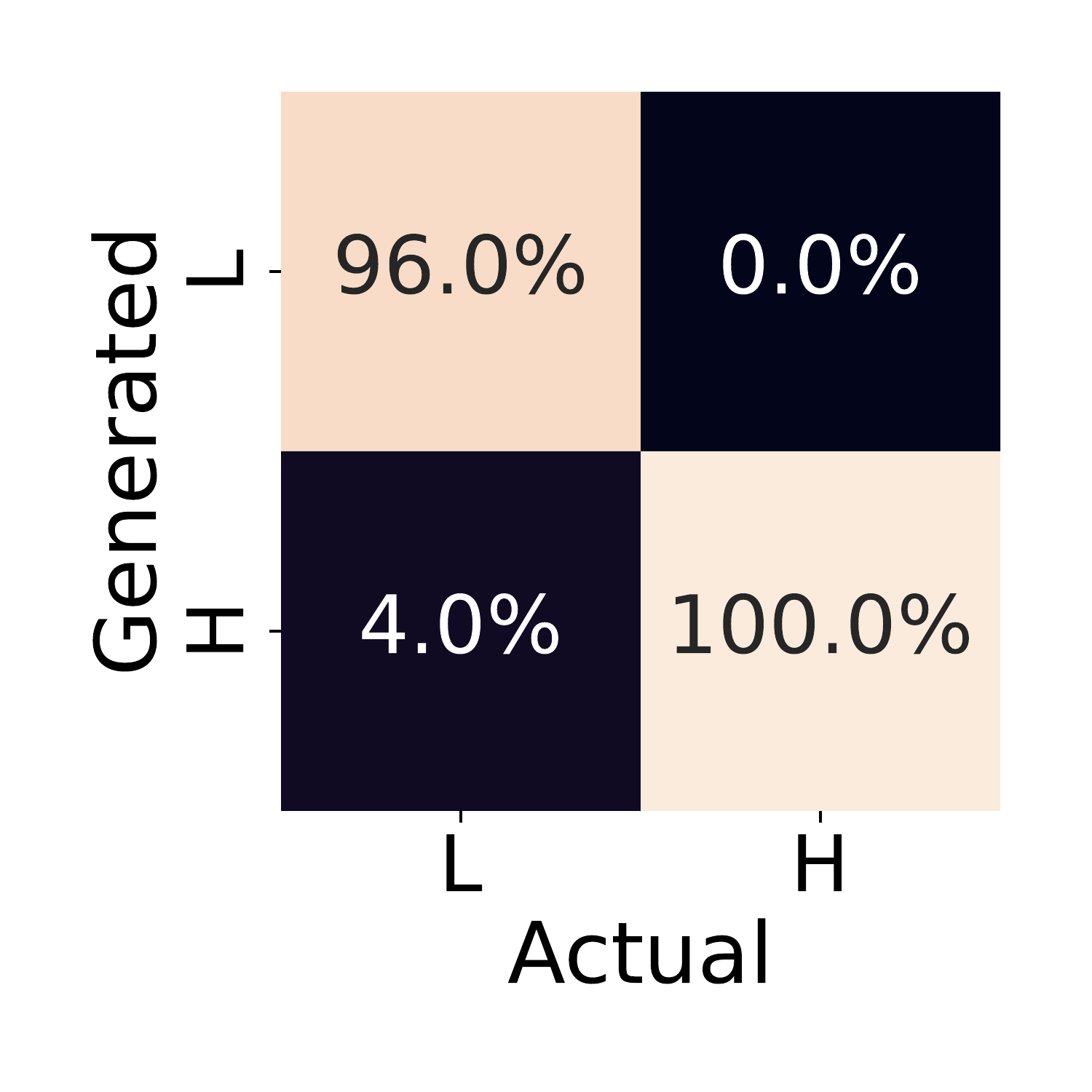}
}
\centering
\subcaption{$\nu=3, \; d_h =1$}

\end{minipage}
\caption{Marginal tail estimation based on synthetic flow samples of vanilla, TAF, gTAF, and mTAF (from left to right) for varying tail index $\nu$ and number of heavy-tailed components $d_h$. We classify marginals whose tail estimator is less than $10$ as heavy-tailed, otherwise it is classified as light-tailed.}
\label{fig:synth_tail_supp}
\vskip -0.2in
\end{figure}

%\begin{figure}[ht]
%\vskip 0.2in
%\begin{center}
%\begin{minipage}{0.3\textwidth}
%\centering
%\centerline{\includegraphics[width=\columnwidth]{plots/synth_tailest/synth_tailest_df2h1.pdf}}
%\subcaption{(a) $\nu=2, \; h=1$}
%\label{fig:synth_taildf2h1}
%\end{minipage}
%\begin{minipage}{0.3\textwidth}
%\centering
%\centerline{\includegraphics[width=\columnwidth]{plots/synth_tailest/synth_tailest_df3h4.pdf}}
%\subcaption{(b) $\nu=3, \; h=4$}
%\label{fig:synth_taildf3h4}
%\end{minipage}
%\begin{minipage}{0.3\textwidth}
%\centering
%\centerline{\includegraphics[width=\columnwidth]{plots/synth_tailest/synth_tailest_df3h1.pdf}}
%\label{fig:synth_taildf3h1}
%\subcaption{(c) $\nu=3, \; h=1$}
%\end{minipage}
%\end{center}
%\vskip -0.2in
%\caption{Marginal tail estimation based on synthetic flow samples in $3$ different settings. The amount of light-tailed components whose samples are estimated as heavy-tailed (left in each subfigure) and the estimated tail indices of the heavy-tailed components (right in each subfigure). In the right plots, the $4$ bars close to $10$ denote components that are estimated as light-tailed.} \label{fig:synth_tail_supp}
%\end{figure}

\begin{table*}[t]
\caption{Average test loss, Area under log-log plot, and $\operatorname{tVaR}$ (lower is better for each metric) in the setting $D=50, \; \nu \in \{2, 3\},$ and $d_h =10$. The copula model serves as an oracle baseline. \label{table:synth_metrics_50D=2}}
\begin{center}
\begin{tabular}{@{}lccccccccccc@{}}
\toprule
$\nu$ & \multicolumn{5}{c}{2} & &\multicolumn{5}{c}{3}  \\
\cmidrule{2-6}
\cmidrule{8-12}
& $L$ & $\operatorname{Area}_l$ & $\operatorname{Area}_h$ & $\operatorname{tVaR}_l$ & $\operatorname{tVaR}_h$ && $L$ & $\operatorname{Area}_l$ & $\operatorname{Area}_h$ & $\operatorname{tVaR}_l$ & $\operatorname{tVaR}_h$ \\ 
\midrule
vanilla & 58.59 & 0.29 & 3.38 & 0.45 & 24.61 && 62.59 & 0.30 & 1.71 & 0.43 & 6.38  \\
TAF & 58.12 & 0.67 & 3.02 & 1.20 & 3.19 && 62.47 & 0.57 & 2.02 & 0.96 & 2.15  \\
gTAF & 58.05 & 0.94 & 0.58 & 1.23 & 1.16 && 62.42 & 0.61 & 0.51 & 0.96 & 0.99  \\
mTAF & 58.17 & 0.39 & 2.30 & 0.84 & 4.20 && 62.60 & 0.30 & 1.32 & 0.46 & 2.86  \\
\midrule
copula & 57.23 & 0.21 & 1.04 & 0.42 & 2.50 && 56.73 & 0.20 & 0.66 & 0.41 & 1.33  \\
 \bottomrule
\end{tabular}
\end{center}
\end{table*}

\begin{figure}[ht]
\vskip 0.2in
\begin{center}
\begin{minipage}{0.24\textwidth}
\centerline{\includegraphics[width=\columnwidth]{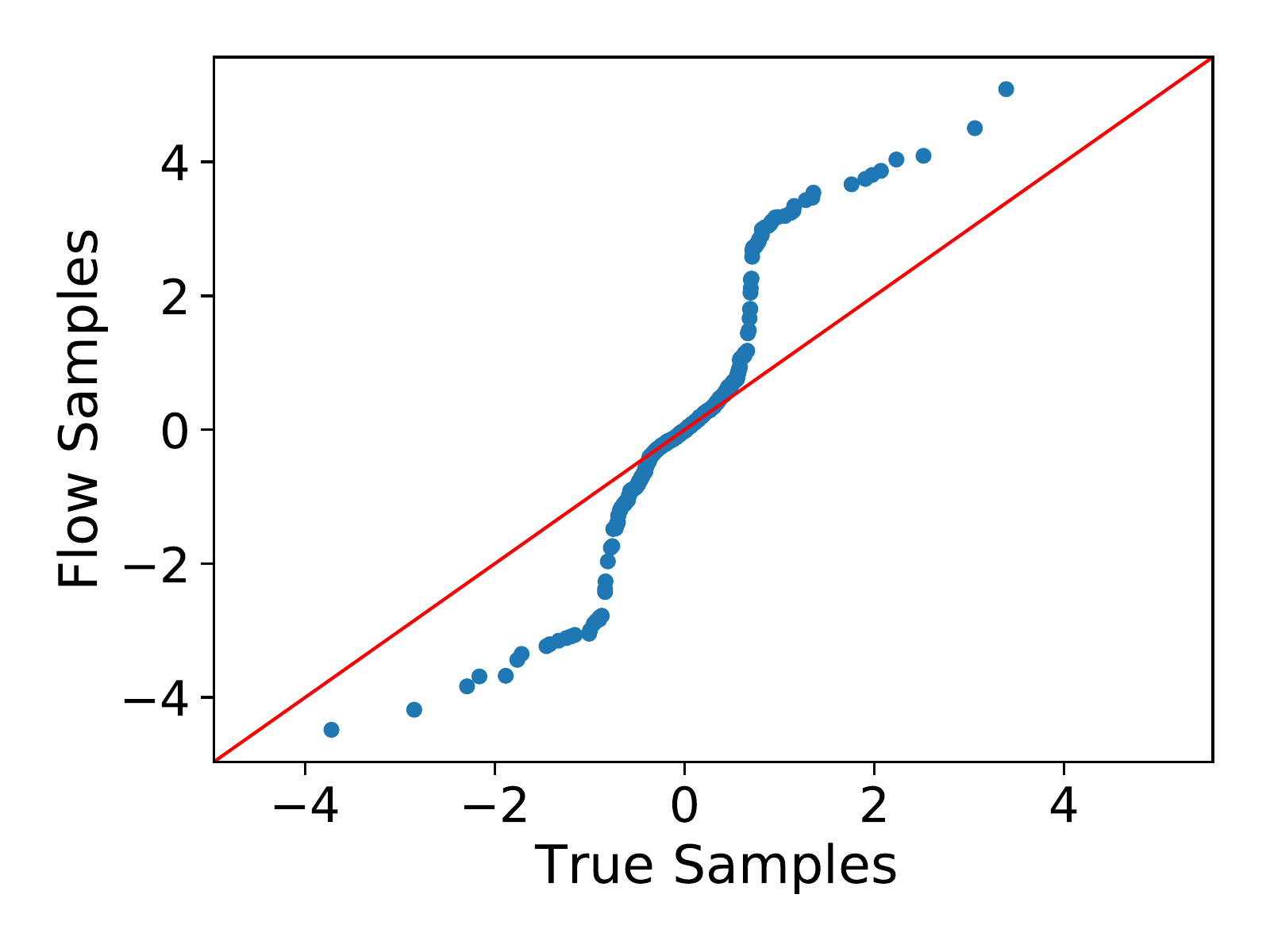}}
\end{minipage}
\begin{minipage}{0.24\textwidth}
\centerline{\includegraphics[width=\columnwidth]{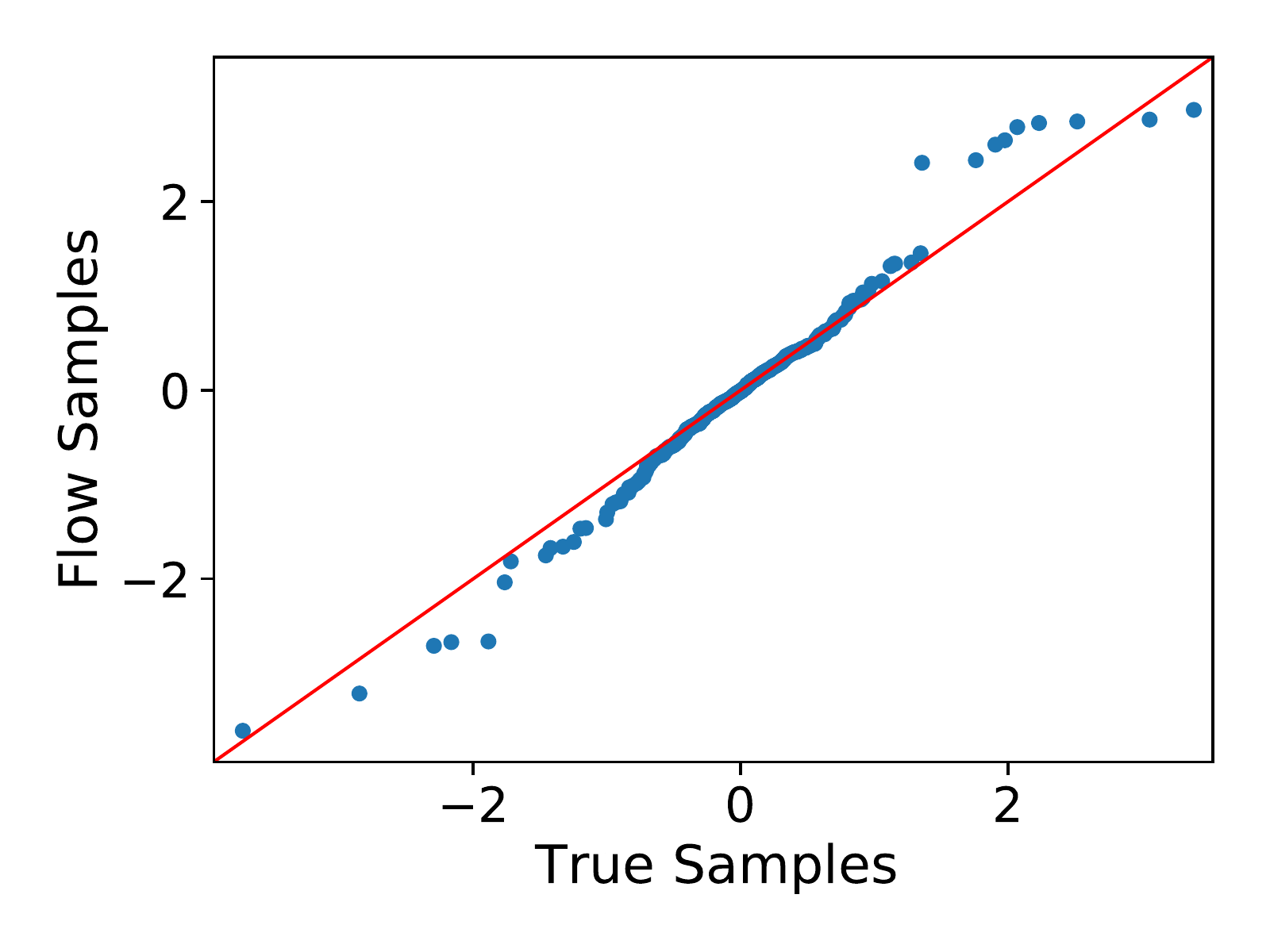}}
\end{minipage}
\begin{minipage}{0.24\textwidth}
\centerline{\includegraphics[width=\columnwidth]{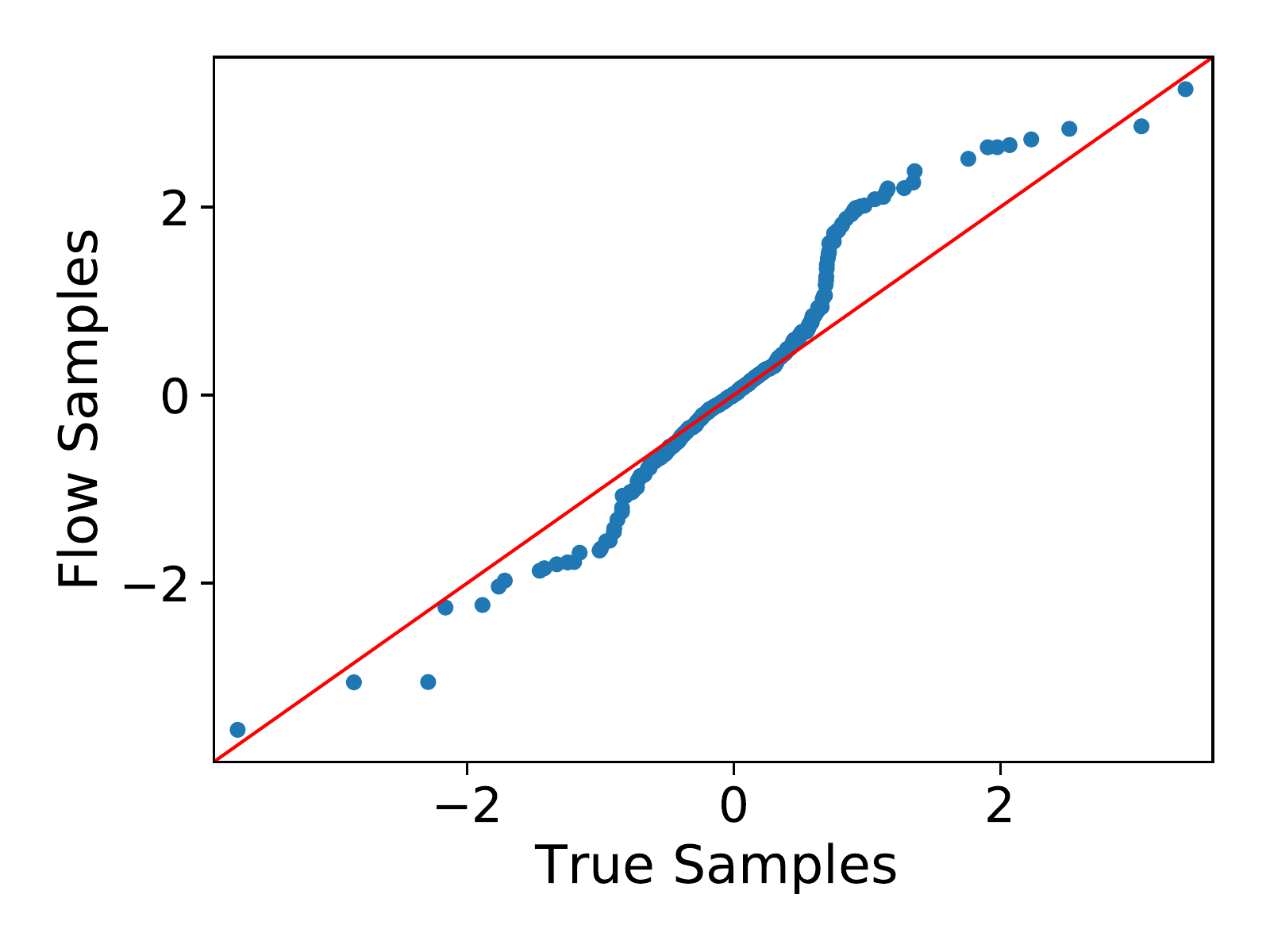}}
\end{minipage}
\begin{minipage}{0.24\textwidth}
\centerline{\includegraphics[width=\columnwidth]{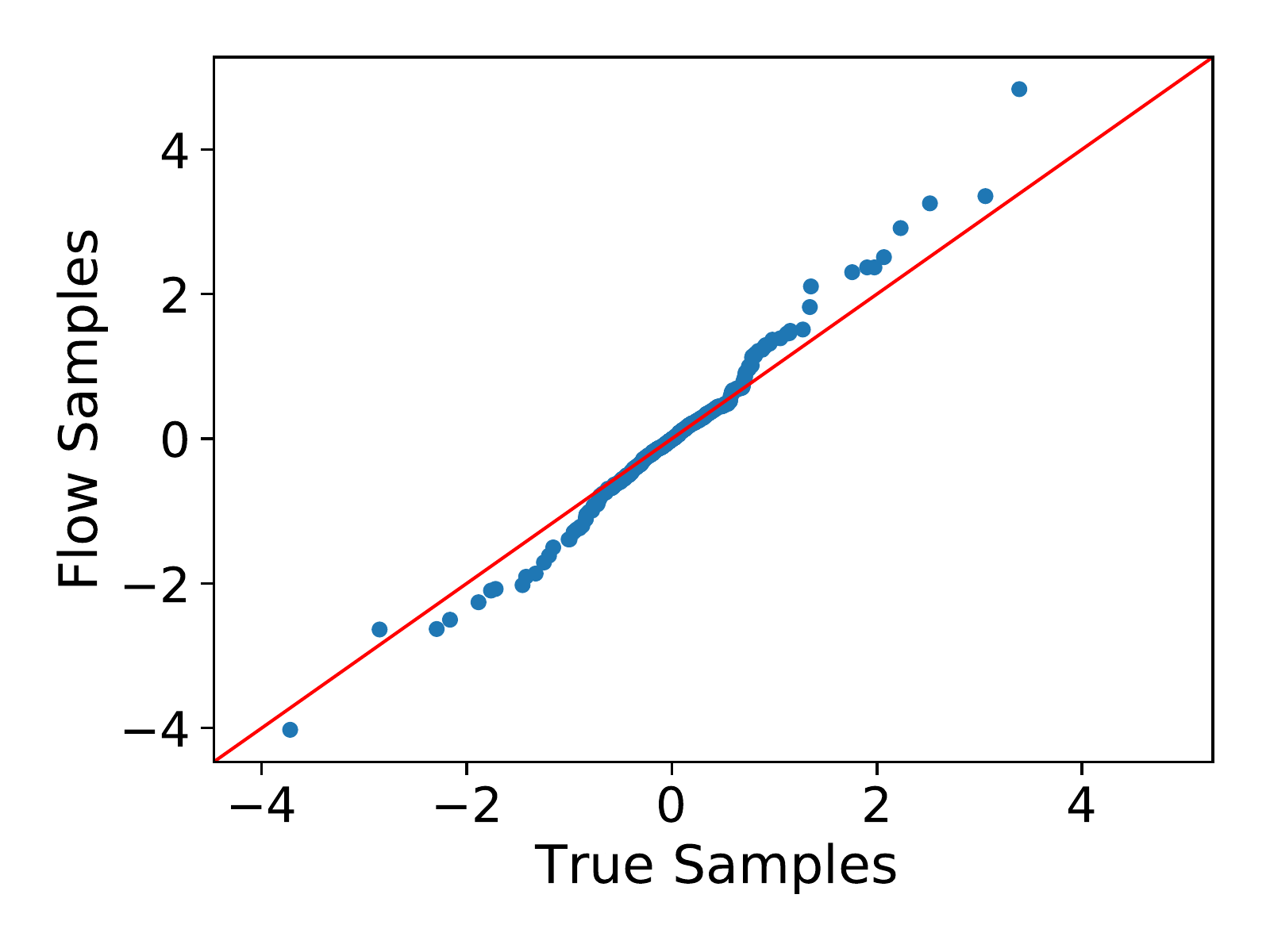}}
\end{minipage}

\end{center}
\vskip -0.2in \caption{QQ-plots for the $8$th heavy-tailed marginal in the setting $\nu=2$ and $d_h =1$. The QQ-plots correspond to samples generated by vanilla, TAF, gTAF, and \name{}, respectively.}
\label{fig:qq_df2h1}
\end{figure}

\begin{figure}[ht]
\vskip 0.2in
\begin{center}
\begin{minipage}{0.24\textwidth}
\centerline{\includegraphics[width=\columnwidth]{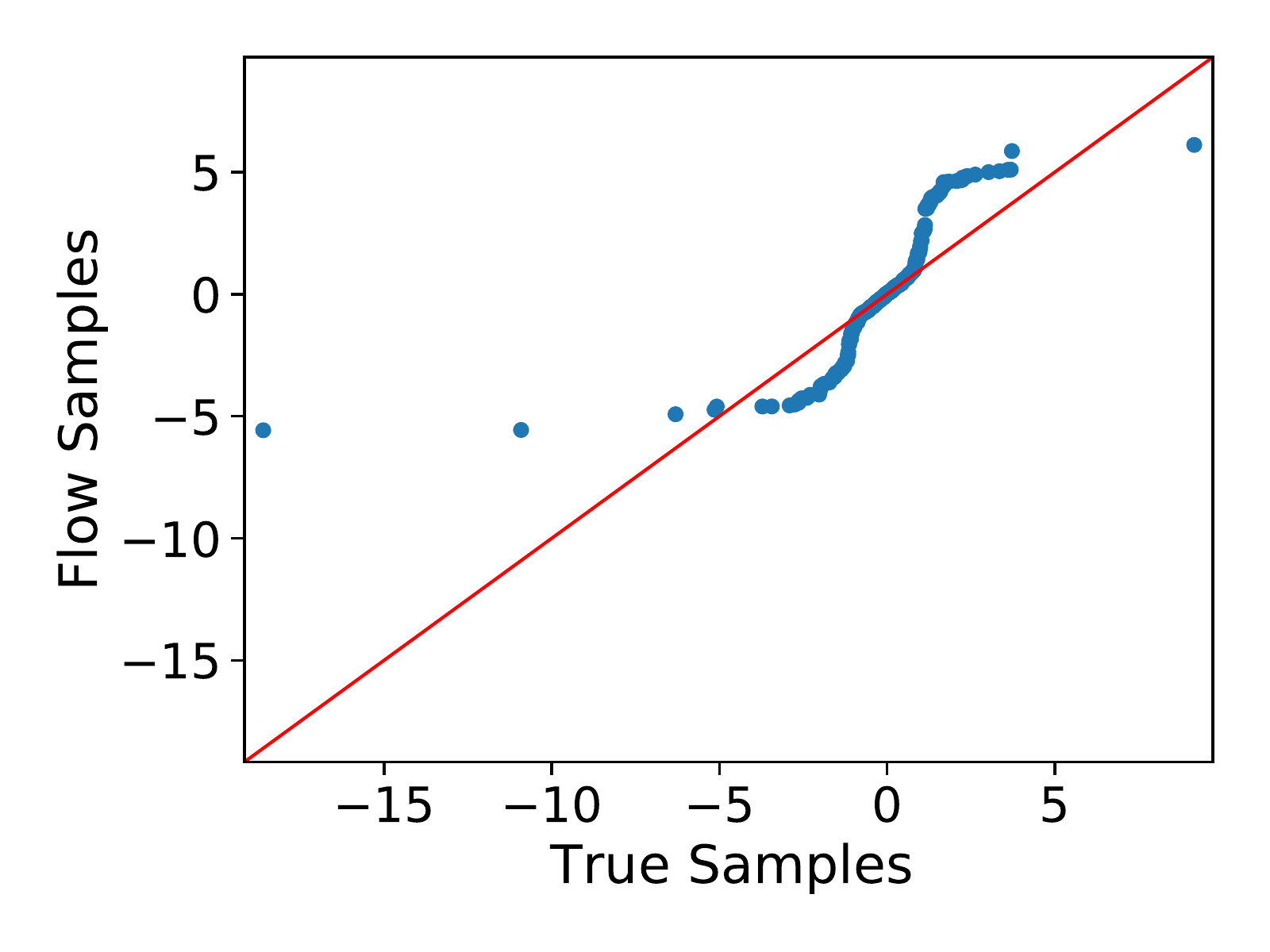}}
\end{minipage}
\begin{minipage}{0.24\textwidth}
\centerline{\includegraphics[width=\columnwidth]{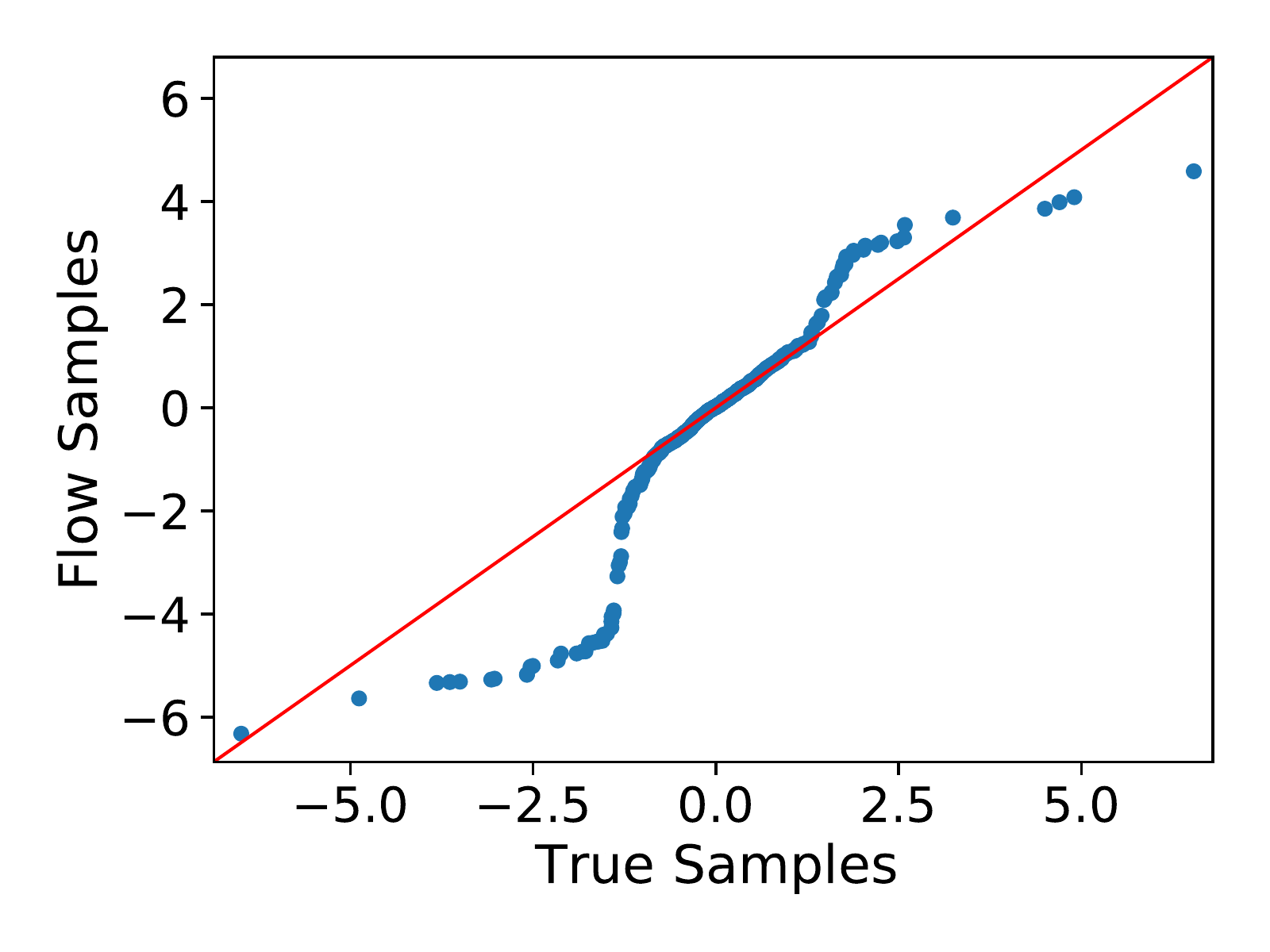}}
\end{minipage}
\begin{minipage}{0.24\textwidth}
\centerline{\includegraphics[width=\columnwidth]{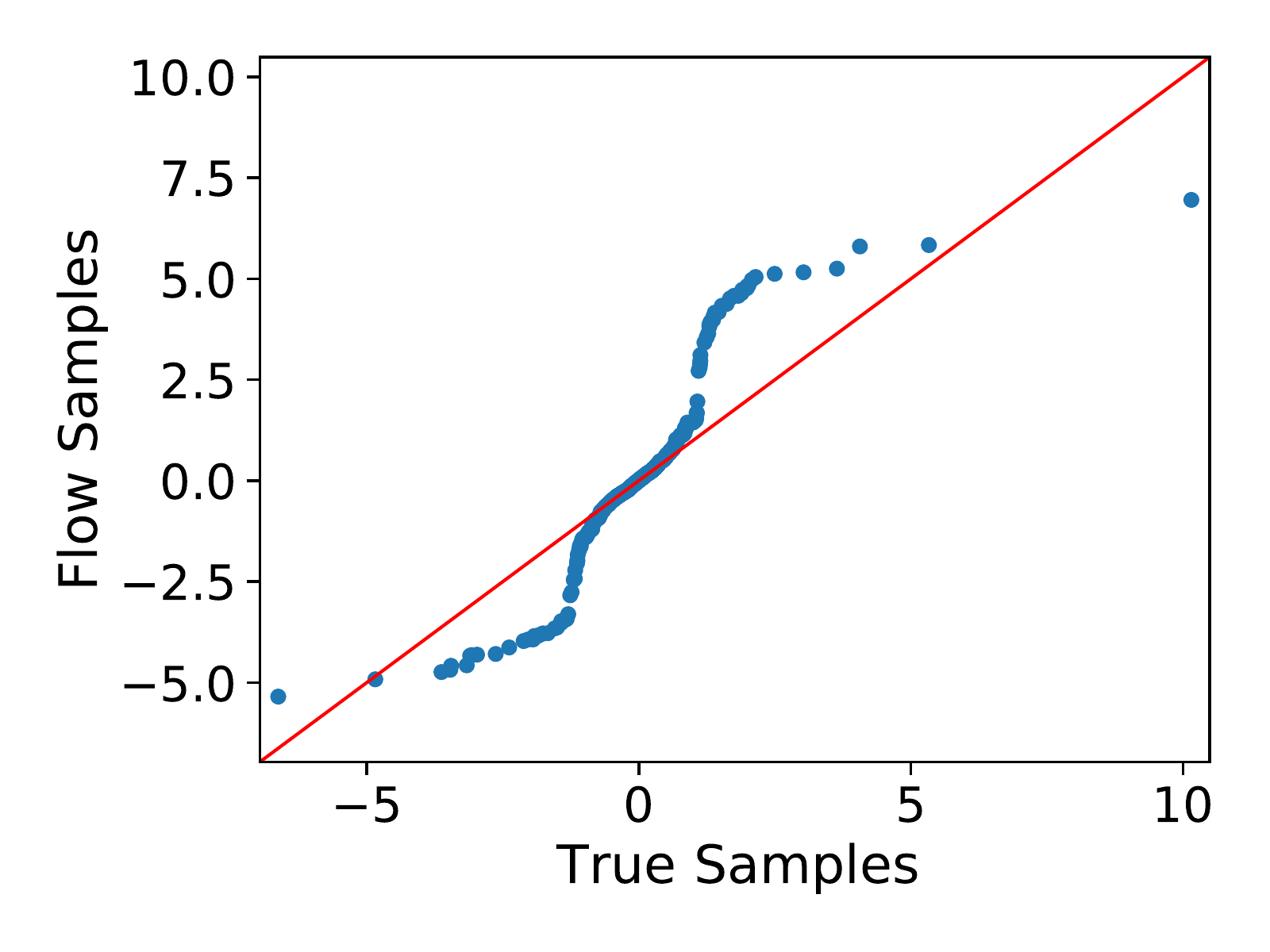}}
\end{minipage}
\begin{minipage}{0.24\textwidth}
\centerline{\includegraphics[width=\columnwidth]{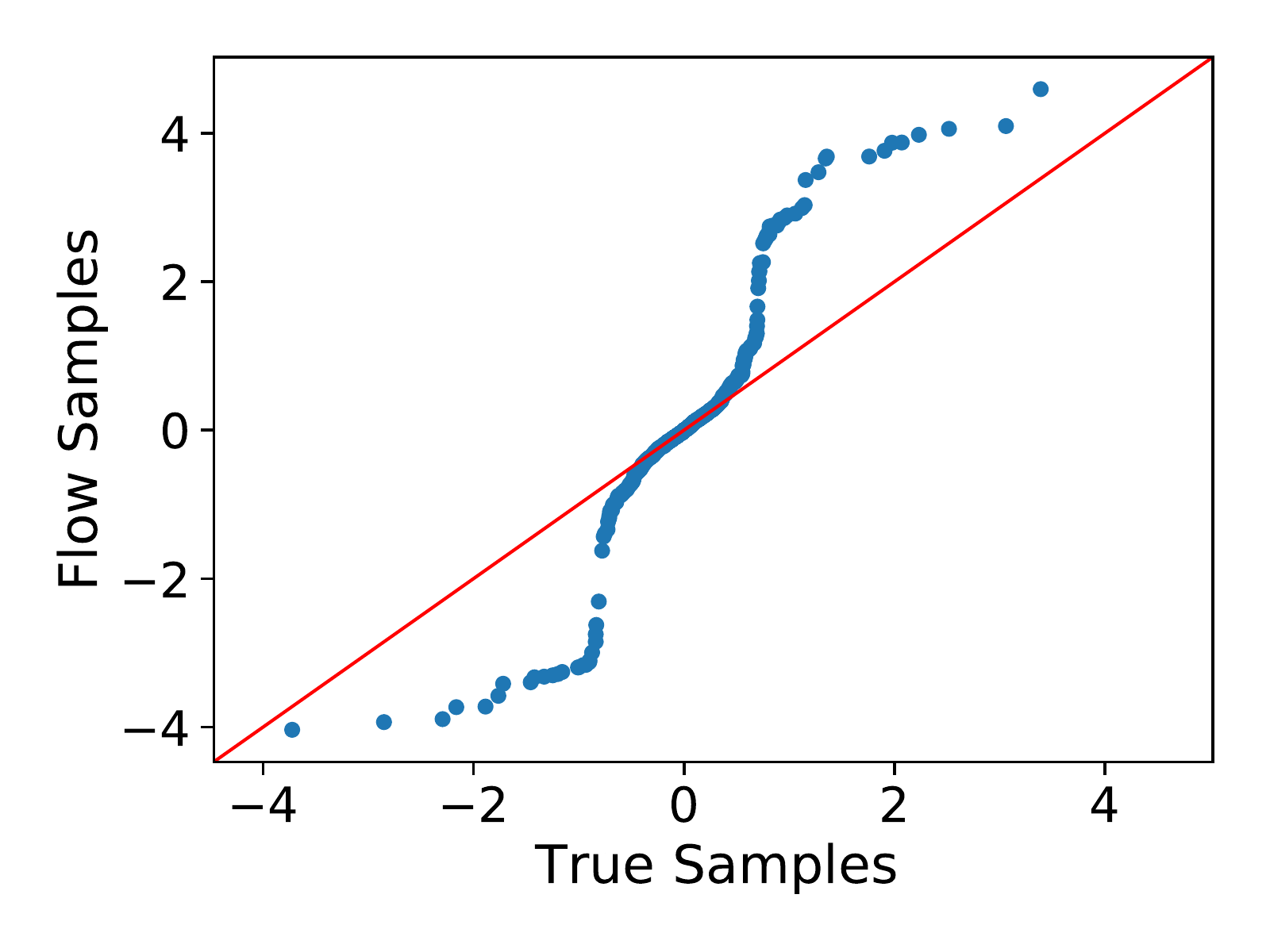}}
\end{minipage}

\begin{minipage}{0.24\textwidth}
\centerline{\includegraphics[width=\columnwidth]{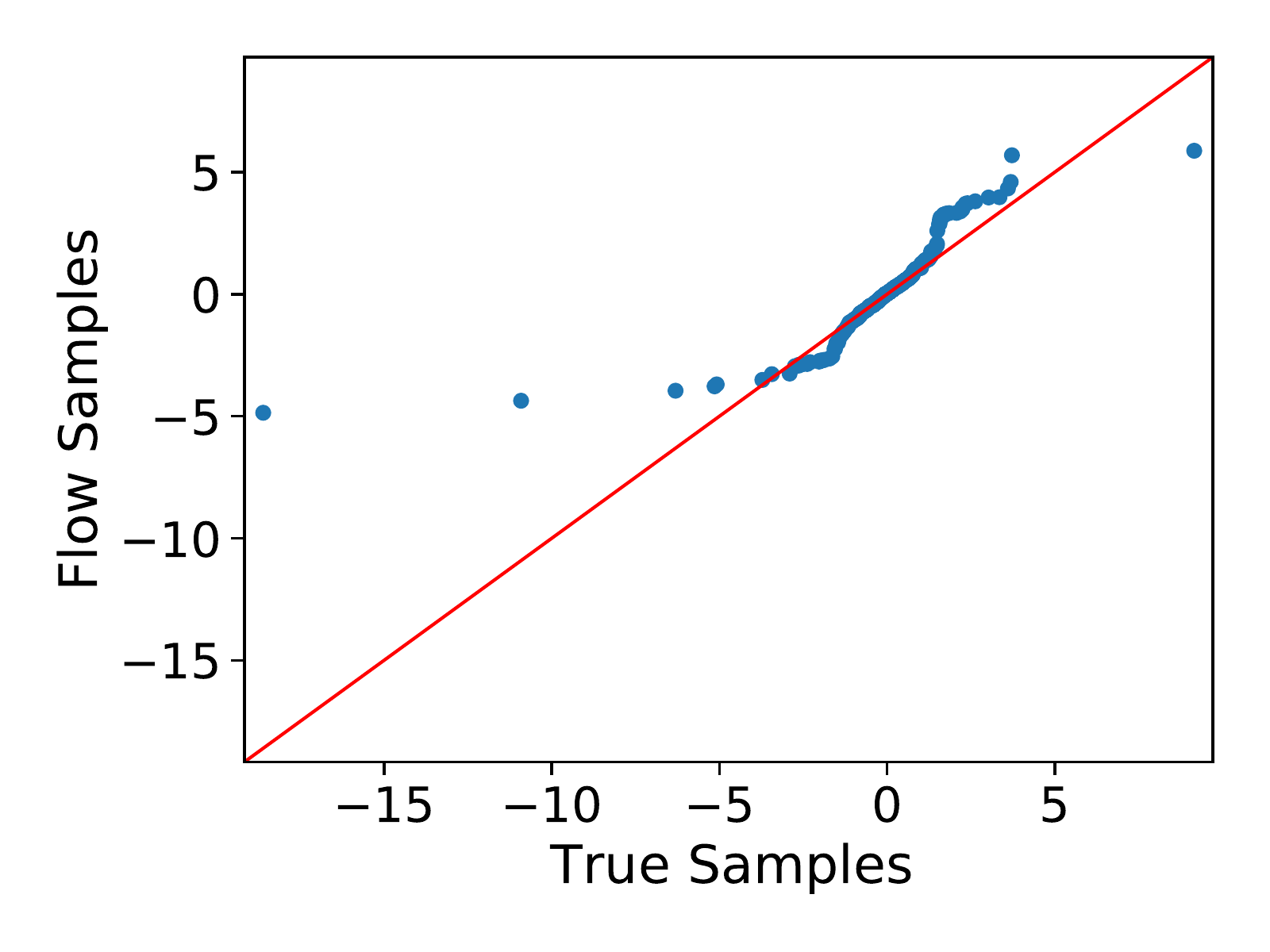}}
\end{minipage}
\begin{minipage}{0.24\textwidth}
\centerline{\includegraphics[width=\columnwidth]{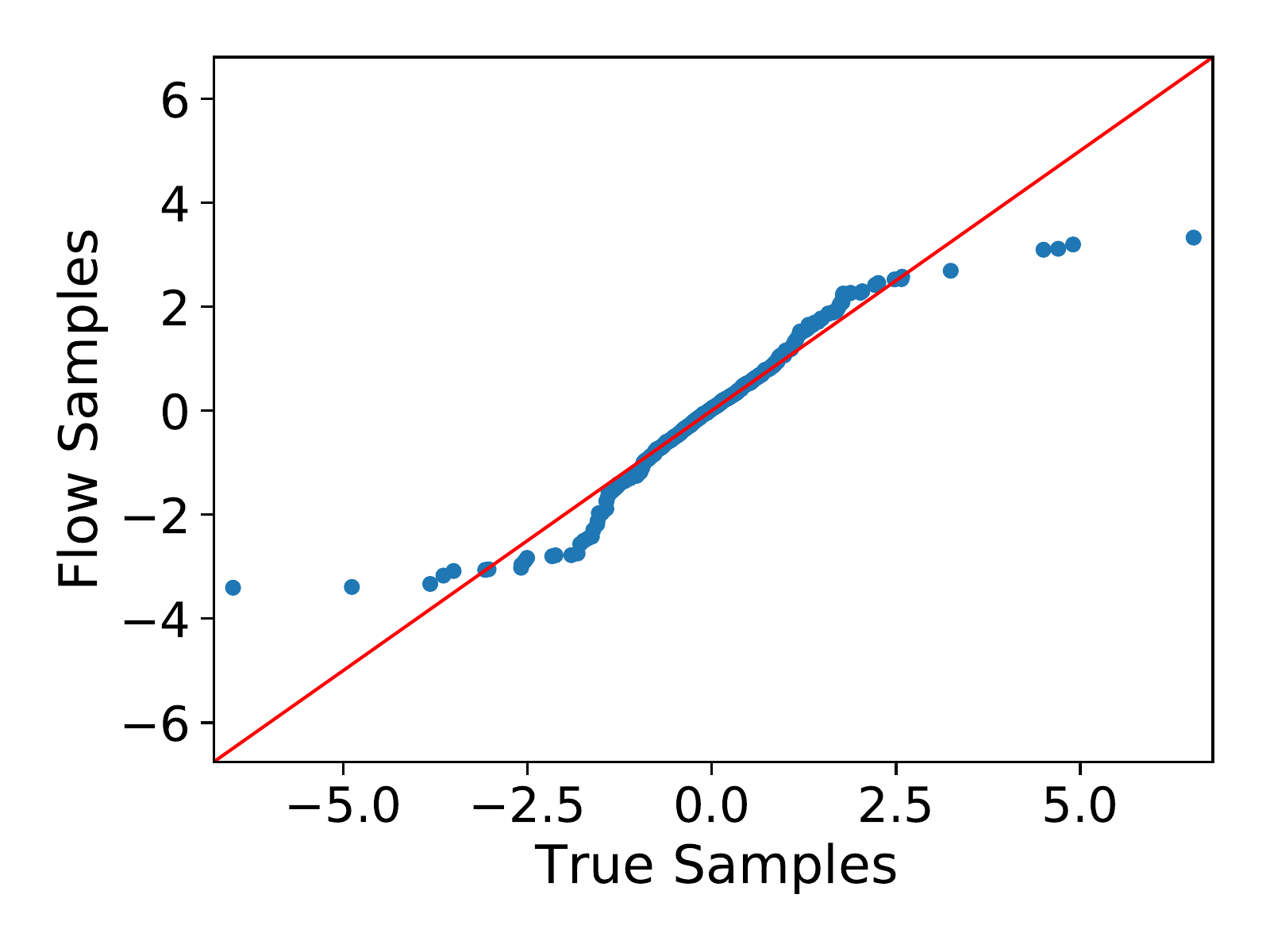}}
\end{minipage}
\begin{minipage}{0.24\textwidth}
\centerline{\includegraphics[width=\columnwidth]{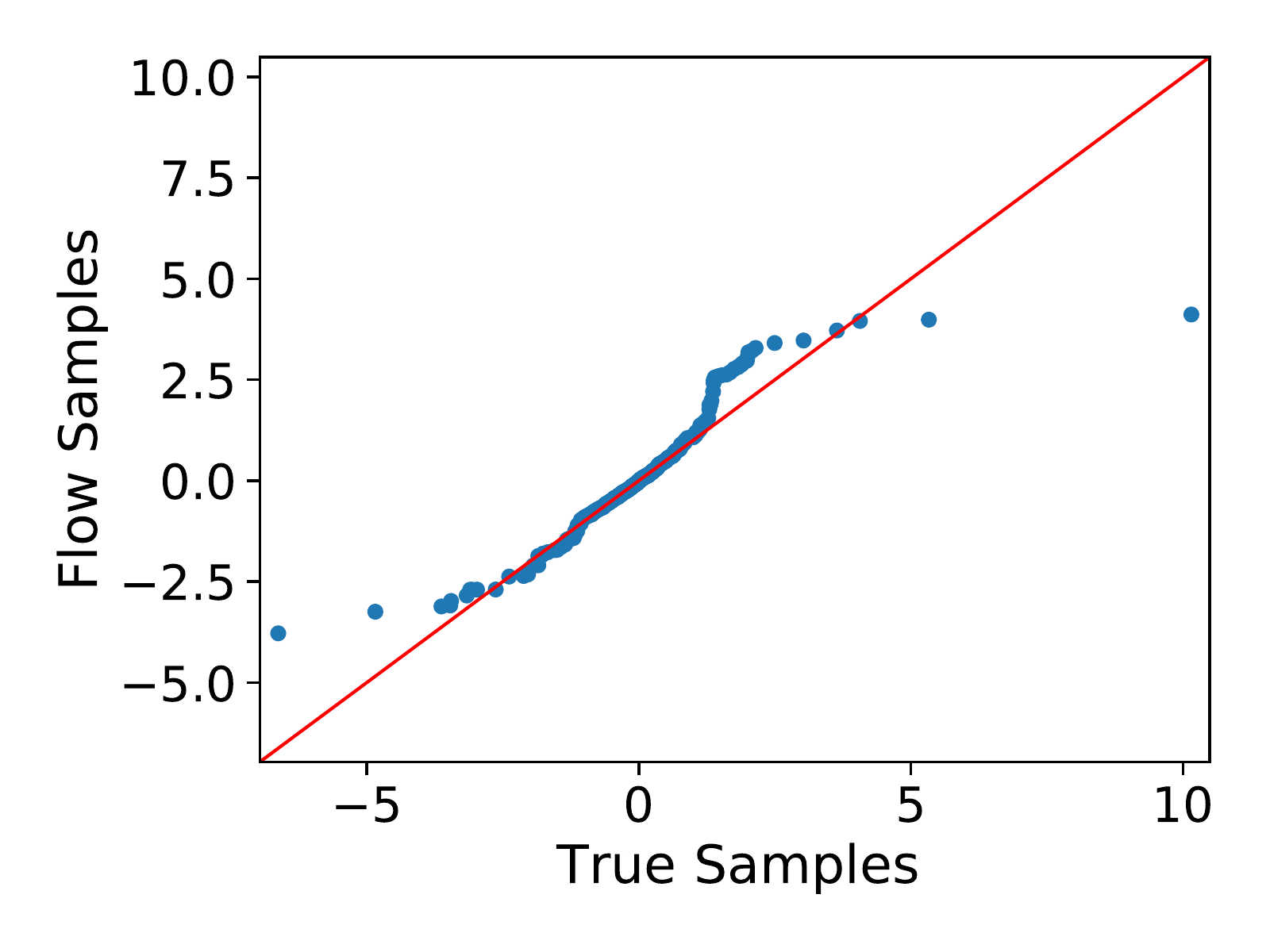}}
\end{minipage}
\begin{minipage}{0.24\textwidth}
\centerline{\includegraphics[width=\columnwidth]{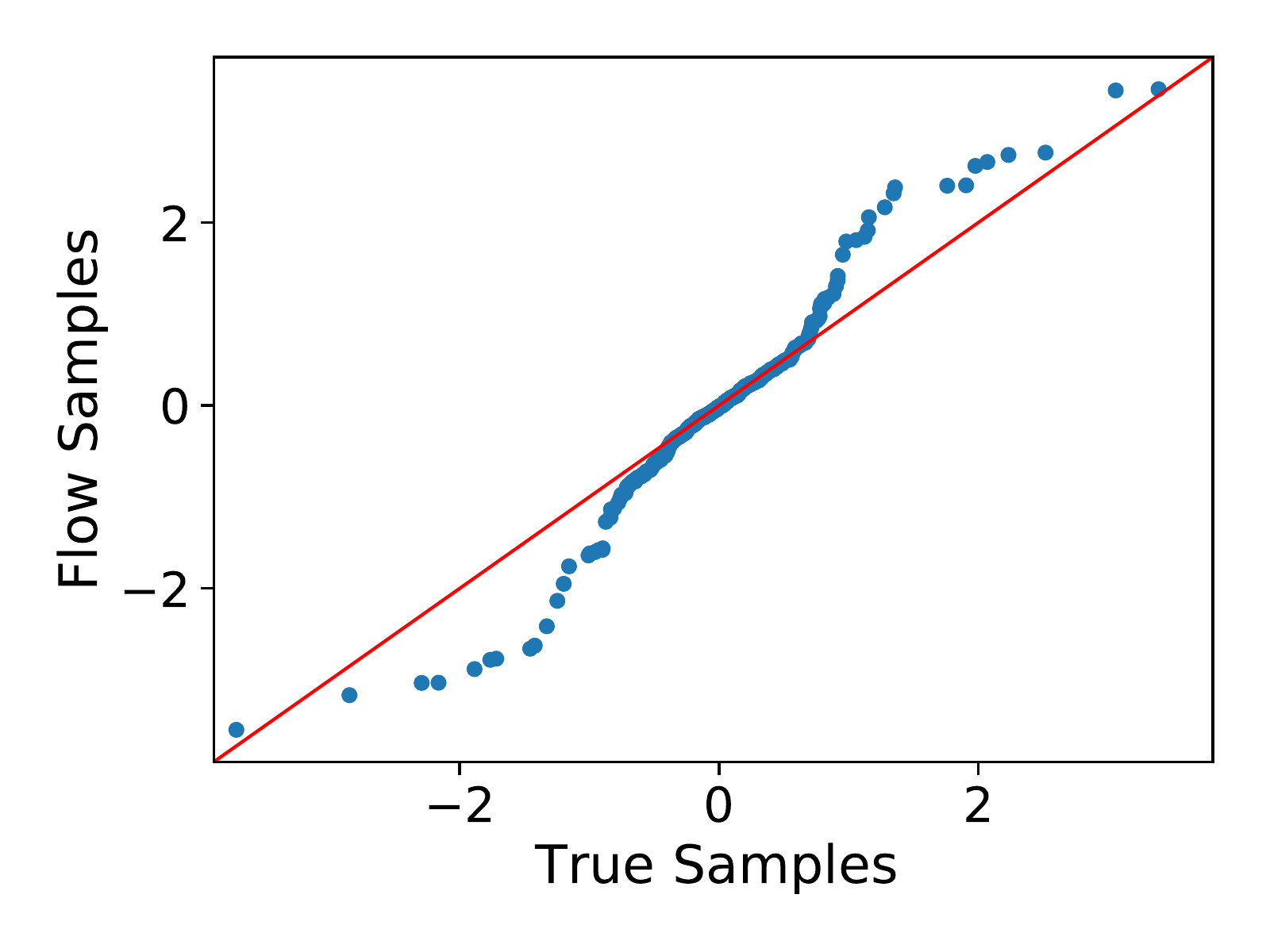}}
\end{minipage}

\begin{minipage}{0.24\textwidth}
\centerline{\includegraphics[width=\columnwidth]{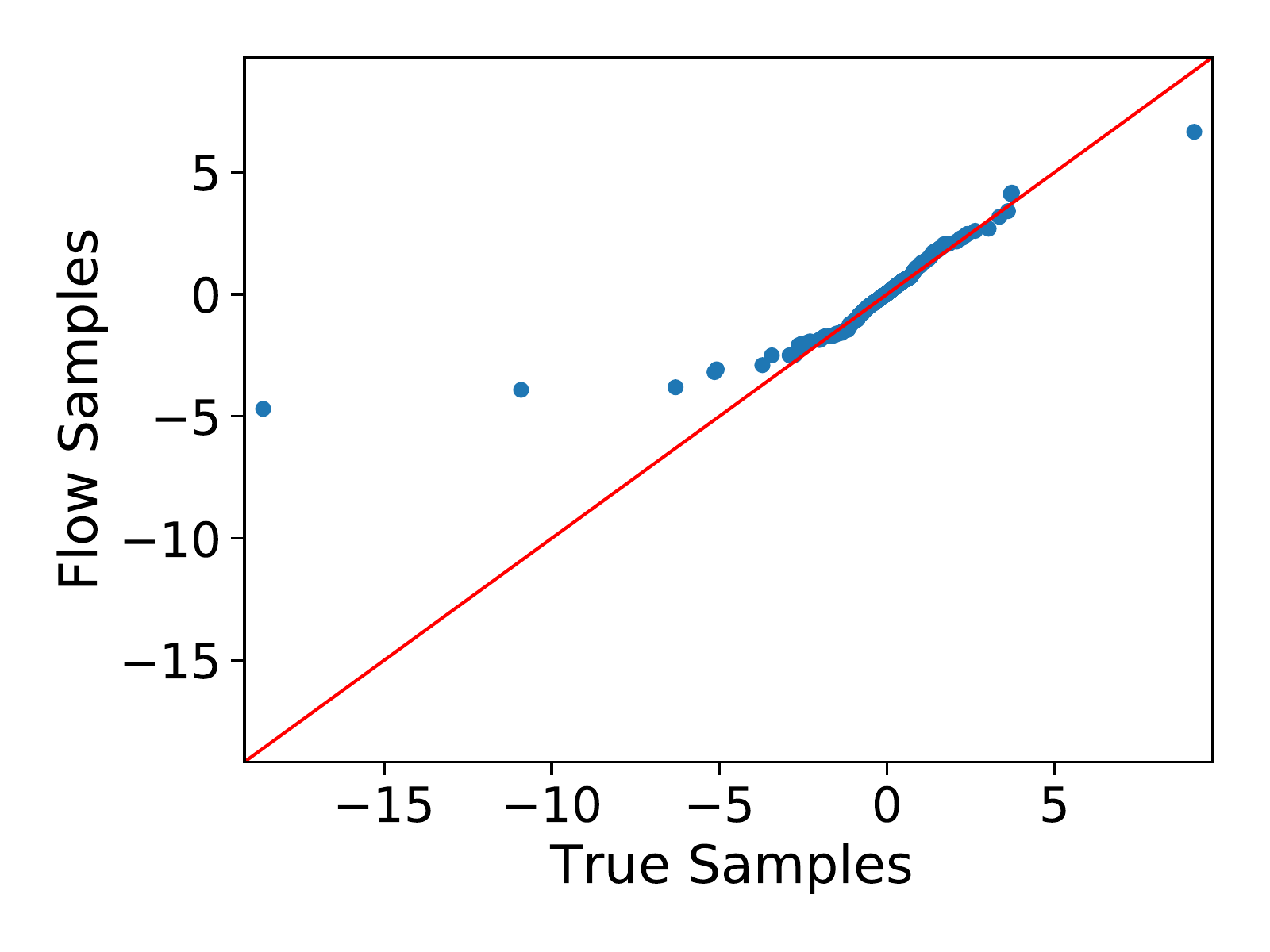}}
\end{minipage}
\begin{minipage}{0.24\textwidth}
\centerline{\includegraphics[width=\columnwidth]{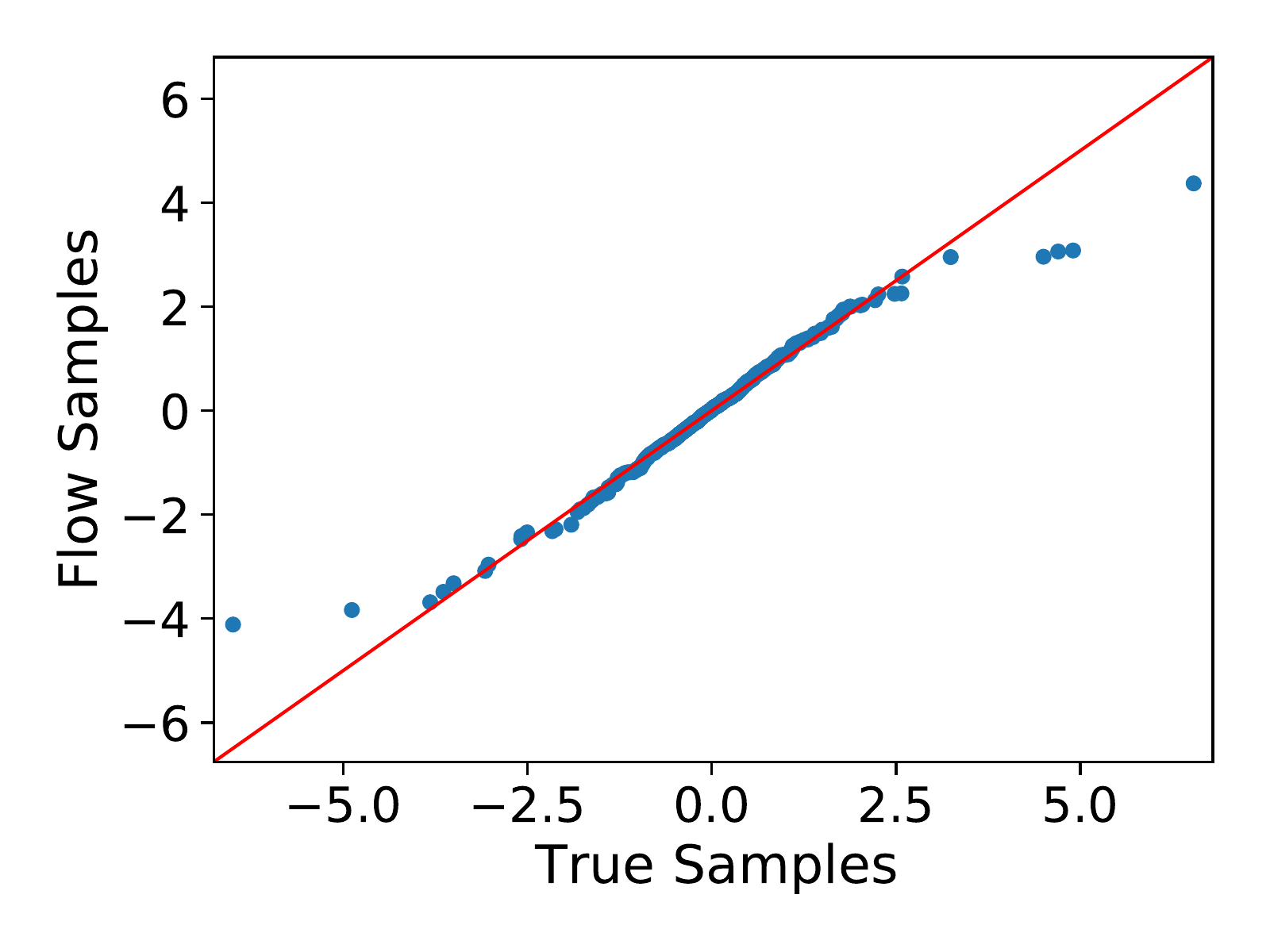}}
\end{minipage}
\begin{minipage}{0.24\textwidth}
\centerline{\includegraphics[width=\columnwidth]{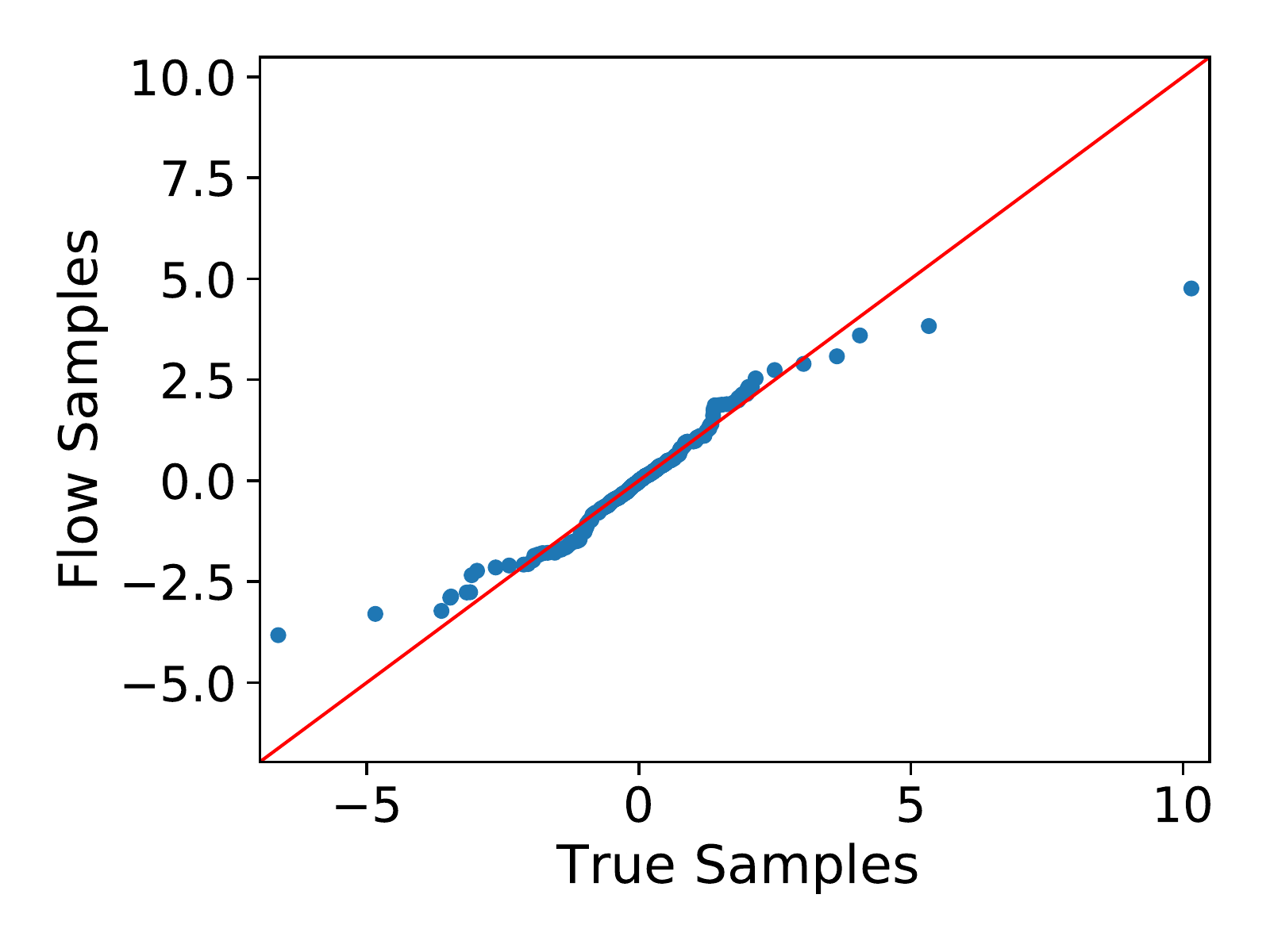}}
\end{minipage}
\begin{minipage}{0.24\textwidth}
\centerline{\includegraphics[width=\columnwidth]{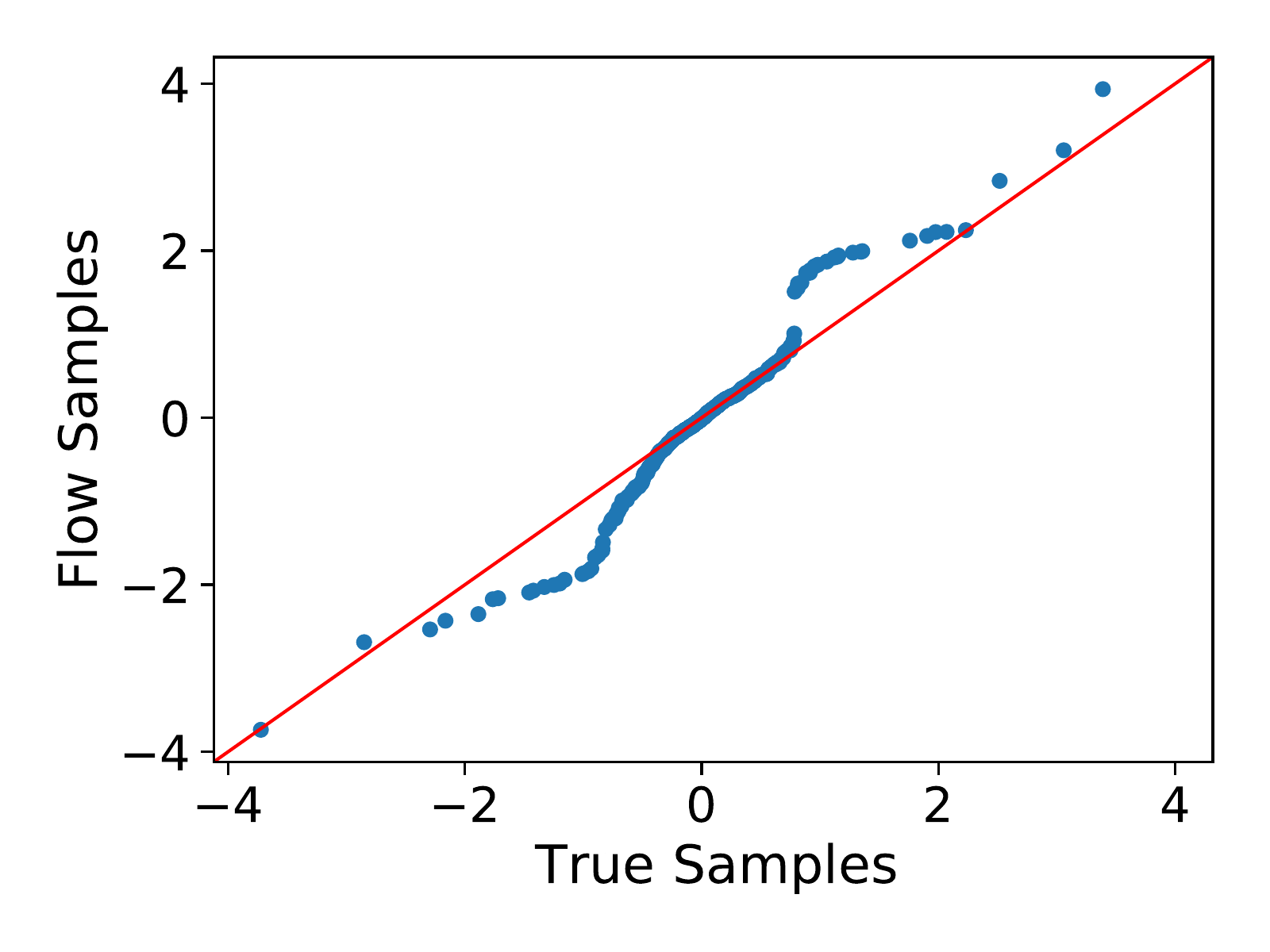}}
\end{minipage}

\begin{minipage}{0.24\textwidth}
\centerline{\includegraphics[width=\columnwidth]{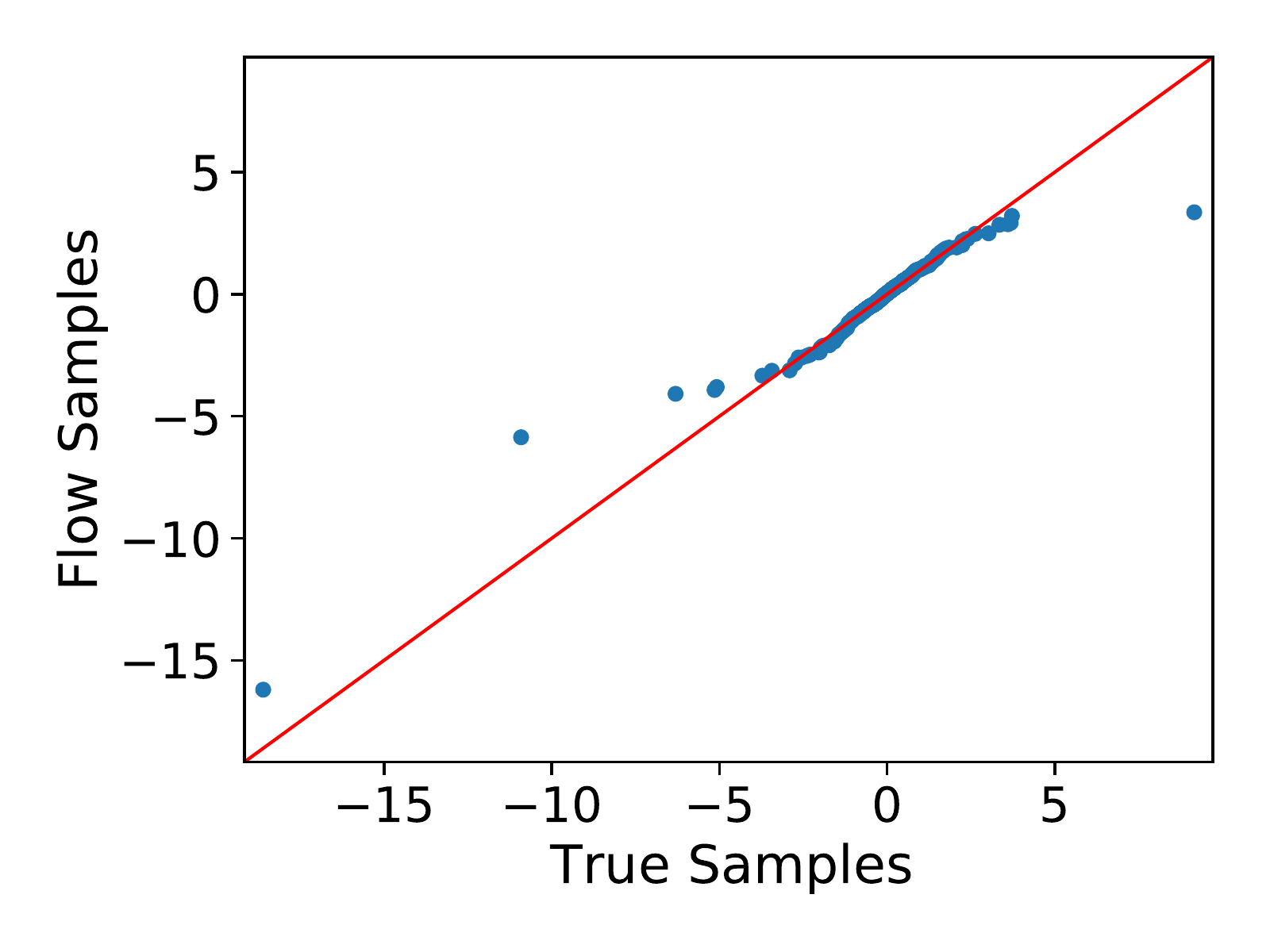}}
\end{minipage}
\begin{minipage}{0.24\textwidth}
\centerline{\includegraphics[width=\columnwidth]{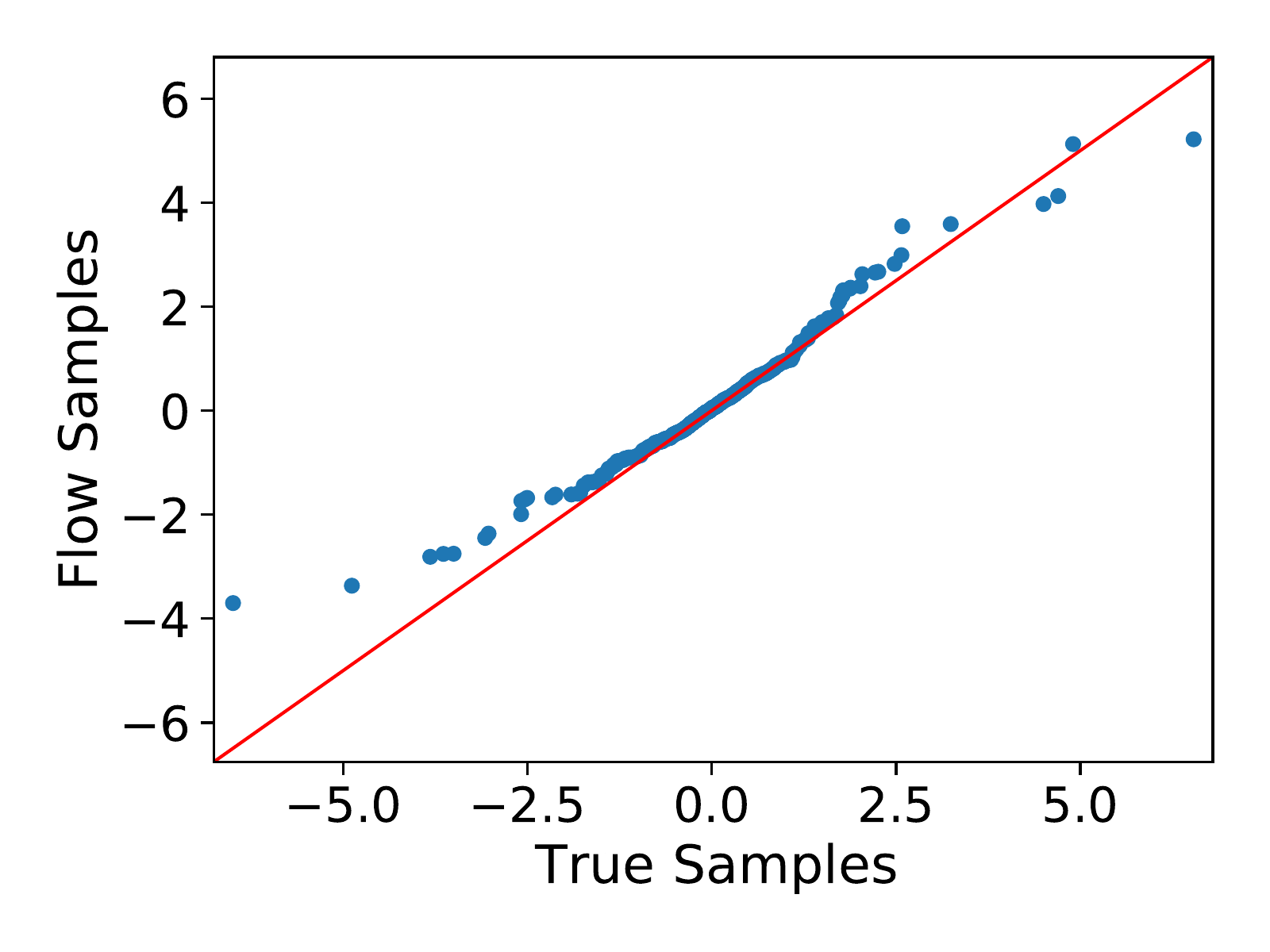}}
\end{minipage}
\begin{minipage}{0.24\textwidth}
\centerline{\includegraphics[width=\columnwidth]{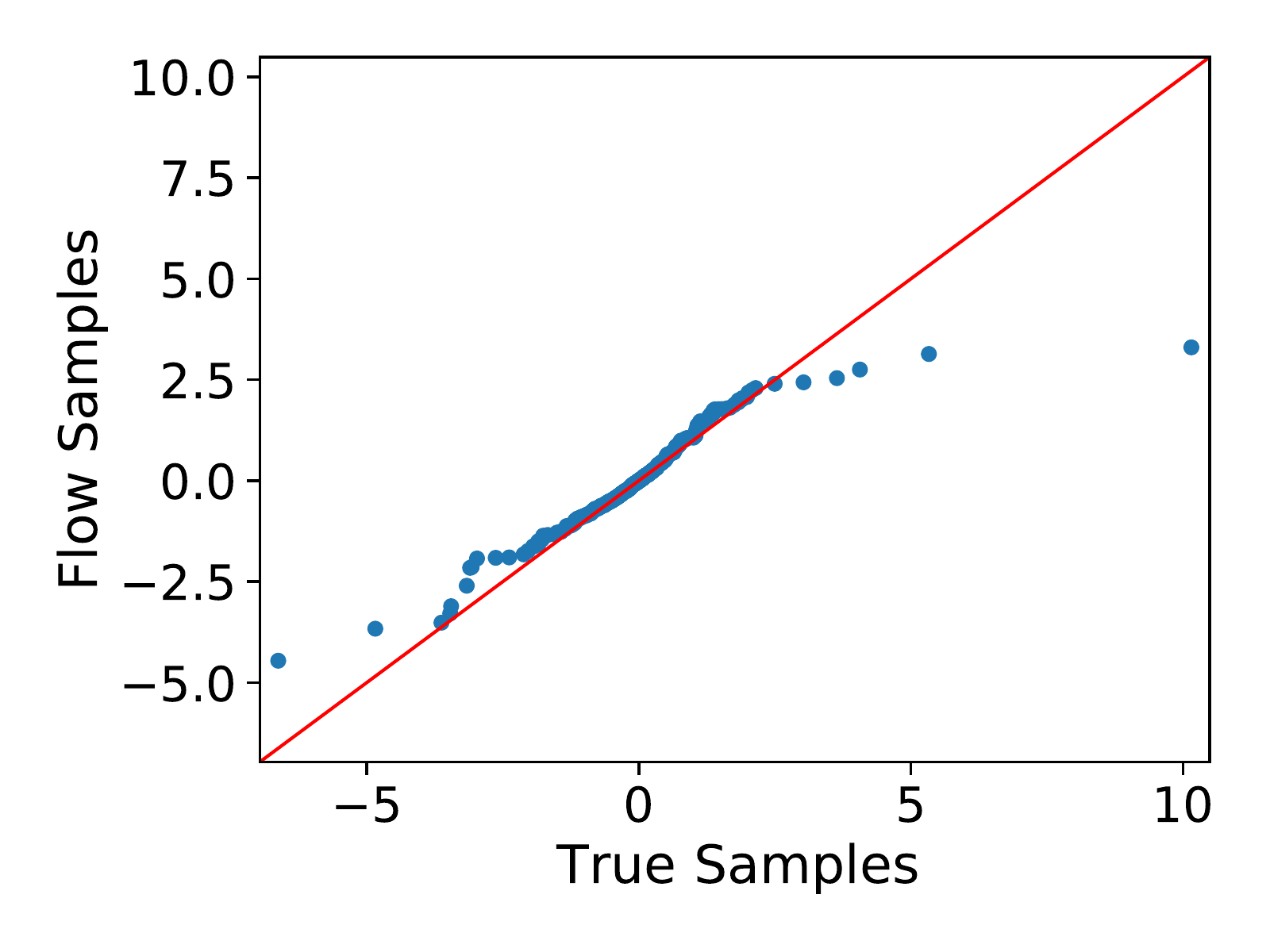}}
\end{minipage}
\begin{minipage}{0.24\textwidth}
\centerline{\includegraphics[width=\columnwidth]{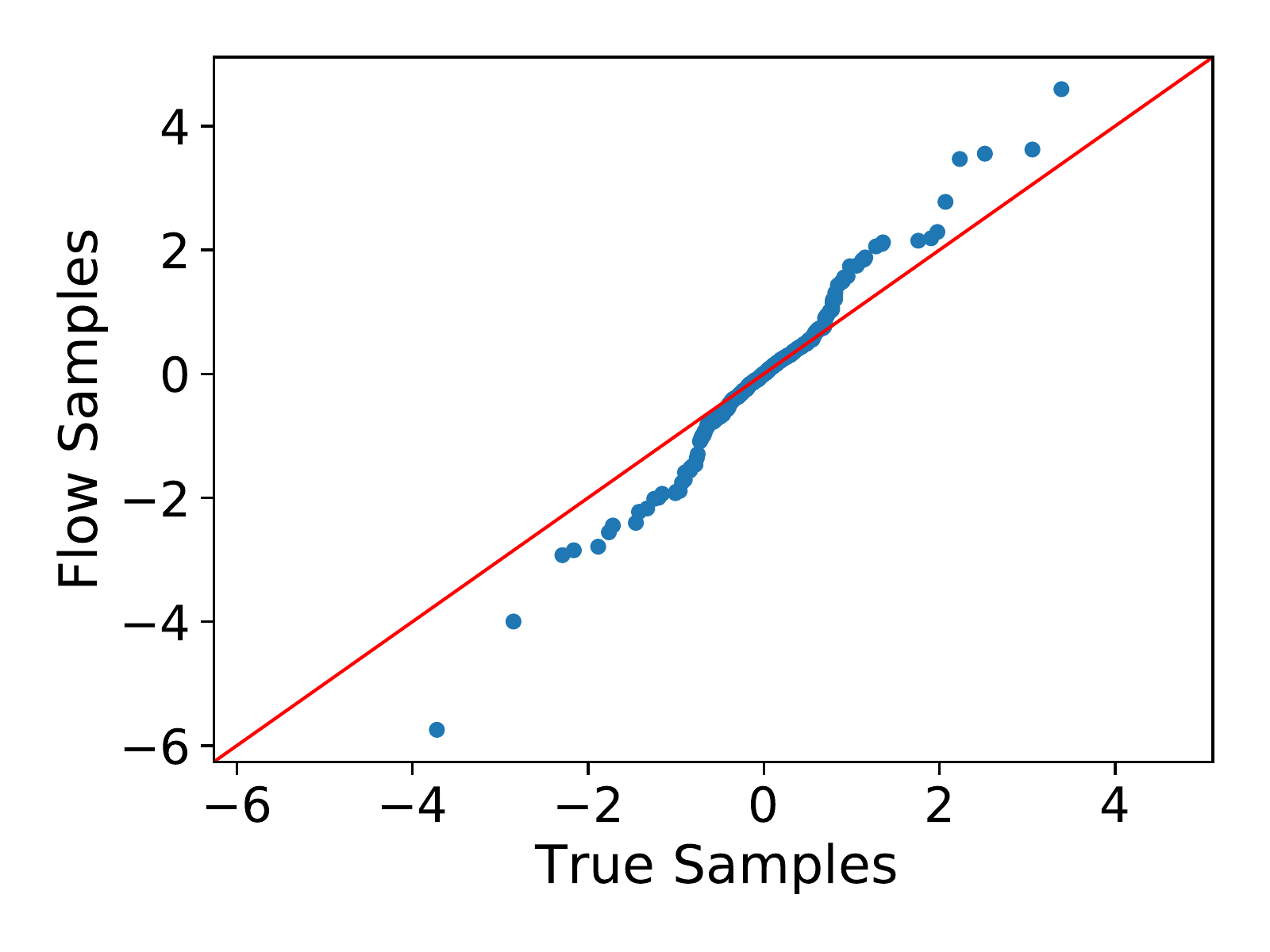}}
\end{minipage}
\end{center}
\vskip -0.2in
\caption{QQ-plots for the last $4$ heavy-tailed marginals in the setting $\nu=2$ and $d_h =4$. The rows of QQ-plots correspond to samples generated by vanilla, TAF, gTAF, and \name{}, respectively. Each marginal in depicted in one column.}
\label{fig:qq_plots_df2h4}
\end{figure}

\subsection{Climate Data}\label{sec:supp_climate}
In this section, we provide more details on the employed architectures of \name{} on the NWP-SAF dataset, which we visualize in Figure~\ref{fig:weather_true}. Furthermore, we present a more in-depth discussion about the results.   

We consider each quantity (i.e.~dry-bulb air temperature, atmospheric pressure, and cloud optical depth) at each atmospheric level as one component, leading to a $412$-dimensional dataset. Recall that \name{} requires a classification into light- and heavy-tailed marginals that lead to a reordering of the initial marginals, which we do as follows. 
This dataset---similar to other time-series-like datasets---gives us a natural autoregressive ordering, which we make use of in our permutation step. This leads to the initial permutation (compare with Step~2 in Figure~\ref{fig:overview} and Section~\ref{sec:method}), in which the first components are given by light-tailed components of the dry-bulb air temperature, followed by the reversed atmospheric pressure and the cloud-optical depth. However, in contrast to our synthetic experiments, where the selection of the heavy-tailed components was more or less trivial, this task is more complicated in a time series with highly dependent features. What we found works best in practice, is to deliberately choose a large set of heavy-tailed components according to Table~\ref{table:tails_weather}, while making the degree of freedom learnable.
%We permute the heavy-tailed components according to the same ordering of measurements. Instead of using random permutation in the LU-layers, we implement the permutation in \name{} by permuting the blocks of components corresponding to each measurement. 
Furthermore, we implemented all NFs---vanilla, TAF, gTAF, and \name{}---using $5$ autoregressive NSF layers with LU-linearities and their modified versions from Section~\ref{sec:mtaf_nsf}. The conditioner networks in the NSF-layers have $2$ hidden layers with $100$ hidden neurons in each layer, we set the tail-bounds to $2.5$, and each spline uses $3$ bins. We apply Batch-Norm after each NSF-layer. We optimize for $20\, 000$ steps using the Adam optimizer with a learning rate of $1$e-$4$ and a learning rate of $0.01$ for the tail indices and scheduled the rates using cosine annealing.   

We plot synthetic samples from the remaining flows in Figure~\ref{fig:weather_synthsamps}.

\begin{figure}[ht]
\vskip 0.2in
\begin{center}
\centerline{\includegraphics[width=0.5\columnwidth]{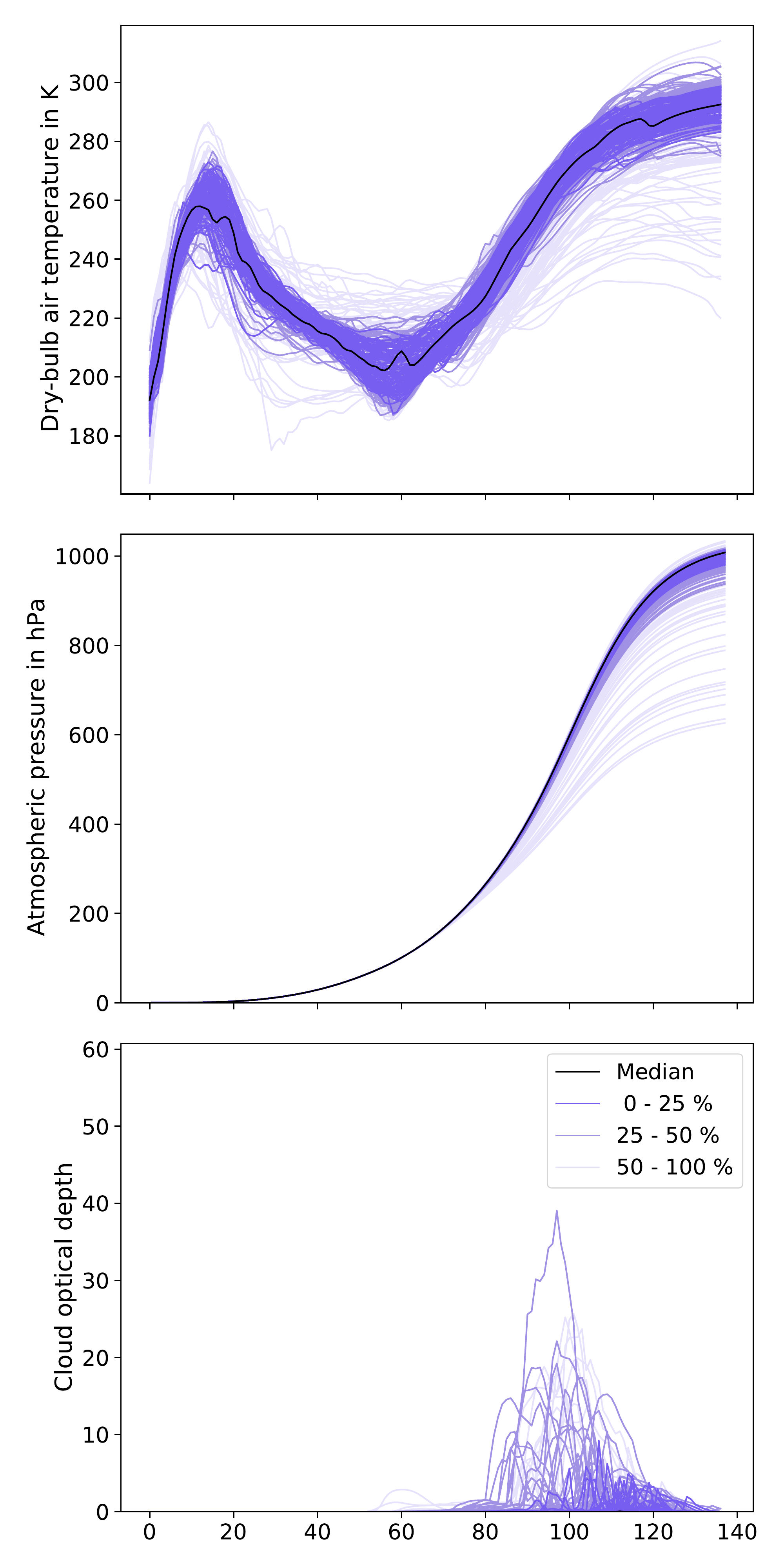}}
\caption{Real profiles from the NWP-SAF dataset. We used the implementation by \citet{meyer2021copula} to generate the figure. The profiles are ordered using band depth statistics \citep{depth_statistics}.} 
\label{fig:weather_true}
\end{center}
\vskip -0.2in
\end{figure}

\begin{figure}[ht]
\vskip 0.2in
\begin{center}
\centerline{\includegraphics[width=0.3\columnwidth]{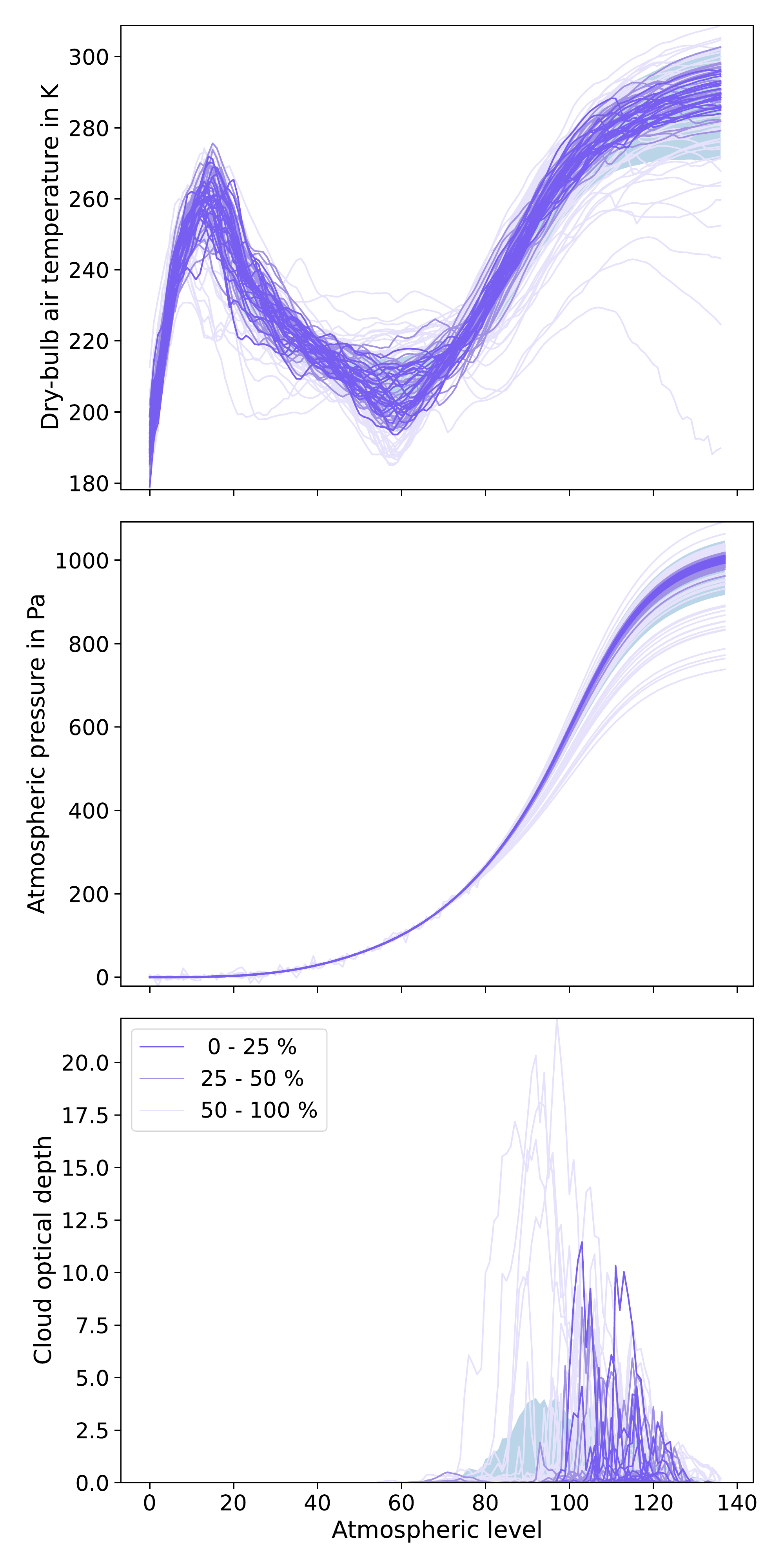}
\includegraphics[width=0.3\columnwidth]{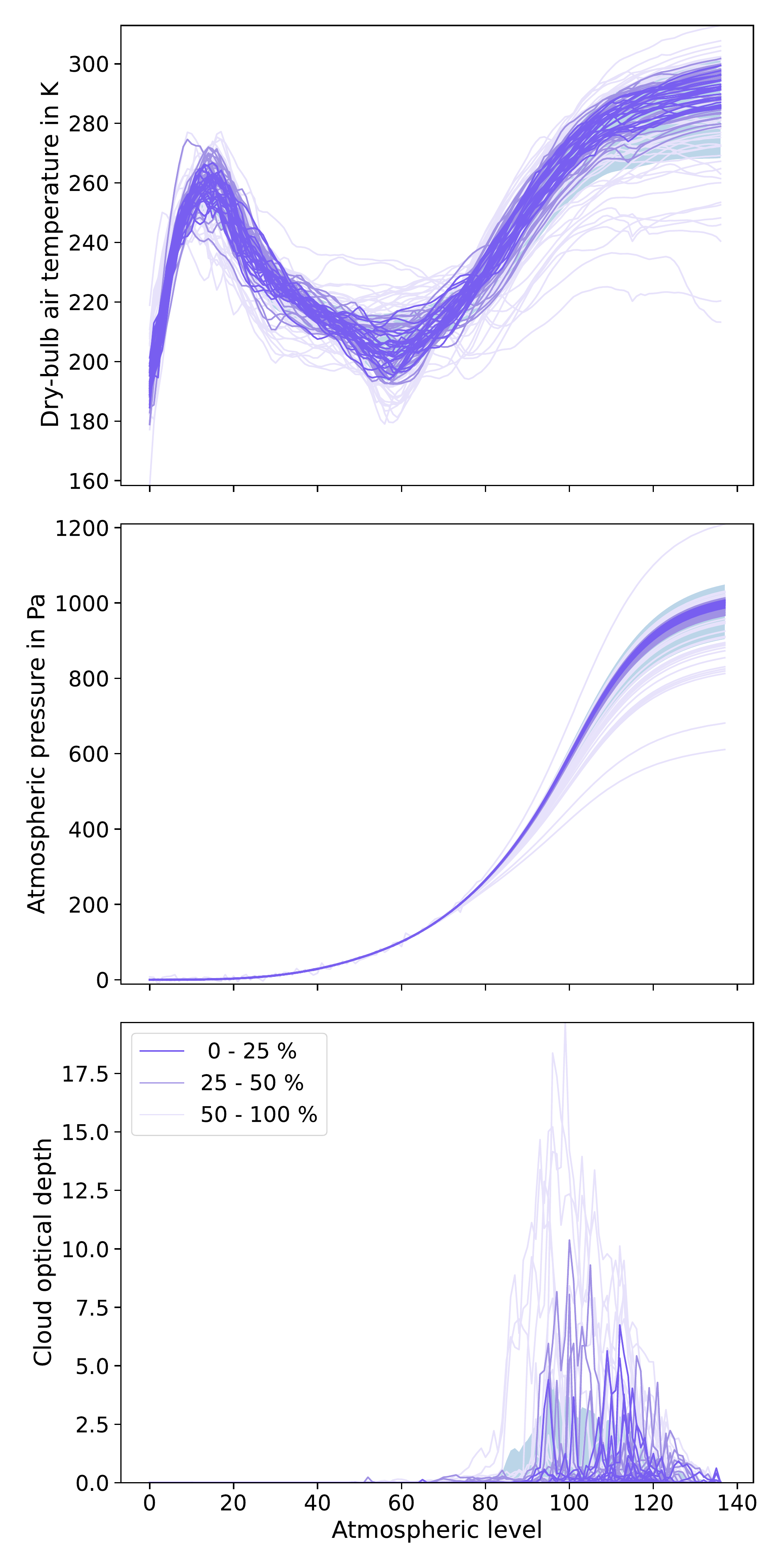}
\includegraphics[width=0.3\columnwidth]{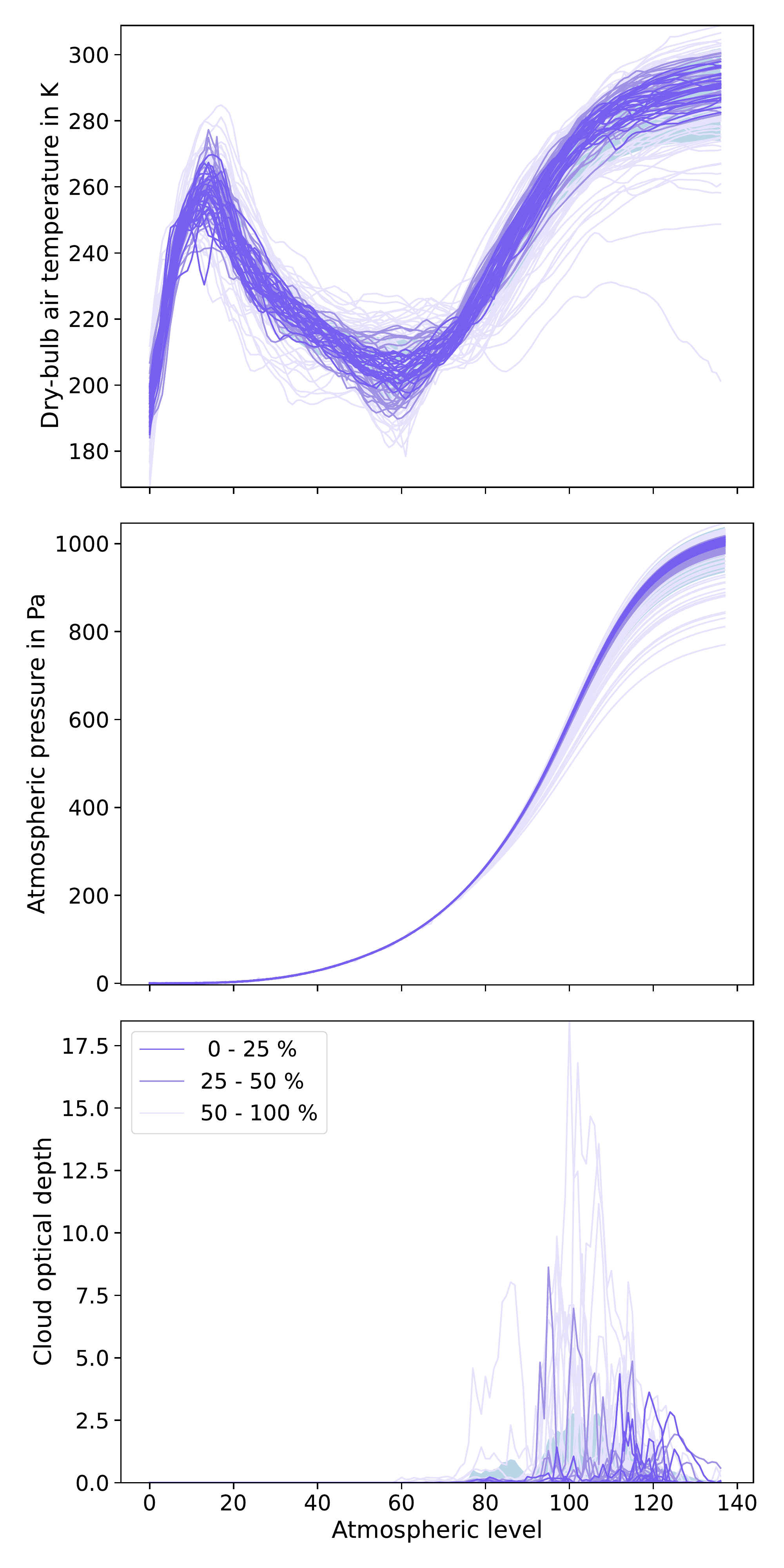}}
\caption{Synthetic flow samples using vanilla, TAF, and gTAF (from left to right), where we clipped the lower-values of the cloud-optical depth at $0$. The corresponding negative log-likelihoods are $-2094.35,\; -2117.48, \;-2121.65$, respectively. We used the implementation by \citet{meyer2021copula} to generate the figure. The profiles are ordered using band depth statistics \citep{depth_statistics} and the shaded areas represent standard deviations.} 
\label{fig:weather_synthsamps}
\end{center}
\vskip -0.2in
\end{figure}

\begin{table}[t]
\caption{Components (i.e.~measurements at a specific atmospheric level), which we manually select as heavy-tailed based on Figure~\ref{fig:weather_true}.}
\label{table:tails_weather}
\vskip 0.15in
\begin{center}
\begin{small}
\begin{sc}
\begin{tabular}{lcc}
\toprule
Measurement & Light-tailed & Heavy-tailed \\
\midrule
Dry-bulb air temperature in K & 1 - 79 & 80 - 137\\
Atmospheric pressure in hPa & 1 - 99 & 100 - 137 \\
Cloud optical depth & 1 - 57 & 58 - 137 \\ 
\bottomrule
\end{tabular}
\end{sc}
\end{small}
\end{center}
\vskip -0.1in
\end{table}

%\subsection{UCI datasets}\label{sec:UCI}

\end{document}